\documentclass{article}

\PassOptionsToPackage{numbers, compress}{natbib}

\usepackage[final]{neurips_2021}

\usepackage[utf8]{inputenc} 
\usepackage[T1]{fontenc}    

\usepackage{url}            
\usepackage{booktabs}       
\usepackage{amsfonts}       
\usepackage{nicefrac}       
\usepackage{microtype}      
\usepackage[pdftex,dvipsnames]{xcolor}       
\usepackage{graphicx}
\graphicspath{{./figures/}{../figures/}}
\usepackage{caption,subcaption}
\usepackage{enumitem}

\usepackage{hyperref}       
\hypersetup{
	plainpages=false,
    colorlinks,
    linktocpage=true,   
    linkcolor={red!50!black},
    citecolor={blue!50!black},
    urlcolor={blue!80!black}
}

\usepackage{amsmath,amssymb,amsthm}
\newtheorem{theorem}{Theorem}[section]
\newtheorem{lemma}[theorem]{Lemma}

\newtheorem{assumption}[theorem]{Assumption}

\newtheorem{remark}[theorem]{Remark}

\newtheorem{manualtheoreminner}{Theorem}
\newenvironment{manualtheorem}[1]{%
  \renewcommand\themanualtheoreminner{#1}%
  \manualtheoreminner
}{\endmanualtheoreminner}

\usepackage{bm,dsfont,siunitx}
\usepackage{kantlipsum} 
\usepackage{multirow, makecell}
\usepackage{caption,subcaption}
\newcommand{\norm}[1]{\left\Vert#1\right\Vert}
\newcommand{\subscript}[2]{$#1 _ #2$} 
\newcommand{\op}[1]{\operatorname{#1}}
\newcommand{\rank}{\op{rank}}

\newcommand{\dd}{\,{\mathrm d}}
\renewcommand{\b}[1]{\bf{#1}}
\newcommand{\bb}[1]{\mathbb{#1}}
\newcommand{\R}{\mathbb{R}}

\makeatletter
\newcommand*\bdot{\mathpalette\bdot@{.65}}
\newcommand*\bdot@[2]{\mathbin{\vcenter{\hbox{\scalebox{#2}{$\m@th#1\bullet$}}}}}
\makeatother

\newcommand{\appendixtitle}[1]{
\rule{\linewidth}{4pt}
\vskip 0.1in
\hfil {\Large \textbf{
#1}
}
\vspace{2mm} \hfil
\vspace{-3mm}
\vskip 0.05in
\rule{\linewidth}{1pt}
}

\title{Choose a Transformer: Fourier or Galerkin}

\author{%
  Shuhao Cao
\\
  Department of Mathematics and Statistics\\
  Washington University in St. Louis\\
  \texttt{s.cao@wustl.edu} \\
}

\begin{document}

\maketitle

\begin{abstract}
In this paper, we apply the self-attention from the state-of-the-art Transformer in \emph{Attention Is All You Need}~\cite{Vaswani;Shazeer;Parmar:2017Attention} for the first time to a data-driven operator learning problem related to partial differential equations. 
An effort is put together to explain the heuristics of, and to improve
the efficacy of the attention mechanism. By employing the operator approximation theory in Hilbert spaces, it is demonstrated for the first time that the softmax normalization in the scaled dot-product attention is sufficient but not necessary. Without softmax, the approximation capacity of a linearized Transformer variant can be proved to be comparable to a Petrov-Galerkin projection layer-wise, and the estimate is independent with respect to the sequence length. 
A new layer normalization scheme mimicking the Petrov-Galerkin projection is proposed to allow a scaling to propagate through attention layers, which helps the model achieve remarkable accuracy in operator learning tasks with unnormalized data. Finally, we present three operator learning experiments, including the viscid Burgers' equation, an interface Darcy flow, and an inverse interface coefficient identification problem. The newly proposed simple attention-based operator learner, Galerkin Transformer, shows
significant improvements in both training cost and evaluation accuracy over
its softmax-normalized counterparts.
\end{abstract}

\section{Introduction}

Partial differential equations (PDEs) arise from almost every multiphysics and biological systems, from the interaction of atoms to the merge of galaxies, from the formation of cells to the change of climate. Scientists and engineers have been working on approximating the governing PDEs of these physical systems for centuries. The emergence of the computer-aided simulation facilitates a cost-friendly way to study these challenging problems. 
Traditional methods, such as finite element/difference \cite{Ciarlet:2002finite,Courant.Friedrichs.ea:1967partial}, spectral methods \cite{Bracewell.Bracewell:1986Fourier}, etc., leverage a discrete structure to reduce an infinite dimensional operator map to a finite dimensional approximation problem. 
Meanwhile, in the field practice of many scientific disciplines, substantial data for PDE-governed phenomena available on discrete grids enable modern black-box models like Physics-Informed Neural Network (PINN) \cite{Raissi.Perdikaris.ea:2019Physics,Lu.Meng.ea:2021DeepXDE,Karniadakis.Kevrekidis.ea:2021Physics} to exploit measurements on collocation points to approximate PDE solutions.

Nonetheless, for traditional methods or data-driven function learners such as PINN, given a PDE, the focus is to approximate a single instance, for example, solving for an approximated solution for one coefficient with a fixed boundary condition. A slight change to this coefficient invokes a potentially expensive re-training of any data-driven function learners. In contrast, an operator learner aims to learn a map between infinite-dimensional function spaces, which is much more difficult yet rewarding. A well-trained operator learner can evaluate many instances without re-training or collocation points, thus saving valuable resources, 
and poses itself as a more efficient approach in the long run. Data-driven resolution-invariant operator learning is a booming new research
direction \cite{LuJinKarniadakis2019Deeponet,Li.Kovachki.ea:2020Neural,Li.Kovachki.ea:2020Multipole,Nelsen.Stuart:2020Random,Wang.Teng.ea:2021Understanding,Li.Kovachki.ea:2021Fourier,Lu.Jin.ea:2021Deeponet,Wang.Wang.ea:2021Learning,GuptaXiaoBogdan2021Multiwavelet,RobertsKhodakDaoLiEtAl2021Rethinking}. The pioneering model, DeepONet \cite{LuJinKarniadakis2019Deeponet}, attributes architecturally to a universal approximation theorem for operators \cite{ChenChen1995Universal}. Fourier Neural Operator (FNO) \cite{Li.Kovachki.ea:2021Fourier} notably shows an awing state-of-the-art performance outclassing classic models such as the one in \cite{Zhu.Zabaras:2018Bayesian} by orders of magnitudes in certain benchmarks.

Under a supervised setting, an operator learner is trained with the operator's input functions and their responses to the inputs as targets. Since both functions are sampled at discrete grid points, this is a special case of a \texttt{seq2seq} problem \cite{Sutskever.Vinyals.ea:2014Sequence}.
The current state-of-the-art \texttt{seq2seq} model is the Transformer first introduced in \cite{Vaswani;Shazeer;Parmar:2017Attention}.
As the heart and soul of the Transformer, the scaled dot-product attention mechanism is capable of unearthing the hidden structure of an operator by capturing long-range interactions. Inspired by many insightful pioneering work in Transformers \cite{Katharopoulos.Vyas.ea:2020Transformers, Choromanski.Likhosherstov.ea:2021Rethinking, Schlag.Irie.ea:2021Linear,Tsai.Bai.ea:2019Transformer,Xiong.Yang.ea:2020Layer,Xiong.Zeng.ea:2021Nystromformer,Wright.Gonzalez:2021Transformers,Lu.Grover.ea:2021Pretrained,Shen.Zhang.ea:2021Efficient,Nguyen.Salazar:2019Transformers}, we have modified the attention mechanism minimally yet in a mathematically profound manner to better serve the purpose of operator learning. 

Among our new Hilbert space-inspired adaptations of the scaled dot-product attention, the first and foremost change is: no softmax, or the approximation thereof. In the vanilla attention \cite{Vaswani;Shazeer;Parmar:2017Attention}, the softmax succeeding the matrix multiplication convexifies the weights for combining different positions' latent representations, which is regarded as an indispensable ingredient in the positive kernel interpretation of the attention mechanism \cite{Tsai.Bai.ea:2019Transformer}. 
However, softmax acts globally in the sequence length dimension for each row of the attention matrix, and further adds to the quadratic complexity of the attention in the classic Transformer. Theory-wise, instead of viewing ``$\text{row}\approx\text{word}$'' in the Natural Language Processing (NLP) tradition, the columns of the query/keys/values are seen as sampling of functions in Hilbert spaces on discretized grids. Thus, taking the softmax away allows us to verify a discrete Ladyzhenskaya–Babu\v{s}ka–Brezzi (LBB) condition, which further amounts to the proof that the newly proposed Galerkin-type attention can explicitly represent a Petrov-Galerkin projection, and this approximation capacity is independent of the sequence length (Theorem \ref{theorem:cea-lemma}).

Numerically, the softmax-free models save valuable computational resources, outperforming the ones with the softmax in terms of training FLOP and memory consumption (Section \ref{sec:experiments}). Yet in an ablation study, the training becomes unstable for softmax-free models (Table \ref{table:burgers-init}). To remedy this, a new Galerkin projection-type layer normalization scheme is proposed to act as a cheap diagonal alternative to the normalizations explicitly derived in the proof of the Petrov-Galerkin interpretation (equation \eqref{eq:attention-update}). Since a learnable scaling can now be propagated through the encoder layers, the attention-based operator learner with this new layer normalization scheme exhibits better comprehension of certain physical properties associated with the PDEs such as the energy decay. Combining with other approximation theory-inspired tricks including a diagonally dominant rescaled initialization for the projection matrices and a layer-wise enrichment of the positional encodings, the evaluation accuracies in various operator learning tasks are boosted by a significant amount.

\paragraph{Main contributions.}
\label{paragraph:contributions}
The main contributions of this work are summarized as follows. 

\begin{itemize}[topsep=0pt, leftmargin=1.5em]
    \item \textbf{Attention without softmax.} We propose a new simple  self-attention operator and its linear variant without the softmax normalization. Two new interpretations are offered, together with the approximation capacity of the linear variant proved comparable to a Petrov-Galerkin projection.

    \item \textbf{Operator learner for PDEs.} We combine the newly proposed attention operators with the current best state-of-the-art operator learner Fourier Neural Operator (FNO) \cite{Li.Kovachki.ea:2021Fourier} to significantly improve its evaluation accuracy in PDE solution operator learning benchmark problems.
    Moreover, the new model is capable of recovering coefficients based on noisy measurements that traditional methods or FNO cannot accomplish. 

    \item \textbf{Experimental results.} We present three benchmark problems to show that operator learners using the newly proposed attentions are superior in computational/memory efficiency, as well as in accuracy versus those with the conventional softmax normalization. The PyTorch codes to reproduce our results
    are available as an open-source software. 
    \footnote{\url{https://github.com/scaomath/galerkin-transformer}}
\end{itemize}

\section{Related Works}
\label{sec:related-works}

\paragraph{Operator learners related to PDEs.} 
In \cite{Alet.Jeewajee.ea:2019Graph,Li.Kovachki.ea:2020Neural}, certain kernel forms of the solution operator of parametric PDEs are approximated using graph neural networks. The other concurrent notable approach is DeepONet \cite{LuJinKarniadakis2019Deeponet,Lu.Jin.ea:2021Deeponet}. 
\cite{Li.Kovachki.ea:2020Multipole} further improves the kernel approach by exploiting the multilevel grid structure. \cite{Li.Kovachki.ea:2021Fourier} proposes a discretization-invariant operator learner to achieve a state-of-the-art performance in certain benchmark problems. \cite{Wang.Teng.ea:2021Understanding,Wang.Wang.ea:2021Learning} proposed a DeepONet roughly equivalent to an additive attention, similar to the one in the Neural Turing Machine (NMT) in \cite{Bahdanau.Cho.ea:2016Neural}.
Model/dimension reduction combined with neural nets is another popular approach to learn the solution operator for parametric PDEs
\cite{Bhattacharya.Hosseini.ea:2020Model,Nelsen.Stuart:2020Random,LiZhangZhao2020data,DalSantoDeparisPegolotti2020Data}. 
Deep convolutional neural networks (DCNN) are widely applied to learn the solution maps with a fixed discretization size \cite{Adler.Oektem:2017Solving,Bhatnagar.Afshar.ea:2019Prediction,He.Xu:2019MgNet,Guo.Li.ea:2016Convolutional,GuoJiang2021Construct,Zhu.Zabaras:2018Bayesian,Ummenhofer.Prantl.ea:2020Lagrangian}. Recently, DCNN has been successfully applied in various inverse problems \cite{GuoJiang2021Construct,JiangLiGuo2021learn} such as Electrical Impedance Tomography (EIT).
To our best knowledge, there is no work on data-driven approaches to an inverse interface coefficient identification for a class of coefficients with random interface geometries.

\paragraph{Attention mechanism and variants.}
Aside from the ground-breaking scaled dot-product attention in \cite{Vaswani;Shazeer;Parmar:2017Attention}, earlier \cite{Bahdanau.Cho.ea:2016Neural} proposed an additive content-based attention, 
however, with a vanishing gradient problem due to multiple nonlinearity composition. \cite{Brebisson.Vincent:2016Cheap} shows the first effort in removing the softmax normalization in \cite{Bahdanau.Cho.ea:2016Neural} after the projection, however, it still uses a Sigmoid nonlinearity before the additive interpolation propagation stage, and performs worse than its softmax counterpart. The current prevailing approach to linearize the attention leverages the assumption of the existence of a feature map to approximate the softmax kernel \cite{Katharopoulos.Vyas.ea:2020Transformers, Choromanski.Likhosherstov.ea:2021Rethinking,PengPappasYogatamaSchwartzEtAl2021Random}. 
Another type of linearization exploits the low-rank nature of the matrix product using various methods such as sampling or projection \cite{RawatChenEtAl2019rfa,BlancRendle2018adaptive,SongJungKimMoon2021implicit,Wang.Li.ea:2020Linformer}, 
or fast multipole decomposition \cite{NguyenSuliafuOsherChenEtAl2021FMMformer}. The conjecture in \cite{Schlag.Irie.ea:2021Linear} inspires us to remove the softmax overall. \cite{Shen.Zhang.ea:2021Efficient} first proposed the inverse sequence length scaling normalization for a linear complexity attention without the softmax, however, the scaling normalization has not been extensively studied in examples and performs worse.

\paragraph{Various studies on Transformers.}
The kernel interpretation in \cite{Tsai.Bai.ea:2019Transformer} inspires us to reformulate the attention using the Galerkin projection. \cite[Theorem 2]{Wright.Gonzalez:2021Transformers} gives a theoretical foundation of removing the softmax normalization to formulate the Fourier-type attention. 
The Nystr{\"o}m approximation \cite{Xiong.Zeng.ea:2021Nystromformer} 
essentially acknowledges the similarity between the attention matrix and an integral kernel. \cite{Xiong.Yang.ea:2020Layer,Nguyen.Salazar:2019Transformers,Lu.Grover.ea:2021Pretrained} inspires us to try different layer normalization and the rescaled diagonally dominant initialization schemes.
 The practices of enriching the latent representations with the positional encoding recurrently in our work trace back to \cite{AlRfouChoeConstantGuoEtAl2019Character,DehghaniGouwsVinyalsUszkoreitEtAl2019Universal}, and more recently, contribute to the success of AlphaFold 2 \cite{JumperEvansPritzelGreenEtAl2021Highly}, as it is rewarding to exploit the universal approximation if the target has a dependence ansatz in the coordinate frame and/or transformation group but hard to be explicitly quantified. Other studies on adapting the attention mechanisms to conserve important physical properties are in \cite{Tai.Bailis.ea:2019Equivariant,Fuchs.Worrall.ea:2020SE3,Hutchinson.Lan.ea:2021LieTransformer}.

\section{Operator learning related to PDEs}
\label{sec:operator-learning}
Closely following the setup in \cite{Li.Kovachki.ea:2020Multipole,Li.Kovachki.ea:2021Fourier}, we consider a data-driven model to approximate a densely-defined operator $T: \mathcal{H}_1 \to \mathcal{H}_2$ between two Hilbert spaces with an underlying bounded spacial domain $\Omega\subset \R^m$. The operator $T$ to be learned is usually related to certain physical problems, of which the formulation is to seek the solution to a PDE of the following two types.

Parametric PDE: given coefficient $a\in \mathcal{A}$, and source $f\in \mathcal{Y}$, find $u\in \mathcal{X}$ such that $L_a (u) = f$.
  \begin{itemize}[topsep=0pt, leftmargin=2.5em]
  \item[(i)] To approximate the nonlinear mapping from the varying parameter $a$ to the solution with a fixed right-hand side, $T: \mathcal{A} \to \mathcal{X}, \; a\mapsto u$.
  \item[(ii)] The inverse coefficient identification problem to recover the coefficient from a noisy measurement $\tilde{u}$ of the steady-state solution $u$, in this case,  $T: \mathcal{X} \to \mathcal{A}, \; \tilde{u}\mapsto a$.
\end{itemize}

Nonlinear initial value problem: given $u_0\in \mathcal{H}_0$, find $u\in C([0, T]; \mathcal{H})$ such that $\partial_t u + N(u)=0$.
\begin{itemize}[topsep=0pt, leftmargin=2.5em]
    \item[(iii)] Direct inference from the initial condition to the solution.  $T: \mathcal{H}_0\to \mathcal{H}, \;u_0(\cdot) \mapsto u(t_1, \cdot)$ with $t_1\gg \Delta t$ with $t_1$ much greater than the step-size in traditional explicit integrator schemes. 
\end{itemize}

Using (i) as an example, based on the given $N$ observations $\{a^{(j)}, u^{(j)}\}^N_{j=1}$ and their approximations $\{a^{(j)}_h, u^{(j)}_h\}$ defined at a discrete grid of size $h\ll 1$, 
the goal of our operator learning problem is 
to build an approximation $T_{\theta}$ to $T$, such that $T_{\theta}(a_h)$ is a good approximation to 
$u =L_a^{-1} f=:T(a)\approx u_h$ independent of the mesh size $h$, where $a_h$ and $u_h$ are in finite dimensional spaces $\bb{A}_h, \bb{X}_h$ on this grid. 
We further assume that
$a^{(j)}\sim \nu$ for a measure $\nu$ compactly supported on $\mathcal{A}$, 
and the sampled data form a reasonably sized subset of $\mathcal{A}$ representative of field applications. 
The loss $\mathcal{J}(\theta)$ is
\begin{equation}
\label{eq:loss-continuous}
\mathcal{J}(\theta) := \bb{E}_{a\sim \nu} 
\left[\|\big(T_{\theta}(a) - u\|_{\mathcal{H}}^2 + \mathfrak{G}(a, u;\theta) \right]
\end{equation}
and in practice is approximated using the sampled observations on a discrete grid
\vspace*{-5pt}
\begin{equation}
\label{eq:loss-discrete}
\mathcal{J}(\theta)  \approx 
\frac{1}{N} \sum_{j=1}^N
\Big\{ \big\|\big(T_{\theta}\big(a_h^{(j)}\big) - u_h^{(j)}\big\|_{\mathcal{H}}^2 
+ \mathfrak{G}\big(a_h^{(j)}, u_h^{(j)};\theta\big) \Big\}.
\vspace*{-5pt}
\end{equation}
In example (i), $\norm{\cdot}_{\mathcal{H}}$ is the standard $L^2$-norm, 
and $\mathfrak{G}(a, u;\theta)$ serves as a regularizer with strength $\gamma$ and is problem-dependent. In Darcy flow where $L_a := -\nabla\cdot(a\nabla (\cdot))$, it is $\gamma\|a\nabla (T_\theta(a) - u)\|_{L^2(\Omega)}^2$, since $u\in H^{1+\alpha}(\Omega)$ ($\alpha>0$ depends on the regularity of $a$) and 
$a\nabla u\in \bm{H}(\mathrm{div};\Omega)$ a priori. For the evaluation metric, we drop the $\mathfrak{G}(a, u;\theta)$ term, and monitor the minimization of \eqref{eq:loss-discrete} using $\|\cdot\|_{\mathcal{H}}$. 

\section{Attention-based operator learner}
\label{sec:network}

\paragraph{Feature extractor.} We assume the functions in both inputs and targets are sampled on a uniform grid. In an operator learning problem on $\Omega\subset \R^1$, 
a simple feedforward neural network (FFN) is used as the feature extractor that is shared by every position (grid point).

\paragraph{Interpolation-based CNN.} If $\Omega\subset \R^2$, inspired by  the multilevel graph kernel network in \cite{Li.Kovachki.ea:2020Multipole}, we use two 3-level interpolation-based CNNs (CiNN) as the feature extractor, but also as the downsampling and upsampling layer, respectively, in which we refer to restrictions/prolongations between the coarse/fine grids both as interpolations. 
For the full details of the network structure please refer to Appendix \ref{sec:appendix-network}.

\paragraph{Recurrent enrichment of positional encoding.} The Cartesian coordinates of the grid, on which the attention operator's input latent representation reside, are concatenated as additional feature dimension(s) to the input, as well as to each latent representation in every attention head.

\paragraph{Problem-dependent decoder.} The decoder is a problem-dependent admissible network that maps the learned representations from the encoder back to the target dimension. For smooth and regular solutions in $H^{1+\alpha}(\Omega)$, we opt for a 2-layer spectral convolution that is the core component in \cite{Li.Kovachki.ea:2021Fourier}. A simple pointwise feedforward neural network (FFN) is used for nonsmooth targets in $L^{\infty}(\Omega)$.

\subsection{Simple self-attention encoder}

\begin{figure}[htp]
\begin{center}
\begin{subfigure}[b]{0.9\linewidth}
  \centering
\includegraphics[width=0.97\linewidth]{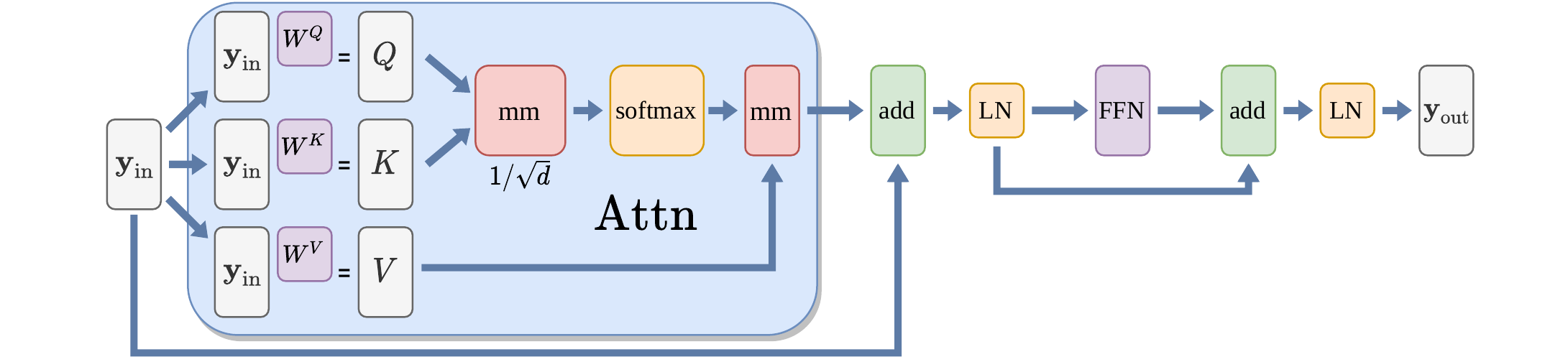}
  \caption{\label{fig:attention-softmax}}
\end{subfigure}%
\\
\begin{subfigure}[b]{0.8\linewidth}
      \centering
      \includegraphics[width=0.97\linewidth]{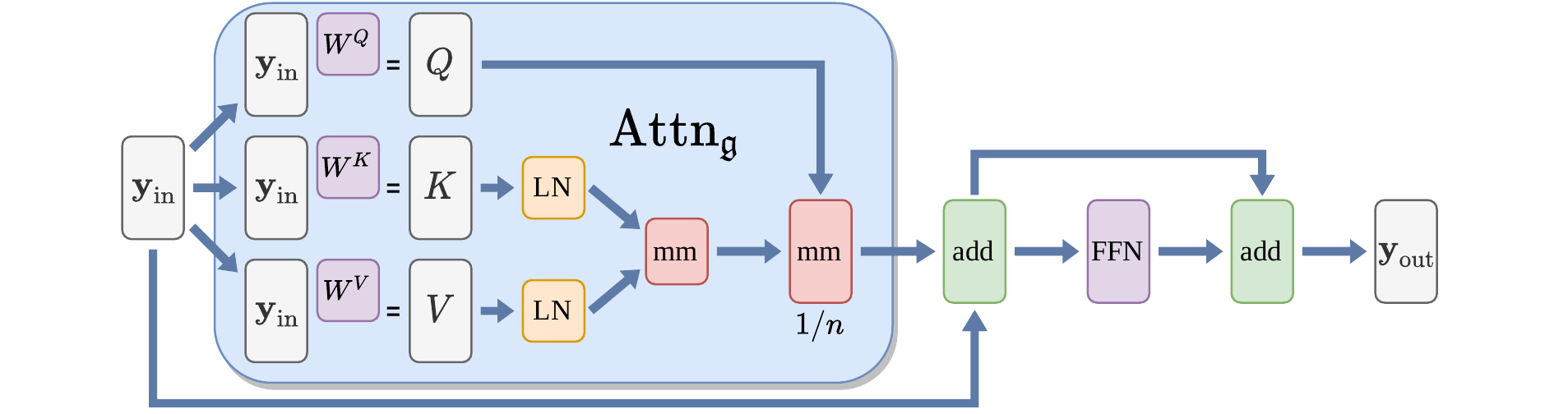}
      \caption{\label{fig:attention-simple}}
\end{subfigure}
\end{center}
\caption{Comparison of the vanilla attention \cite{Vaswani;Shazeer;Parmar:2017Attention} with the Galerkin-type simple self-attention in a single head; (\subref{fig:attention-softmax}) in the standard softmax attention, the softmax is applied row-wise after the matrix product \texttt{matmul}; (\subref{fig:attention-simple}) a mesh-weighted normalization allows an integration-based interpretation.} 
\label{fig:attention}
\end{figure}

The encoder contains a stack of identical simple attention-based encoder layers. For simplicity, we consider a single attention head that 
maps $\mathbf{y}\in \R^{n\times d}$ to another element in $\R^{n\times d}$, and define the trainable projection matrices, and the latent representations $Q/K/V$ as follows.
\begin{equation}
\label{eq:attention-weights}
W^Q, W^K, W^V \in \R^{d\times d}, \quad \text{ and } \quad Q := \mathbf{y} W^Q, 
\quad K := \mathbf{y} W^K,\quad V := \mathbf{y}W^V.
\end{equation}
We propose the following simple attention that (i) uses a mesh (inverse sequence length)-weighted normalization without softmax, (ii) allows a scaling to propagate through the encoder layers. 
\begin{equation}
\label{eq:attention-simple}
\op{Attn}_{\text{sp}}: \R^{n\times d}\to \R^{n\times d},\quad 
\widetilde{\mathbf{y}} \gets \mathbf{y} + \op{Attn}_{\dagger} (\mathbf{y} ), \quad 
\mathbf{y} \mapsto  \widetilde{\mathbf{y}} + g(\widetilde{\mathbf{y}}),
\end{equation}
where the head-wise normalizations are applied pre-dot-product: 
for $\dagger \in \{\mathfrak{f},\mathfrak{g}\}$, 
\begin{align}
\label{eq:attention-fourier}
\text{(Fourier-type attention)} \qquad & \mathbf{z} = \text{Attn}_{\mathfrak{f}} (\mathbf{y}) 
:=  (\widetilde{Q}\widetilde{K}^{\top}) V/n,
\\
\label{eq:attention-galerkin}
\text{(Galerkin-type attention)}\qquad  & \mathbf{z} = \text{Attn}_{\mathfrak{g}} (\mathbf{y}) 
:=  Q(\widetilde{K}^{\top} \widetilde{V})/n,
\end{align}  
and $\widetilde{\diamond}$ denotes a trainable non-batch-based normalization. As in the classic Transformer \cite{Vaswani;Shazeer;Parmar:2017Attention}, and inspired by the Galerkin projection interpretation, we choose $\widetilde{\diamond}$ as the layer normalization $\op{Ln}(\diamond)$, and
$g(\cdot)$ as the standard 2-layer FFN identically applied on every position (grid point). 
In simple attentions, 
the weight for each row of $V$, or column of $Q$ in the linear variant, is not all positive anymore. This can be viewed as a cheap alternative to the cosine similarity-based attention. 

\begin{remark}
If we apply the regular layer normalization rule that eliminates any scaling:
\begin{equation}
\label{eq:attention-ln}
\mathbf{y} \mapsto \op{Ln}\!
\left(\mathbf{y} + \op{Attn}_{\dagger} (\mathbf{y} )+g\big(\op{Ln}(\mathbf{y} 
+ \op{Attn}_{\dagger} (\mathbf{y} ) ) \big)\right),
\;\;\text{where }\op{Attn}_{\dagger}(\mathbf{y}) := Q(K^{\top} V)/n,
\end{equation}
then this reduces to the efficient attention first proposed in \cite{Shen.Zhang.ea:2021Efficient}. 
\end{remark}

\subsubsection{Structure-preserving feature map as a function of positional encodings}
Consider an operator learning problem with an underlying domain $\Omega\subset \R^1$. $\{x_i\}_{i=1}^{n}$ denotes the set of grid points in the discretized $\Omega$ such that the weight $1/n=h$ is the mesh size. Let $\zeta_q(\cdot), \phi_k(\cdot), \psi_v(\cdot) : \Omega \to \R^{1\times d}$ denote the feature maps of $Q, K, V$, i.e., the $i$-th row of $Q, K, V$ written as 
$\bm{q}_i=\zeta_q(x_i)$, $\bm{k}_i=\phi_k(x_i)$, $\bm{v}_i=\psi_v(x_i)$. They are, in the NLP convention, viewed as the feature (embedding) vector at the $i$-th position, respectively. The inter-position topological structure such as continuity/differentiability in the same feature dimension is learned thus not explicit. 
The following ansatz for $Q/K/V$ in the same attention head is fundamental to our new interpretations.
\begin{assumption}
\label{assumption:attn}
The columns of $Q/K/V$, respectively, contain the vector representations of the learned basis functions spanning certain subspaces of the latent representation Hilbert spaces.
\end{assumption}

Using $V \in \R^{n\times d}$ with a full column rank as an example, its columns contain potentially a set of bases $\{v_j(\cdot)\}_{j=1}^{d}$ evaluated at the grid points (degrees of freedom, or DoFs). 
Similarly, the learned bases whose DoFs form the columns of $Q, K$ are denoted 
as $\{q_j(\cdot)\}_{j=1}^{d}$, $\{k_j(\cdot)\}_{j=1}^d$, as well as $\{z_j(\cdot)\}_{j=1}^d$ for the outputs in \eqref{eq:attention-fourier} and \eqref{eq:attention-galerkin}. 
To be specific, the $j$-th column of $V$, denoted by $\bm{v}^j$, then stands for a vector representation of the $j$-th basis function evaluated at each grid point, i.e., its $l$-th position stands for $(\bm{v}^j)_l =  v_j(x_l)$.
Consequently, the row $\bm{v}_i= (v_1(x_i), \dots, v_d(x_i))$ can be alternatively viewed as the evaluation of a vector latent basis function at $x_i$.

\subsubsection{Fourier-type attention of a quadratic complexity}
\begin{figure}[htbp]
\centering
\includegraphics[width=0.8\textwidth]{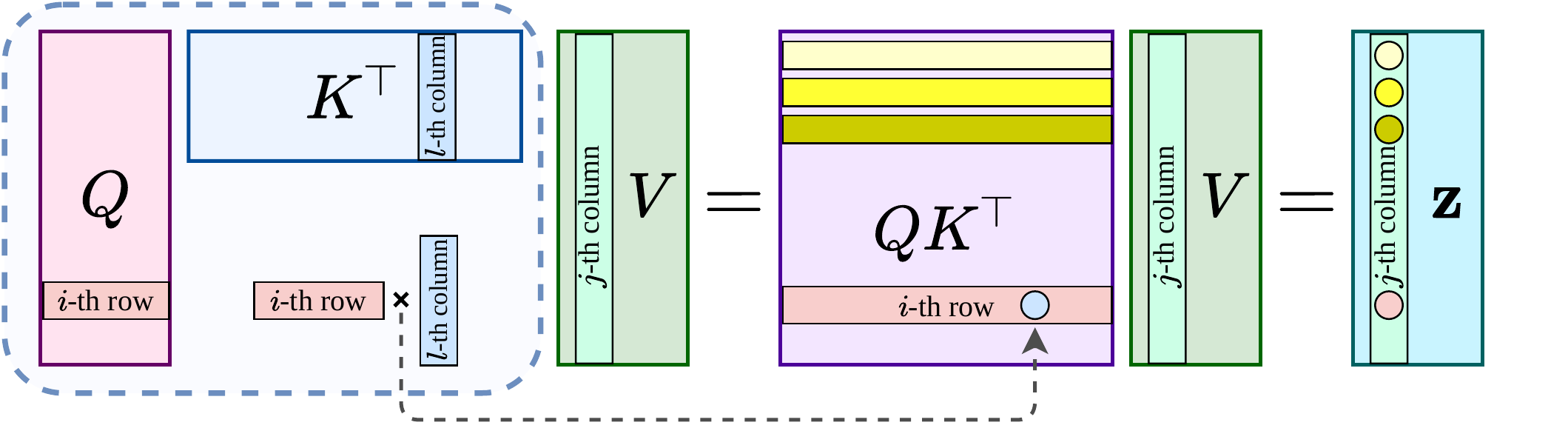}
\caption{A dissection of Fourier-type attention's output. Both \texttt{matmul}s have complexity $O(n^2d)$.}
\label{fig:attention-f}
\end{figure}

In the Fourier-type attention \eqref{eq:attention-fourier}, $Q,K$ are assumed to be normalized for simplicity, the $j$-th column ($1\leq j \leq d $) in the $i$-th row ($1\leq i \leq n$) of $\mathbf{z}$ is computed by (see Figure \ref{fig:attention-f}): 
\begin{equation}
\begin{aligned}
(\bm{z}_i)_{j} &=  h(QK^{\top})_{i\hspace*{0.07em}\bdot} \; \bm{v}^j
= h\big(\bm{q}_i\cdot \bm{k}_1, \ldots, \bm{q}_i \cdot \bm{k}_l, \ldots, \bm{q}_i \cdot \bm{k}_n \big)^{\top} 
\cdot \bm{v}^j
\\
& = h\sum_{l=1}^n (\bm{q}_i\cdot\bm{k}_l) (\bm{v}^j)_l
\approx \int_{\Omega} \big(\zeta_q(x_i) \cdot \phi_k(\xi)\big) v_j(\xi) \dd \xi,
\end{aligned}
\end{equation}
where the $h$-weight facilitates the numerical quadrature interpretation of the inner product.
Concatenating columns $1\leq j\leq d$ yields the $i$-row $\bm{z}_i$ of the output $\mathbf{z}$: $\bm{z}_i \approx \int_{\Omega} \big(\zeta_q(x_i) \cdot \phi_k(\xi)\big) \psi_v(\xi) \dd \xi$.
Therefore, without the softmax nonlinearity, the local dot-product attention output at $i$-th row computes approximately an integral transform with a non-symmetric learnable kernel function $\kappa(x,\xi):= \zeta_q(x)\phi_k(\xi)$ evaluated at $x_i$, whose approximation property has been studied in \cite[Theorem 2]{Wright.Gonzalez:2021Transformers}, yet without the logits technicality due to the removal of the softmax normalization. 

After the skip-connection, if we further exploit the learnable nature of the method and assume $W^V = \op{diag}\{\delta_1, \cdots, \delta_d\}$ such that $\delta_j\neq 0$ for $1\leq j\leq d$, 
under Assumption \ref{assumption:attn}:
\vspace*{-3pt}
\begin{equation}
\label{eq:fredholm}
\delta_j^{-1} v_j(x)\approx 
z_j(x) - \int_{\Omega} \kappa(x,\xi) v_j(\xi) \dd \xi, \quad  \text{for}\; j=1,\cdots,d,\;\text{ and } x\in \{x_i\}_{i=1}^n.
\vspace*{-3pt}
\end{equation}
This is the forward propagation of the Fredholm equation of the second-kind for each $v_j(\cdot)$. When using an explicit orthogonal expansion such as Fourier to solve for $\{v_j(\cdot)\}_{j=1}^d$, or to seek for a better set of $\{v_j(\cdot)\}$ in our case, it is long known being equivalent to the Nystr\"{o}m’s method with numerical integrations \cite{Berrut.Trummer:1987Equivalence} (similar to the $h=1/n$ weighted sum). Therefore, the successes of the random Fourier features in \cite{Choromanski.Likhosherstov.ea:2021Rethinking,PengPappasYogatamaSchwartzEtAl2021Random} and the Nystr\"{o}mformer's approximation \cite{Xiong.Zeng.ea:2021Nystromformer} are not surprising.

Finally, we name this type of simple attention ``Fourier'' is due to the striking resemblance between the scaled dot-product attention 
and a Fourier-type kernel \cite{Fox:1961functions} integral transform, since eventually the target resides in a Hilbert space with an underlying spacial domain $\Omega$, while the latent representation space parallels a ``frequency'' domain on $\Omega^*$. 
This also bridges the structural similarity of the scaled dot-product attention with the Fourier Neural Operator \cite{Li.Kovachki.ea:2021Fourier} where the Fast Fourier Transform (FFT) can be viewed as a non-learnable change of basis.

\subsubsection{Galerkin-type attention of a linear complexity}
\label{sec:galerkin}
\begin{figure}[htbp]
\centering
\includegraphics[width=0.7\textwidth]{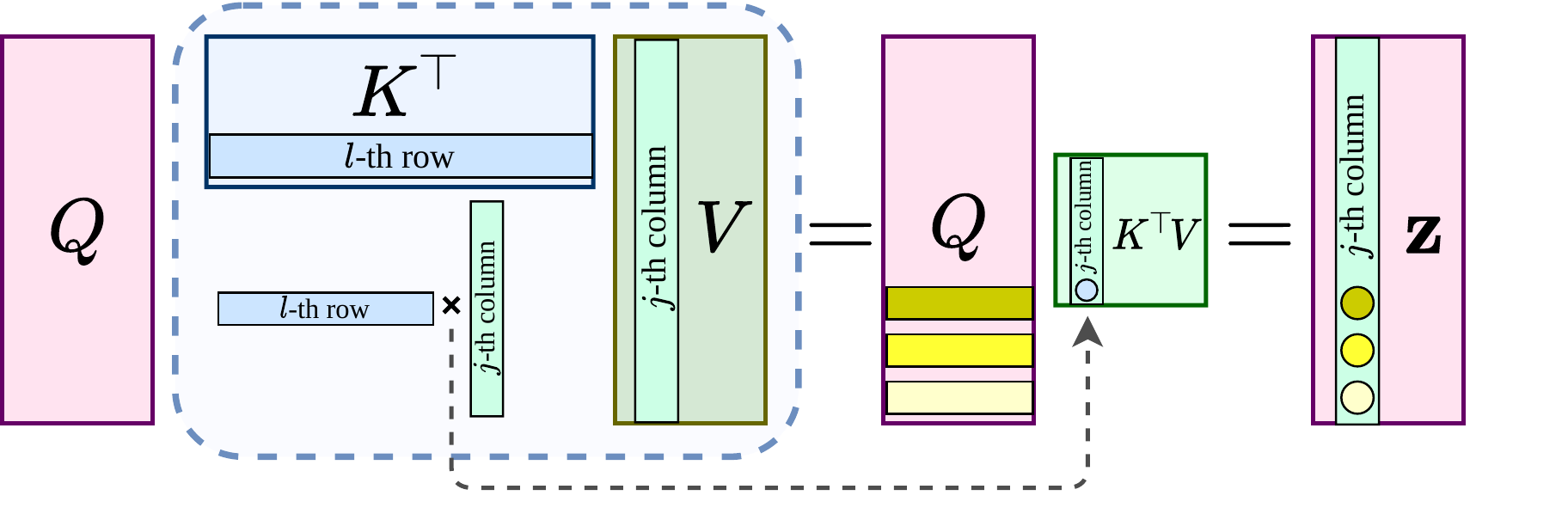}
\caption{A dissection of Galerkin-type attention's output. Both \texttt{matmul}s have complexity $O(nd^2)$.}
\label{fig:attention-g}
\end{figure}
For the Galerkin-type simple attention in \eqref{eq:attention-galerkin}, $K,V$ are assumed to be normalized for simplicity, we first consider the $i$-th entry in the $j$-th column $\bm{z}^j$ of $\mathbf{z}$  (see Figure \ref{fig:attention-g}): 
\begin{equation}
(\bm{z}^j)_i =h\, \bm{q}_i^{\top}\cdot (K^{\top}V)_{\bdot\hspace*{0.03em} j},
\end{equation}
which is the inner product of the $i$-th row of $Q$ and the $j$-th column of $K^{\top}V$. Thus,
\begin{equation}
  \bm{z}^j = h
\left(
\begin{array}{cccc}| & | & | & | \\
\bm{q}_1 & \bm{q}_2 & \cdots & \bm{q}_n \\
| & | & | & |
\end{array}\right)^{\top} 
\; (K^{\top}V)_{\bdot\hspace*{0.03em} j}
= h
\left(
(K^{\top}V)_{\bdot\hspace*{0.03em} j}^{\top} \begin{pmatrix}
\rule[.4ex]{2em}{0.2pt} & \bm{q}^1  &\rule[.4ex]{2em}{0.2pt}
\\[-4pt]
\rule[.8ex]{2em}{0.3pt} & \vdots & \rule[.8ex]{2em}{0.3pt}
\\[-1pt]
\rule[.5ex]{2em}{0.2pt}&  \bm{q}^d &  \rule[.5ex]{2em}{0.2pt}
\end{pmatrix}
\right)^{\top}
\end{equation}
This reads as: $(K^{\top}V)_{\bdot\hspace*{0.03em} j}$ contains the coefficients for the linear combination of the vector representations $\{\bm{q}^l\}_{l=1}^d$ of the bases stored in $Q$'s column space to form the output $\mathbf{z}$. Meanwhile, the $j$-th column $(K^{\top}V)_{\bdot\hspace*{0.03em} j}$ of $K^{\top}V$ consists the inner product of $j$-th column of $V$ with every column of $K$.
\vspace*{-2pt}
\begin{equation}
  \bm{z}^j = h\sum_{l=1}^d \bm{q}^l (K^{\top}V)_{lj}, 
\quad \text{where }\; (K^{\top}V)_{\bdot\hspace*{0.03em} j}
= \big(\bm{k}^1\cdot \bm{v}^j,  \bm{k}^2\cdot \bm{v}^j , \cdots, 
\bm{k}^d\cdot \bm{v}^j\big)^{\top}.
\vspace*{-2pt}
\end{equation}
As a result, using Assumption \ref{assumption:attn}, and for simplicity the latent Hilbert spaces $\mathcal{Q},\mathcal{K},\mathcal{V}$ are assumed to be defined on the same spacial domain $\Omega$, i.e., $k_l(\cdot)$, $v_j(\cdot)$ evaluated at every $x_i$ are simply their vector representations $\bm{k}^l$ ($1\leq l \leq d$) and $\bm{v}^j$, we have the functions represented by the columns of the output $\mathbf{z}$ can be then compactly written as:
rewriting $\langle v_j, k_l \rangle:= (K^{\top}V)_{lj}$
\vspace*{-3pt}
\begin{equation}
\label{eq:attention-galerkin-int}
  z_j(x) := \sum_{l=1}^d \langle v_j, k_l \rangle \, q_l(x), 
\; \;\text{for}\; j=1,\cdots,d,\; \text{ and } x\in \{x_i\}_{i=1}^n,
\vspace*{-3pt}
\end{equation}
where the bilinear form $\langle \cdot, \cdot \rangle: \mathcal{V}\times \mathcal{K}\to \R$. \eqref{eq:attention-galerkin-int} can be also written in a componentwise form:
\vspace*{-3pt}
\begin{equation}
\label{eq:attention-galerkin-projection}
   z_j(x_i) := (\bm{z}^j )_i = h \sum_{l=1}^d  (\bm{k}^l\cdot \bm{v}^j) (\bm{q}^l)_i
\approx  \sum_{l=1}^d \left( \int_{\Omega}  v_j(\xi) k_l(\xi) \dd \xi \right) q_l(x_i).
\vspace*{-3pt}
\end{equation}
Therefore, when $\{\diamond_j(\cdot)\}_{j=1}^d$, $\diamond\in \{q, k, v\}$ consist approximations to three sets of bases for potentially different subspaces, 
and if we set the trial spaces as the column spaces of $Q$ and the test space as that of $K$, respectively, the forward propagation of the Galerkin-type attention is a recast of a learnable Petrov--Galerkin-type
projection (cf. Appendix \ref{sec:appendix-cea-background}) for every basis represented by the columns of $V$. While the form of \eqref{eq:attention-galerkin-projection} 
suggests the orthonormality of the basis represented by $Q,K,V$, as well as
being of full column ranks, the learnable nature of the method suggests otherwise (see Appendix \ref{sec:appendix-cea}). 
At last, we have the following strikingly simple yet powerful approximation result.

\begin{theorem}[C\'{e}a-type lemma, simplified version]
\label{theorem:cea-lemma}
Consider a Hilbert space $\mathcal{H}$ defined on a bounded domain $\Omega\subset \R^m$ discretized by $n$ grid points, and $f\in \mathcal{H}$.  $\mathbf{y}\in \R^{n\times d}$ is the current latent representation for $n>d>m$ and full column rank. 
$\bb{Q}_h\subset \mathcal{Q}\subset \mathcal{H}$ 
and $\bb{V}_h\subset\mathcal{V}\subset \mathcal{H}$ are the latent approximation subspaces spanned by basis functions with the columns of 
$Q$ and $V$ in \eqref{eq:attention-weights} as degrees of freedom, respectively, and $0<\dim \bb{Q}_h=r\leq \dim\bb{V}_h=d$. Let $\mathfrak{b}(\cdot, \cdot): \mathcal{V}\times \mathcal{Q}\to \R$ be a continuous bilinear form, and  
if for any fixed $q\in \bb{Q}_h$ the functional norm of $\mathfrak{b}(\cdot, q)$ is bounded below by $c>0$, then there exists  a learnable map $g_{\theta}(\cdot)$ that is the composition of the Galerkin-type attention operator  with an updated set of projection matrices $\{W^Q, W^K, W^V\}$, and a pointwise universal approximator, such that for $f_h\in \bb{Q}_h$ being the best approximation of $f$ in $\|\cdot\|_{\mathcal{H}}$ it holds:
\begin{equation}
\label{eq:approximation-cea} 
\|f- g_{\theta}(\mathbf{y})\|_{\mathcal{H}}
\leq c^{-1} \min_{q\in \bb{Q}_h} \max_{v\in \bb{V}_h}
\frac{|\mathfrak{b}(v,f_h - q)|}{\|v\|_{\mathcal{H}}}
+\|f-f_h\|_{\mathcal{H}}.
\end{equation}
\end{theorem}
\paragraph{Remarks on and interpretations of the best approximation result.}
\label{paragraph:approximation}
Theorem \ref{theorem:cea-lemma} states that the Galerkin-type attention has the architectural capacity to represent a quasi-optimal approximation in $\|\cdot\|_{\mathcal{H}}$ in the current subspace $\bb{Q}_h$. For the
mathematically rigorous complete set of notations and the full details of the proof we refer the readers to Appendix \ref{sec:appendix-cea-proof}. Even though Theorem \ref{theorem:cea-lemma} is presented for a single instance of $f\in \mathcal{H}$ for simplicity, the proof shows that the attention operator is fully capable of simultaneously approximating a collection of functions (Appendix \hyperref[paragraph:dynamical-basis-update]{D.3.4}).

Estimate \eqref{eq:approximation-cea} comes with great scalability with respect to the sequence length in that it all boils down to whether $c$ is independent of $n$ in the lower bound of $\|\mathfrak{b}(\cdot, q)\|_{\bb{V}_h'}$. The existence of an $n$-independent lower bound is commonly known as the discrete version of the Ladyzhenskaya–Babu\v{s}ka–Brezzi (LBB) condition \cite[Chapter 6.12]{Ciarlet:2013Linear}, also referred as the Banach-Ne\v{c}as-Babu\v{s}ka (BNB) condition in Galerkin methods on Banach spaces \cite[Theorem 2.6]{Ern.Guermond:2004Theory}. 

As the cornerstone of 
the approximation to many PDEs, the discrete LBB condition establishes the surjectivity of a map from $\bb{V}_h$ to $\bb{Q}_h$. In a simplified context \eqref{eq:approximation-cea} above of approximating functions using this linear attention variant ($Q$: values, query, $V$: keys), it roughly translates to: for an incoming ``query'' (function $f$ in a Hilbert space), to deliver its best approximator in ``value'' (trial function space), the ``key'' (test function space) has to be sufficiently rich such that there exists a key to unlock every possible value. 

\paragraph{Dynamic basis update.}
Another perspective is to interpret the Galerkin-type dot-product attention \eqref{eq:attention-galerkin-projection} as a change of basis: essentially, the new set of basis is the column space of $Q$, and
how to linearly combine the bases in $Q$ is based on the inner product (response) of the corresponding feature dimension's basis in $V$ against every basis in $K$.
From this perspective ($Q$: values, $K$: keys, $V$: query), we have the following result of a layer-wise dynamical change of basis: through testing against the ``keys'', a latent representation is sought such that ``query'' (input trial space) and ``values'' (output trial space) can achieve the minimum possible difference under a functional norm; for details and the proof please refer to Appendix \hyperref[paragraph:dynamical-basis-update]{D.3.4}. 

\begin{theorem}[layer-wise dynamic basis update, simple version]
Under the same assumption as Theorem \ref{theorem:cea-lemma}, it is further assumed that $\mathfrak{b}(\cdot, q)$ is bounded below on $\bb{K}_h\subset\mathcal{K}=\mathcal{V}\subset \mathcal{H}$ and $\mathfrak{a}(\cdot, \cdot): \mathcal{V}\times \mathcal{K}\to \R$ is continuous. Then, there exists a set of projection matrices to update the value space $\{\tilde{q}_l(\cdot)\}_{l=1}^d\subset \bb{Q}_h=\operatorname{span}\{q_l(\cdot)\}_{l=1}^d$, for $z_j \in \bb{Q}_h$ ($j=1,\cdots, d$) obtained through the basis update rule \eqref{eq:attention-galerkin-projection}, it holds
\begin{equation}
 \big\| \mathfrak{a}(v_{j}, \cdot) - \mathfrak{b}(\cdot, z_j)\big\|_{\bb{K}_h'}
\leq \min_{q\in \bb{Q}_h} \max_{k\in \bb{K}_h} 
\frac{\left|\mathfrak{a}(v_{j}, k) - \mathfrak{b}(k, q) \right|}{\norm{k}_{\mathcal{K}}}.
\end{equation}
\end{theorem}

\paragraph{The role of feed-forward networks and positional encodings in the dynamic basis update.}  
\label{paragraph:ffn}
Due to the presence of the concatenated coordinates $\mathbf{x}:=\Vert_{i=1}^n x_i \in \R^{n\times m} $ to the latent representation $\mathbf{y}$, the pointwise subnetwork $g_{s}(\cdot):\R^{n\times m} \to \R^{n\times d}$ of the nonlinear universal approximator (FFN) in each attention block is one among many magics of the attention
mechanism. In every attention layer, the basis functions in $\bb{Q}_h/\bb{K}_h/\bb{V}_h$ are being constantly enriched by  
$\op{span}\{w_j\in \bb{X}_h: w_j(x_i) = (g_s(\mathbf{x}))_{ij}, 1\leq j\leq d \} \subset \mathcal{H}$, thus being dynamically updated to try to capture how an operator of interest responses to the subset of inputs. 
Despite the fact that the FFNs, when being viewed as a class of functions, bear no linear structure within, the basis functions produced this way act as a building block to characterize a linear space for a learnable projection. This heuristic shows to be effective when the target is assumed to be a function of the (relative) positional encodings (coordinates, transformation groups, etc.), in that this is incorporated in many other attention-based learners with applications in physical sciences \cite{Tai.Bailis.ea:2019Equivariant,Fuchs.Worrall.ea:2020SE3,Hutchinson.Lan.ea:2021LieTransformer,JumperEvansPritzelGreenEtAl2021Highly}.

\section{Experiments}
\label{sec:experiments}
In this section we perform a numerical study the proposed Fourier Transformer (\textbf{FT}) with the Fourier-type encoder, and the Galerkin Transformer (\textbf{GT}) with the Galerkin-type encoder, in various PDE-related operator learning tasks. The models we compare our newly proposed models with 
are the operator learners with the simple attention replaced by the standard softmax normalized scaled 
dot-product attention (\textbf{ST}) \cite{Vaswani;Shazeer;Parmar:2017Attention},
and a linear variant (\textbf{LT}) \cite{Shen.Zhang.ea:2021Efficient} in which two independent softmax normalizations are applied on $Q, K$ separately.\footnote{\url{https://github.com/lucidrains/linear-attention-transformer}}
The data are obtained courtesy of the PDE benchmark under the MIT license.\footnote{\url{https://github.com/zongyi-li/fourier_neural_operator} } 
For full details of the training/evaluation and model structures please refer to Appendix \ref{sec:appendix-experiments}.

Instead of the standard Xavier uniform initialization \cite{Glorot.Bengio:2010Understanding}, inspired by the interpretations of Theorem \ref{theorem:cea-lemma} in Appendix \hyperref[paragraph:dynamical-basis-update]{D.3.4}, we modify the initialization for the projection matrices slightly as follows
\begin{equation}
\label{eq:attention-init}
W^{\diamond}_{\text{init}}\gets \eta U + \delta I, \quad \text{ for } \diamond \in \{Q, K, V\},
\end{equation}
where $U=(x_{ij})$ is a random matrix using the Xavier initialization with gain $1$ such that $x_{ij}\sim \mathcal{U}([-\sqrt{3/d}, \sqrt{3/d}])$, and $\delta$ is a small positive number.
In certain operator learning tasks, we found that this tiny modification boosts the evaluation performance of models by up to $50\%$ (see Appendix \ref{sec:appendix-burgers}) and improves the training stability acting as a cheap remedy to the lack of a softmax normalization. We note that similar tricks have been discovered concurrently in \cite{CsordasIrieSchmidhuber2021Devil}.

Unsurprisingly, when compared the memory usage and the speed of the networks (Table \ref{table:memory-speed}), the Fourier-type attention features a 40\%--50\% reduction in memory versus the attention with a softmax normalization. 
The Galerkin attention-based models have a similar memory profile with the standard linear attention, it offers up to a 120\% speed boost over the linear attention in certain tests. 

\begin{table}[htbp]
\caption{The memory usage/FLOP/complexity comparison of the models. Batch size: 4; 
the CUDA mem (GB): the sum of the \texttt{self\_cuda\_memory\_usage}; GFLOP: Giga FLOP for 1 backpropagation (BP); both are from the PyTorch \texttt{autograd} profiler for 1 BP averaging from 1000 BPs; the mem (GB) is recorded from \texttt{nvidia-smi} of the memory allocated for the active Python process during profiling; the speed (iteration per second) is measured during training; the exponential operation is assumed to have an explicit complexity of $c_e>1$ \cite{Brent1976Multiple}.}
\label{table:memory-speed}
\begin{center}
  \resizebox{0.98\textwidth}{!}{%
  \begin{tabular}{rccccccccc}\toprule
& \multicolumn{4}{c}{Example 1: $n=8192$} 
& \multicolumn{4}{c}{Encoders only: $n=8192, d=128, l=10$}
& \multirowcell{2}{Computational complexity \\ of the dot-product per layer}
\\
\cmidrule(lr){2-5}
\cmidrule(lr){6-9}
&  Mem  & CUDA Mem & Speed &  GFLOP 
&  Mem  &  CUDA Mem  & Speed & GFLOP
\\
\midrule
ST   & $18.39$  & $31.06$  & $5.02$ & $1393$ & $18.53$ & $31.34$  & $ 4.12$  & $1876$ & $O(n^2 c_e d)$
\\
FT  & $10.05$ & $22.92$  & $6.10$   & $1138$ &  $10.80$ & $22.32$  &  $ 5.46$ & $1610$ & $O(n^2 d)$ 
\\
LT  & $2.55$ & $2.31$  &  $12.70$   & $606$ & $2.73$ & $2.66$  &  $ 10.98$ & $773$ & $O(n(d^2+c_e d))$
\\
GT &  \b{2.36} & \b{1.93}  & \b{27.15}  & \b{275} & \b{2.53} & \b{2.33}  &  \b{19.20} & \b{412} & $O(nd^2)$
\\
\bottomrule
\end{tabular}
}
\end{center}
\end{table}

The baseline models for each example are the best operator learner to-date, the state-of-the-art Fourier Neural Operator (FNO) in \cite{Li.Kovachki.ea:2021Fourier} but without the original built-in batch normalization. All attention-based models match the parameter quota of the baseline, and are trained using the loss in \eqref{eq:loss-discrete} with the same \texttt{1cycle} scheduler \cite{Smith.Topin:2019Super} for 100 epochs. For fairness, we have also included the results for the standard softmax normalized models (ST and LT) using the new layer normalization scheme in \eqref{eq:attention-fourier} and \eqref{eq:attention-galerkin}.
We have retrained the baseline with the same \texttt{1cycle} scheduler using the code provided in \cite{Li.Kovachki.ea:2021Fourier}, and listed the original baseline results 
using a step scheduler of 500 epochs of training from \cite{Li.Kovachki.ea:2021Fourier} Example 5.1 and Example 5.2, respectively.

\subsection{Example 1: viscous Burgers' equation}
\label{sec:burgers}
In this example, we consider a benchmark problem of the viscous Burgers' equation with a periodic boundary condition on $\Omega :=(0,1)$ in \cite{Li.Kovachki.ea:2021Fourier}. The nonlinear operator to be learned is the discrete approximations to the solution operator $T: C^0_{p} (\Omega) \cap L^2(\Omega) \to C^0_{p} (\Omega) \cap H^1 (\Omega)$, $u_0(\cdot) \mapsto u(\cdot, 1)$. The initial condition $u_0(\cdot)$'s are sampled following a Gaussian Random Field (GRF).

The result can be found in Table \ref{table:burgers}. 
All attention-based operator learners achieve a resolution-invariant performance similar with FNO1d in \cite{Li.Kovachki.ea:2021Fourier}. The new Galerkin projection-type layer normalization scheme significantly outperforms the regular layer normalization rule in this example, in which both inputs and targets are unnormalized. For full details please refer to Appendix \ref{sec:appendix-burgers}.

\subsection{Example 2: Darcy flow}
\label{sec:darcy}
In this example, we consider another well-known benchmark $-\nabla \cdot (a\nabla u) = f$ for $u\in H^1_0(\Omega)$ from \cite{Bhattacharya.Hosseini.ea:2020Model,Li.Kovachki.ea:2021Fourier,Li.Kovachki.ea:2020Multipole,Nelsen.Stuart:2020Random}, and the operator to be learned is the approximations to $T: L^\infty(\Omega) \to H^1_0 (\Omega), a\mapsto u$, in which $a$ is the coefficient with a random interface geometry, and $u$ is the weak solution. 
Here $L^{\infty}(\Omega)$ is a Banach space and cannot be compactly embedded in $L^2(\Omega)$ (a Hilbert space), we choose to avoid this technicality as the finite dimensional approximation space can be embedded in $L^2(\Omega)$ given that $\Omega$ is compact.

The result can be found in Table \ref{table:darcy}. As the input/output are normalized, in contrast to Example \ref{sec:burgers}, the Galerkin projection-type layer normalization scheme does not significantly outperform the regular layer normalization rule in this example. 
The attention-based operator learners achieve on average 30\% to 50\% better evaluation results than the baseline FNO2d (only on the fine grid) using the same trainer.
For full details please refer to Appendix \ref{sec:appendix-darcy}.

\begin{table}[htbp]
\caption{(\subref{table:burgers}) Evaluation relative error ($\times 10^{-3}$) of Burgers' equation \ref{sec:burgers}. (\subref{table:darcy}) Evaluation relative error ($\times 10^{-2}$) of Darcy interface problem \ref{sec:darcy}. }
\begin{subtable}[t]{0.48\linewidth}
\caption{\label{table:burgers}}
\resizebox*{0.99\textwidth}{!}{%
  \begin{tabular}{lcccc}\toprule
&  $n=512$ &  $n=2048$ 
&  $n=8192$ 
\\
\midrule
FNO1d \cite{Li.Kovachki.ea:2021Fourier}  & 15.8 &   14.6 &    13.9   
\\
FNO1d \texttt{1cycle}     & $4.373$  & $4.126$   &   $4.151$
\\
FT regular $\text{Ln}$  & $1.400$   & $1.477$  & $1.172$  
\\
GT regular $\text{Ln}$ &  $2.181$   & $1.512$   & $2.747$  
\\
ST regular $\text{Ln}$  & $1.927$  & $2.307$   & $1.981$  
\\
LT regular $\text{Ln}$  & $1.813$  & $1.770$   & $1.617$  
\\
FT $\text{Ln}$ on $Q,K$ & \b{1.135} & \b{1.123}   & \b{1.071}  
\\
GT $\text{Ln}$ on $K,V$ & \b{1.203}  & \b{1.150}   & \b{1.025} 
\\
ST $\text{Ln}$ on $Q,K$ & $1.271$ & $1.266$   & $1.330$ 
\\
LT $\text{Ln}$ on $K,V$ & $1.139$ & $1.149$  & $1.221$ 
\\
\bottomrule
\end{tabular}
}
\end{subtable}
\begin{subtable}[t]{0.505\linewidth}
\caption{\label{table:darcy}}
\resizebox*{0.99\textwidth}{!}{%
  \begin{tabular}{lcccc}\toprule
&  $n_f, n_c = 141, 43$
&   $n_f, n_c=211, 61$
\\
\midrule
FNO2d \cite{Li.Kovachki.ea:2021Fourier}   & $1.09$  & $1.09$ 
\\
FNO2d \texttt{1cycle}         & $1.419$  & $1.424$  
\\
FT regular $\text{Ln}$     & \b{0.838}  &  \b{0.847}     
\\
GT regular $\text{Ln}$     & $0.894$ &  $0.856$      
\\
ST regular $\text{Ln}$     & $1.075$ &  $1.131$   
\\
LT regular $\text{Ln}$       & $1.024$ & $1.130$      
\\
FT  $\text{Ln}$ on $Q,K$    & $0.873$ &   $0.921$   
\\
GT $\text{Ln}$ on $K,V$     & \b{0.839} &   \b{0.844} 
\\
ST $\text{Ln}$ on $Q,K$       & $0.946$    &  $0.959$  
\\
LT $\text{Ln}$ on $K,V$     & $0.875$    &    $0.970$
\\
\bottomrule
\end{tabular}
}
\end{subtable} 
\end{table}

\subsection{Example 3: inverse coefficient identification for Darcy flow}
\label{sec:darcy-inverse}
In this example, we consider an inverse coefficient identification problem based on the same data used in Example \ref{sec:darcy}. The input (solution) and the target (coefficient) are reversed from Example \ref{sec:darcy}, and the noises are added to the input. The inverse problems in practice are a class of important tasks in many scientific disciplines such as geological sciences and medical imaging but much more difficult due to poor stability \cite{Kirsch:2011Introduction}. We aim to learn an approximation to an ill-posed operator $T: H^1_0 (\Omega) \to L^{\infty}(\Omega), u+\epsilon N_\nu(u)\mapsto a$, where $N_\nu(u)$ stands for noises related to the sampling distribution and the data. $\epsilon=0.01$ means 1\% of noise added in both training and evaluation, etc.

The result can be found in Table \ref{table:darcy-inv}. It is not surprising that FNO2d, an excellent smoother which filters higher modes in the frequency domain, struggles in this example to recover targets consisting of high-frequency traits (irregular interfaces) from low-frequency prevailing data (smooth solution due to ellipticity).
We note that, the current state-of-the-art methods \cite{Chan.Tai:2003Identification} for inverse interface coefficient identification need to carry numerous iterations to recover a single instance of a simple coefficient with a regular interface, provided that a satisfactory denoising has done beforehand. 
The attention-based operator learner has capacity to unearth structurally how this inverse operator's responses on a subset, with various benefits articulated in \cite{Li.Kovachki.ea:2020Multipole,Li.Kovachki.ea:2021Fourier,Li.Kovachki.ea:2020Neural,Nelsen.Stuart:2020Random,Bhattacharya.Hosseini.ea:2020Model}.

\begin{table}[htbp]
\caption{Evaluation relative error ($\times 10^{-2}$) of the inverse problem \ref{sec:darcy-inverse}.}
\label{table:darcy-inv}
\begin{center}
\resizebox{0.7\textwidth}{!}{%
\begin{tabular}{rcccccc}\toprule
& \multicolumn{3}{c}{$n_f, n_c = 141, 36$} & \multicolumn{3}{c}{$n_f, n_c=211, 71$}
\\
\cmidrule(lr){2-4}
\cmidrule(lr){5-7}
& $\epsilon=0$ & $\epsilon=0.01$ & $\epsilon=0.1$ 
& $\epsilon=0$ & $\epsilon=0.01$ & $\epsilon=0.1$ 
\\
\midrule
FNO2d (only $n_f$)     &  $13.71$   &  $13.78$   &  $15.12$  &  $13.93$  &  $13.96$ & $15.04$
\\
FNO2d (only $n_c$)      &  $14.17$   &  $14.31$  & $17.30$ &  $13.60$  &  $13.69$  & $16.04$
\\
FT regular $\text{Ln}$  & \b{1.799} & \b{2.467}  & $6.814$ & $1.563$  & $2.704$   & $8.110$
\\
GT regular $\text{Ln}$  & $2.026$ &  $2.536$  &  \b{6.659} & $1.732$  &  $2.775$  & $8.024$ 
\\
ST regular $\text{Ln}$ & $2.434$ &   $3.106$  &  $7.431$ & $2.069$  & $3.365$ & $8.918$  
\\
LT regular $\text{Ln}$  & $2.254$ & $3.194$  &  $9.056$ & $2.063$  &   $3.544$ & $9.874$
\\
FT $\text{Ln}$ on $Q,K$ & $1.921$  &  $2.717$  & $6.725$  &  \b{1.523} &  \b{2.691} & $8.286$
\\
GT $\text{Ln}$ on $K,V$ & $1.944$  & $2.552$  & $6.689$ &  $1.651$ &  $2.729$  & \b{7.903} 
\\
ST $\text{Ln}$ on $Q,K$  & $2.160$ & $2.807$  & $6.995$  &   $1.889$   & $3.123$ & $8.788$
\\
LT $\text{Ln}$ on $K,V$  & $2.360$ & $3.196$   & $8.656$  & $2.136$  & $3.539$ & $9.622$    
\\
\bottomrule
\end{tabular}
}
\end{center}
\end{table}

\section{Conclusion}
\label{sec:conclusion}

We propose a general operator learner based on a simple attention mechanism. The network is versatile and is able to approximate both the PDE solution operator and the inverse coefficient identification operator. The evaluation accuracy on the benchmark problems surpasses the current best state-of-the-art operator learner Fourier Neural Operator (FNO) in \cite{Li.Kovachki.ea:2021Fourier}. 
However, we acknowledge the limitation of this work: (i) similar to other operator learners, the subspace, on which we aim to learn the operator's responses, may be infinite dimensional, 
but the operator must exhibit certain low-dimensional attributes (e.g., smoothing property of the higher frequencies in GRF); 
(ii) it is not efficient for the attention operator to be applied at the full resolution for a 2D problem, and this limits the approximation to a nonsmooth subset such as functions in $L^{\infty}$;
(iii) due to the order of the matrix product, the proposed linear variant of the scaled dot-product attention is non-causal thus can only apply to encoder-only applications.

\section{Broader Impact}
\label{sec:impact}
Our work introduces the state-of-the-art self-attention mechanism the first time to PDE-related operator learning problems. The new interpretations of attentions invite numerical analysts to work on a more complete and delicate approximation theory of the attention mechanism. We have proved the Galerkin-type 
attention's approximation capacity in an ideal Hilbertian setting. Numerically, the new attention-based operator learner has capacity to approximate the difficult inverse coefficient identification problem with an extremely noisy measurements,
which was not attainable using traditional iterative methods for nonlinear mappings. Thus, our method may pose a huge positive impact in geoscience, medical imaging, etc.
Moreover, traditionally the embeddings in Transformer-based NLP models map the words to a high dimensional space, but the topological structure in the same feature dimension between different positions are learned thereby not efficient. Our proof provides a theoretical guide for the search of feature maps that preserve, or even create, structures such as differentiability or physical invariance. Thus, it may contribute to the removal of the softmax nonlinearity to speed up significantly the arduous training or pre-training of larger encoder-only models such as BERT \cite{DevlinChangLeeToutanova2019BERT}, etc.
However, we do acknowledge that our research may negatively impact on the effort of building a cleaner future for our planet, as inverse problems are widely studied in reservoir detection, and we have demonstrated that the attention-based operator learner could potentially help to discover new fossil fuel reservoirs due to its capacity to infer the coefficients from noisy measurements.

\newpage
\begin{ack}
The hardware to perform this work is kindly donated by Andromeda Saving Fund.
The first author was supported in part by the National Science Foundation
under grants DMS-1913080 and DMS-2136075. No additional revenues are related to this work. We would like to thank the anonymous reviewers and the area chair for the suggestions on improving this article. We would like to thank Dr. Long Chen (Univ of California Irvine) for the inspiration of and encouragement on the initial conceiving of this paper, as well as numerous constructive advices on revising this paper, not mentioning his persistent dedication of making publicly available tutorials \cite{Chen:2008iFEM} on writing beautiful vectorized code. \footnote{\url{https://github.com/lyc102/ifem}} We would like to thank Dr. Ari Stern (Washington Univ in St. Louis) for the help on the relocation during the COVID-19 pandemic. We would like to thank Dr. Likai Chen (Washington Univ in St. Louis) for the invitation to the Stats and Data Sci seminar at WashU that resulted the reboot of this study. \footnote{\href{https://math.wustl.edu/events/statistics-and-data-science-seminar-transformer-dissection-amateur-applied-mathematician}{Transformer: A Dissection from an Amateur Applied Mathematician}} We would like to thank Dr. Ruchi Guo (Univ of California Irvine) and Dr. Yuanzhe Xi (Emory Univ) for the invaluable feedbacks on the choice of the numerical experiments. We would like to thank the Kaggle community, including but not limited to Jean-François Puget (CPMP@Kaggle) for sharing a simple Graph Transformer in TensorFlow,\footnote{\url{https://www.kaggle.com/cpmpml/graph-transfomer}}  Murakami Akira (mrkmakr@Kaggle) for sharing a Graph Transformer with a CNN feature extractor in Tensorflow, \footnote{\url{https://www.kaggle.com/mrkmakr/covid-ae-pretrain-gnn-attn-cnn}} and Cher Keng Heng (hengck23@Kaggle) for sharing a Graph Transformer in PyTorch.\footnote{\url{https://www.kaggle.com/c/stanford-covid-vaccine/discussion/183518}} We would like to thank 
daslab@Stanford, OpenVaccine, and Eterna for hosting the COVID-19 mRNA Vaccine competition and Deng Lab (Univ of Georgia) for collaborating in this competition. 
We would like to thank CHAMPS (Chemistry and Mathematics in Phase Space) for hosting the $J$-coupling quantum chemistry competition and Corey Levinson (Eligo Energy, LLC) for collaborating in this competition. 
We would like to thank Zongyi Li (Caltech) for sharing some early dev code in the updated PyTorch \texttt{fft} interface and the comments on the viscosity of the Burgers' equation. 
We would like to thank Ziteng Pang (Univ of Michigan) and Tianyang Lin (Fudan Univ) to update us with various references on Transformers. We would like to thank Joel Schlosser (Facebook) to incorporate our change to the PyTorch \texttt{transformer} module to simplify our testing pipeline. 
We would be grateful to the PyTorch community for selflessly code sharing, including Phil Wang(lucidrains@github) and Harvard NLP group \cite{Klein.Kim.ea:2017OpenNMT}. We would like to thank the \texttt{chebfun} \cite{Driscoll.Hale.ea:2014Chebfun} for integrating powerful tools into a simple interface to solve PDEs. We would like to thank Dr. Yannic Kilcher (ykilcher@twitter) and Dr. Hung-yi Lee (National Taiwan Univ) for frequently covering the newest research on Transformers in video formats.
We would also like to thank the Python community \cite{VanRossum.Drake:2009Python,Oliphant:2007Python} for sharing and developing the tools that enabled this work, including PyTorch \cite{Paszke.Gross.ea2019PyTorch}, NumPy \cite{Harris.Millman.ea:2020Array}, SciPy \cite{Virtanen.Gommers.ea:2020SciPy}, Plotly \cite{Inc.:2015plotly} Seaborn \cite{Waskom:2021seaborn}, Matplotlib \cite{Hunter:2007Matplotlib}, and the Python team for Visual Studio Code. We would like to thank \texttt{draw.io} \cite{drawio} for providing an easy and powerful interface for producing vector format diagrams. For details please refer to the documents of every function that is not built from the ground up in our open-source software library.\footnote{\url{https://github.com/scaomath/galerkin-transformer}}

\end{ack}

\newpage

\bibliographystyle{plainnatnourl}

\newpage

\appendixtitle{Appendices of Choose a Transformer: Fourier or Galerkin}
\appendix


\section{Table of notations}
\label{sec:notations}

\begin{table}[htbp]
\caption{Notations used in an approximate chronological order and their meaning in this work.}
\label{table:notations}
\begin{center}
\resizebox{\textwidth}{!}{%
\begin{tabular}{cl}
\toprule
\textbf{Notation} &
\multicolumn{1}{c}{\textbf{Meaning}}
\\
\midrule
$\mathcal{H}, \mathcal{X}, \mathcal{Y}$ & Hilbert spaces defined on a domain $\Omega$,
$f\in \mathcal{H}: \Omega \to \R$
\\
$\mathcal{Q}, \mathcal{K}, \mathcal{V}$ & Latent representation Hilbert spaces, e.g., $v\in \mathcal{V}: \Omega^*\to \R$
\\
$(\mathcal{H}, \langle\cdot,\cdot\rangle)$ & $\mathcal{H}$ and its inner-product structure, $\langle u, v\rangle := \int_{\Omega}u(x)v(x) dx $ for simplicity
\\
$\mathcal{H}'$ & the space of the bounded linear functionals defined on a Hilbert space
\\
$\|v\|_{\mathcal{H}}$ & The norm defined by the inner product $\|v\|_{\mathcal{H}}:= \langle v, v\rangle^{1/2}$
\\
$\|f_u\|_{\mathcal{H}'}$ & The natural induced norm of $f_u(\cdot):= \langle u, \cdot\rangle$, 
$\|f_u\|_{\mathcal{H}'}:=\sup_{v\in \mathcal{H}} |f_u(v)|/\|v\|_{\mathcal{H}}$
\\
$\|\bm{v}\|$ &    The $\ell^2$-norm defined by the inner product 
$\|\bm{v}\|:=  (\bm{v}\cdot\bm{v})^{1/2}$ for $\bm{v}\in\R^d$
\\
$\mathfrak{b}(\cdot, \cdot)$ & A bilinear form, having two inputs from potentially different subspaces
\\
$L^p(\Omega)$ & The space of functions with integrable $p$-th moments in $\Omega$
\\
$L^{\infty}(\Omega)$ & The space of functions with a bounded essential supremum in $\Omega$
\\
$H^1(\Omega)$ & Sobolev space 
$W^{1,2}(\Omega):=\{\phi \in L^2(\Omega): D\phi \in L^2(\Omega)\}$
\\
$H^1_0(\Omega)$ & $\{ v\in H^1(\Omega): \; \Upsilon(u) =0\text{ on }\partial \Omega\}$, where $\Upsilon(\cdot)$ is the trace operator
\\
$\bm{H}(\mathrm{div};\Omega)$ & Hilbert space with a graph norm 
$\{\phi \in \bm{L}^2(\Omega): \op{div} \phi \in L^2(\Omega)\}$
\\
$C^0_p(\Omega)\simeq C^0(\bb{S}^1)$ & The space of continuous functions with a periodic boundary condition
\\
$\|u\|_{L^2(\Omega)}$ & The $L^2$-norm of $u$, $\|u\|_{L^2(\Omega)}^2 := \int_{\Omega} |u|^2 \dd x$
\\
$|u|_{H^1(\Omega)}$ & The $H^1$-seminorm of $u$, $|u|_{H^1(\Omega)}^2 :=\int_{\Omega} |D u|^2 \dd x$
\\
$x\in \Omega\subset \R^m$  & A point in the spacial domain $\Omega$ of interest
\\
$m$                       & The dimension of the underlying spacial domain in $\R^m$
\\
$d$                       & The dimension of the latent representation approximation subspace
\\
$L_a(u) = f$ &    The operator form (strong form) of a PDE with coefficient $a$
\\
$u(\cdot)$ &  The solution to the weak form $\langle Lu, v\rangle = \langle f, v\rangle$ for any test function $v\in \mathcal{H}$
\\
$a(\cdot)$ &  The coefficients in a PDE operator
\\
$\partial_t u + N(u)=0$ & A time-dependent stiff PDE, where $N(\cdot)$ is nonlinear differential operator
\\
$h$ &         The mesh size of a uniform grid
\\
$n\approx 1/h^m$ &    The discretization size (sequence length) of data, $O(n^m)$ for an $\R^m$ problem
\\
$n_f, n_c$       &  The fine grid size, the coarse grid size
\\
$\bb{X}_h, \bb{Y}_h, \bb{A}_h$ & The discrete function space with degrees of freedom on grid points of mesh size $h$
\\
$\bb{Q}_h, \bb{V}_h$ & Certain subspaces spanned by functions in $\bb{X}_h, \bb{Y}_h$ 
\\
$u_h, a_h$ & The approximation to $u, a$ whose degrees of freedom defined at the grid points
\\
$T$  & The operator to be learned related to a partial differential equation
\\
$T_h$ & The approximation to $T$ applied on functions on a discrete grid with mesh size $h$
\\
$\mathbf{y}, \mathbf{z}$ &  the input of and the output from the attention operator, in $\R^{n\times d}$
\\
$ \bm{q}_i, \bm{k}_i, \bm{v}_i$     &  the $i$-th row of, or the $i$-th position's feature vector in a latent representation
\\
$\bm{z}^i, \bm{q}^i, \bm{k}^i, \bm{v}^i$  
& the $i$-th column of, or the $i$-th basis's discrete DoFs in a latent representation
\\
$A_{i\hspace*{0.07em}\bdot}\;/\; A_{\bdot\hspace*{0.03em} j}$   &  the $i$-th row/$j$-th column of a matrix $A$
\\
$\{y_j(\cdot)\}_{j=1}^d$ & A set of latent basis whose DoFs form the column space of $Y\in \R^{n\times d}$
\\
$\{\chi_{y_j}(\cdot)\}_{j=1}^d$ & The set of degrees of freedom associated with the set of bases $\{y_j(\cdot)\}_{j=1}^d$ 
\\
$(\bm{v})_i$ &   the $i$-th entry/row of a vector $\bm{v}$
\\
$I_h$  & The nodal interpolation operator such that $(I_h v) (x_i) = v(x_i)$
\\
$\Pi_h$  & The interpolation or projection operator that maps function to a grid with mesh size $h$
\\
$\mathcal{H} \hookrightarrow C^0(\Omega)$ & $\mathcal{H}$ is continuously embedded in the space of continuous functions
\\
\bottomrule
\end{tabular}
}
\end{center}
\end{table}

\section{Network structures}
\label{sec:appendix-network}

The network in Figure \ref{fig:simple-ft} is used in Example \ref{sec:burgers}. The model used in the forward Darcy problem \ref{sec:darcy} is in Figure \ref{fig:cnn-ft}. A detailed comparison can be found in Table \ref{table:networks}.

When the target is smooth, the spectral convolution layer from \cite{Li.Kovachki.ea:2021Fourier} is used in our network as a smoother (decoder), and the original ReLU activation is replaced by the Sigmoid Linear Unit (SiLU) \cite{Ramachandran.Zoph.ea:2017Searching}. We have removed the batch normalization (BN) from the original spectral convolution layers as well. Whenever the input or the output of certain layers of the network is approximating a non-smooth function a priori, the activation are changed from SiLU to ReLU.

\begin{table}
\caption{The detailed comparison of networks; SC: spectral convolution layer; a \texttt{torch.cfloat} type parameter entry counts as two parameters.}
\label{table:networks}
\begin{center}
  \resizebox{0.8\textwidth}{!}{%
  \begin{tabular}{rccccccccc}\toprule
& \multicolumn{3}{c}{Encoder} 
& \multicolumn{4}{c}{Decoder}
&  \multirow{3}{*}{\# params} 
\\
\cmidrule(lr){2-4}
\cmidrule(lr){5-8}
&  layers & \texttt{dmodel} & \texttt{nhead}  
& \# SC & \texttt{dmodel} & modes & activation &  
\\
\midrule
FNO 1D          & 0 & N/A & N/A  & 4 & 64 & 16 & ReLU & 550k
\\
FT/GT in \ref{sec:burgers}   & 4 & 96 & 1    & 2 & 48 & 16 & SiLU & 523k--530k
\\
FNO 2D          & 0 & N/A & N/A  & 4 & $32\times 32$ & 12 & ReLU & 2.37m
\\
FT/GT in \ref{sec:darcy}   & 6 & 128 & 4    & 2 & $32\times 32$ & 12 & SiLU & 2.22m
\\
FT/GT in \ref{sec:darcy-inverse}  & 6 & 192 & 4    & 0 & N/A & N/A & SiLU & 2.38m
\\
\bottomrule
\end{tabular}
}
\end{center}
\end{table}

\begin{figure}[htb]
\centering
\includegraphics[width=0.7\textwidth]{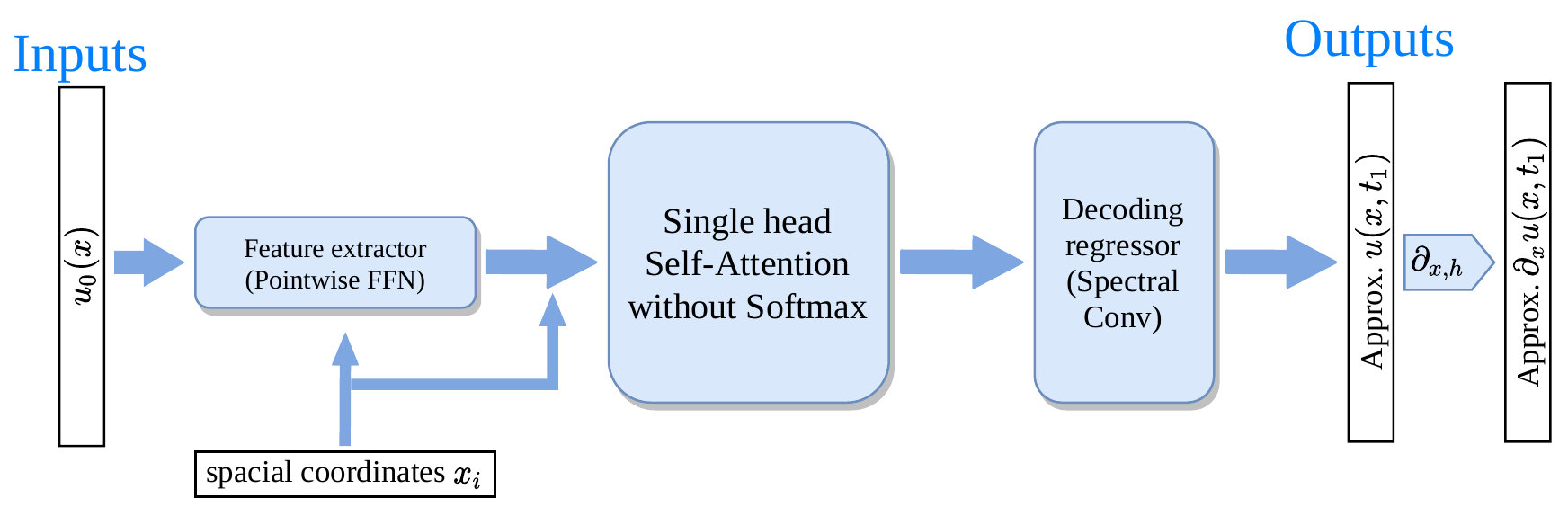}
\caption{A simple attention-based operator learner in Example \ref{sec:burgers}.}
\label{fig:simple-ft}
\end{figure}

\begin{figure}[htb]
\centering
\includegraphics[width=0.95\textwidth]{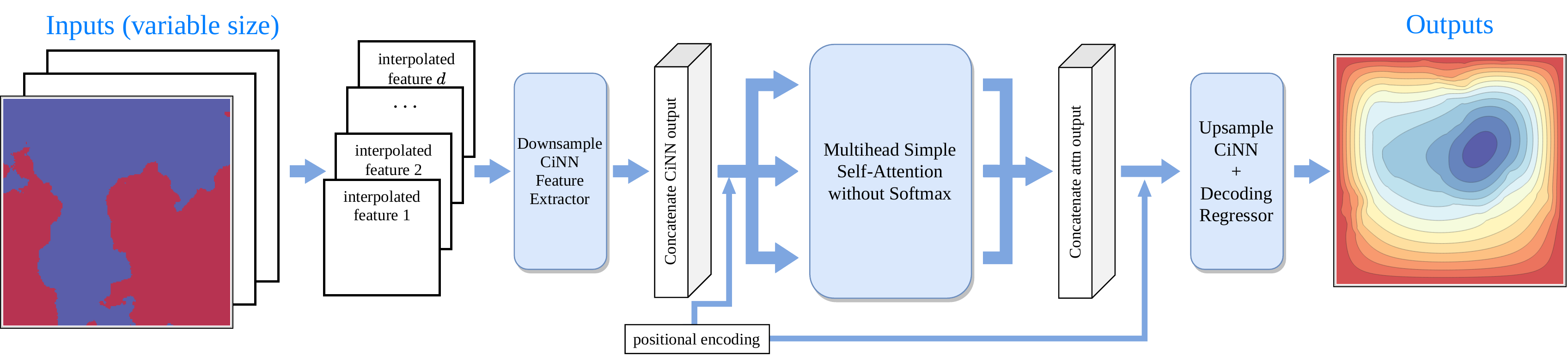}
\caption{An attention-based operator learner on $\Omega\subset \R^2$.}
\label{fig:cnn-ft}
\end{figure}

\paragraph{Downsampling CNN.} The downsampling interpolation-based CNN (CiNN) is to reduce the size of the input to a computationally feasible extent for attention-based encoder layers, in addition to a channel expansion to match the hidden feature dimension (channels).
The full resolution input function sampled at a fine grid of size $n_f\times n_f$ is downsampled by CiNN to a collection of latent functions on an $n_c\times n_c$ coarse grid. Then, the coarse grid representations are concatenated with the positional encoding (Euclidean coordinates on the coarse grid) to be sent to the attention-based encoder layers.

The structures of the downsampling CiNN can be found in Figure \ref{fig:cnn-down}. 
In CiNN, instead of pooling, the downsampling are performed through a bilinear interpolation with nonmatching coarse/fine grids. 
The convolution block adopts a simplified variant from the basic block in \cite{He.Zhang.ea:2016Identity}. The convolution layer is applied only once before the skip-connection, and the batch normalization is removed from the block. 

In the downsampling CiNN, the first convolution layer maps the input data to a tensor of which the number of channels matches the number of hidden dimension of the attention layers. Then, the full resolution representations in all channels are interpolated from the $n_f\times n_f$ grid to an $n_m\times n_m$ grid of an intermediate size between $n_f$ and $n_c$, and $n_m\approx \sqrt{n_f n_c}$.
Next, another three convolution layers are applied consecutively together with their outputs stacked in the channel dimension. Finally, this stacked tensor is downsampled again by another bilinear interpolation as the output the downsampling CiNN.

The coarse grid positional encoding of size $n_c\times n_c\times 2$ is concatenated to the output latent representations before flattening and the scaled dot-product attention. 
As a result, the input/output dimensions of an attention-based encoder layer are both $n_c^2 \times d$. Nevertheless, inside an encoder layer, the propagation runs for a tensor of size $n_c^2 \times (d+m\cdot (\texttt{nhead}))$.

\paragraph{Upsampling CNN.} The output from the attention-based encoder layers is first reshaped from an $n_c^2\times d$ matrix to an $n_c\times n_c\times d$ tensor, then upsampled by another CiNN to the full resolution. The upsampling CiNN has a simpler structure (Figure \ref{fig:cnn-up}) than the downsampling CiNN. We have two interpolations that map tensors of $n_c\times n_c \times d$ to $n_m\times n_m \times d$, and $n_m\times n_m \times d$ to $n_f\times n_f \times d$, respectively. A simple convolution layer with matching number of channels is between them. The positional encoding of the fine grid is then concatenated with the output from the upsample layer, and sent to the decoder layers.

\begin{figure}[htb]
\centering
\includegraphics[width=0.9\textwidth]{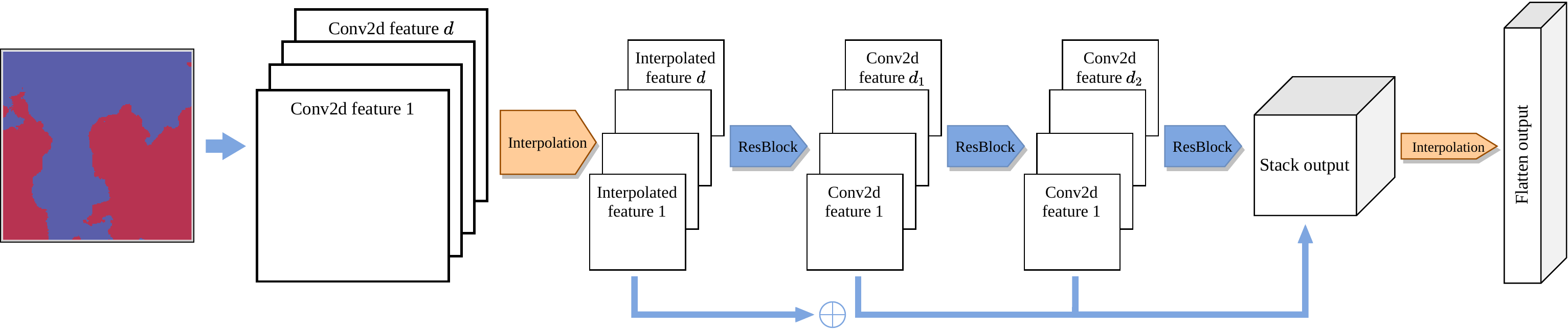}
\caption{A 2D bilinear interpolation-based CNN (CiNN) for downsampling.}
\label{fig:cnn-down}
\end{figure}

\begin{figure}[htb]
\centering
\includegraphics[width=0.7\textwidth]{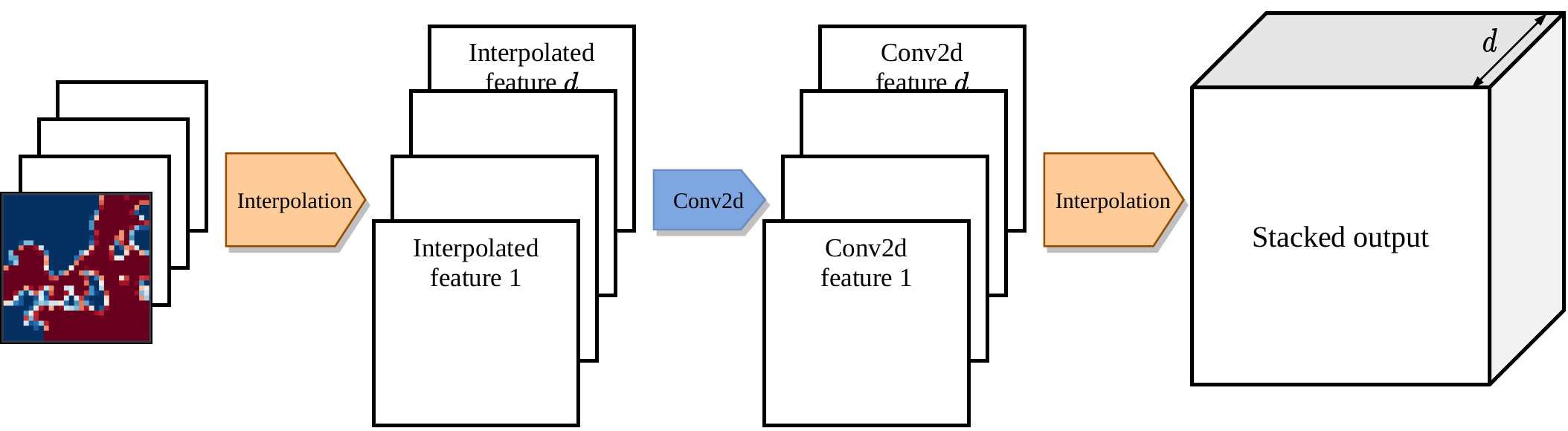}
\caption{A 2D bilinear interpolation-based CNN (CiNN) for upsampling.}
\label{fig:cnn-up}
\end{figure}

\section{Supplemental details of the experiments}
\label{sec:appendix-experiments}

\subsection{Training and evaluation setup}
In training our models, we opt for a standard \texttt{1cycle} \cite{Smith.Topin:2019Super} learning rate strategy with a warm-up phase for an environmental responsible and a seed-invariant training. We run a mini-batch ADAM iterations for a total number of 100 epochs (12800 iterations with batch size 8).  
The learning rate starts and ends with $10^{-4}\cdot lr_{\max}$, and reaches the maximum of $lr_{\max}$ at the end of the 30-th epoch. The $lr_{\max}=10^{-3}$ for all models except being $5\times 10^{-4}$ for ST and FT in 2D problems.

The batch size is set to 8 in 1D in $n=512, 2048$, and 4 in $n=8192$ as well as in 2D examples. The training cost of our models are reported in Table \ref{table:training-times}. 
It is not surprising that attention-based operator learners can be more efficiently trained than traditional MLP-based operator learners. Our model can be trained using merely a fraction of time versus MLP-based operator learners (cf. Burgers' equation training time in \cite[Appendix C]{Wang.Wang.ea:2021Learning}), thus proven to be a more environment-friendly data-driven model.

For the three examples prepared, there are 1024 samples in the training set, and 100 in the testing set. Even though the initial conditions or the coefficients for different samples, in Example \ref{sec:burgers} 
and Example \ref{sec:darcy} respectively, follow the same distribution constructed based on GRF, 
there are no repetitions between the functions in the training set and those in the testing set.

\begin{table}[htbp]
\caption{Environmental impact measured in computational cost of training (in hours).}
\label{table:training-times}
\begin{center}
  \resizebox{0.78\textwidth}{!}{%
\begin{tabular}{rcccccrccc}
\toprule
& \multicolumn{3}{c}{Example 1 } & \multicolumn{2}{c}{Example 2}
& \multicolumn{2}{c}{Example 3 }
\\
\cmidrule(lr){2-4}
\cmidrule(lr){5-6}
\cmidrule(lr){7-8}
&  $n=512$  & $n=2048$ & $n=8192$ &  $n_f=141$  & $n_f=211$ 
& $n_f=141$  & $n_f=211$ 
\\
\midrule
FT  & $0.063$ & $0.138$  & $1.217$  & $0.615$  &  $1.553$ &  $0.452$   & $2.638$
\\
GT & $0.064$  & $0.079$  & $0.245$  &  $0.367$ &  $0.610$ & $0.248$ &  $0.857$
\\
\bottomrule
\end{tabular}
}
\end{center}
\end{table}

During training, there is no regularization applied for the weights of the models. A simple gradient clip of 1 is applied. 
When the target function is known a priori being smooth with an $H^{1+\alpha}$ regularity ($\alpha>0$), 
we employ an $H^1$-seminorm regularization between the 2nd order approximation 
(the central difference in 1D, and the 5-point stencil in 2D) to derivatives of the targets and those of the outputs from the model (see Section \ref{sec:operator-learning}). We choose $\gamma=0.1h$ in Example \ref{sec:burgers} and $\gamma = 0.5h$ in Example \ref{sec:darcy}.
The dropouts during training are obtained through a simple grid search in $\{0.0, 0.05, 0.1\}$ and can be found in Table \ref{table:dropouts}.

\begin{table}
\caption{The dropout comparison during training.}
\label{table:dropouts}
\begin{center}
  \resizebox{0.65\textwidth}{!}{%
  \begin{tabular}{rccccccc}\toprule
&  attention & FFN & downsample 
& upsample & decoder   
\\
\midrule
FT both $\text{Ln}$s in \ref{sec:burgers}   & $0.0$ & $0.05$ &  N/A   & N/A & $0.0$ 
\\
GT new $\text{Ln}$ in \ref{sec:burgers}   & $0.0$ & $0.0$ & N/A    & N/A & $0.0$ 
\\
GT reg $\text{Ln}$ in \ref{sec:burgers}   & $0.1$ & $0.1$ & N/A    & N/A & $0.0$ 
\\
FT in \ref{sec:darcy}   & $0.1$ & $0.1$ &  $0.05$    &  $0.0$ & $0.0$
\\
GT in \ref{sec:darcy}   & $0.1$ & $0.05$ &  $0.05$  & $0.0$ & $0.0$
\\
FT/GT in \ref{sec:darcy-inverse}  & $0.05$ & $0.05$ &  $0.05$  & N/A & $0.05$
\\
\bottomrule
\end{tabular}
}
\end{center}
\end{table}

Throughout all the numerical experiments, the result is obtained from setting $1127802$ as the random number generator seed, and the PyTorch \texttt{cuDNN} backend to deterministic.
Error bands are reported using 10 different seeds (see Figure \ref{fig:conv}). 
All benchmarks are run on a single RTX 3090. Due to the nature of our problem and the data pairs, there is no randomness between the input and the output for a single instance in Example \ref{sec:burgers} and Example \ref{sec:darcy}, the only stochastic components are the sampling for optimizations and the initializations of the model, 
and the impact due to the choice of seeds for our models is empirically minimal.

\begin{figure}[htp]
\begin{center}
\begin{subfigure}[b]{0.45\linewidth}
  \centering\includegraphics[width=0.97\linewidth]{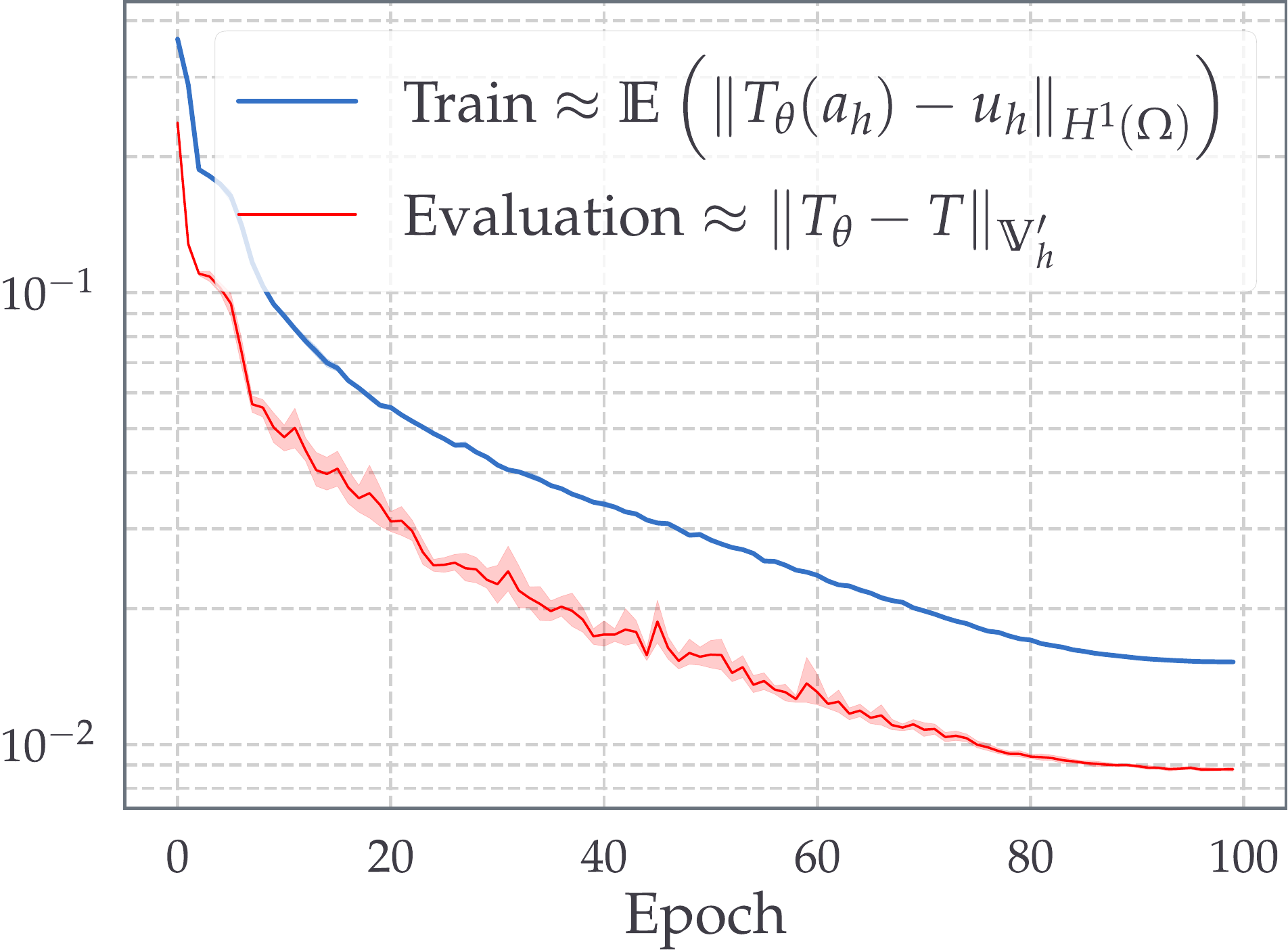}
  \caption{\label{fig:ex2-conv}}
\end{subfigure}%
\hspace*{0.2in}
\begin{subfigure}[b]{0.45\linewidth}
      \centering
      \includegraphics[width=0.97\linewidth]{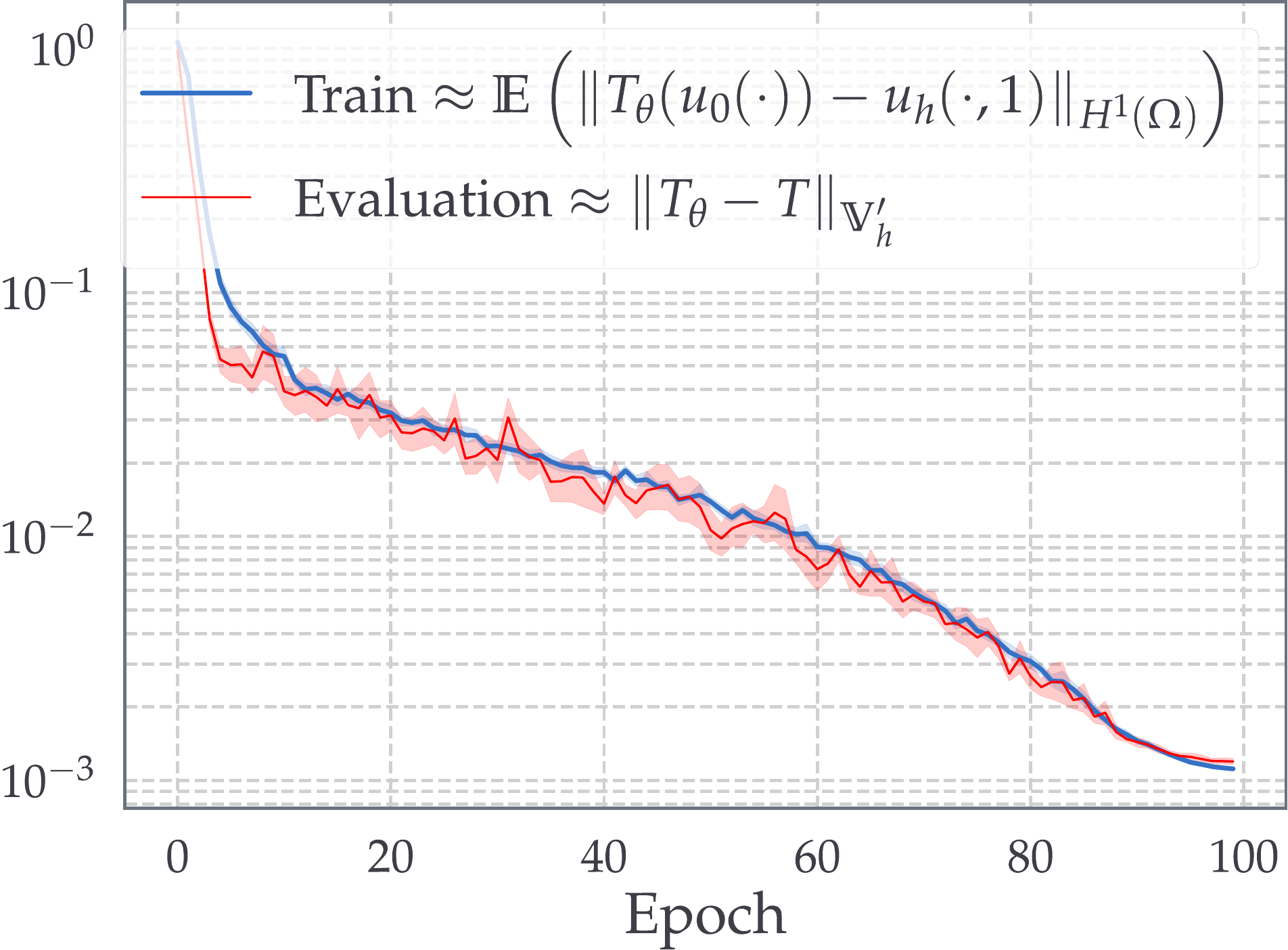}
      \caption{\label{fig:ex1-conv}}
\end{subfigure}
\end{center}
\caption{Typical training and evaluation convergences of the Galerkin Transformer. (\subref{fig:ex2-conv}) example \ref{sec:darcy}: 10 different seeds for the model initialization; (\subref{fig:ex1-conv}) example \ref{sec:burgers}: 10 different seeds in both model initializations and the train loader; the discrete norm of a nonlinear operator $\|N\|_{\bb{V}_h'} := \sup_{v\in \bb{V}_h} \|N(v)\|_{\mathcal{H}}/\|v\|_{\mathcal{H}}$ is defined similarly to that of a linear operator. The $H^1$-seminorm part in the $H^1$-norm shown in figures is weighted by $\gamma$ from Section \ref{sec:operator-learning}.} 
\label{fig:conv}
\end{figure}


\subsection{Experimental details for Example \ref{sec:burgers}}
\label{sec:appendix-burgers}
\paragraph{Data preparation.} The following problem is considered in Example \ref{sec:burgers}: 
\begin{equation}
\label{eq:burgers}
\begin{cases}
\displaystyle {\partial_t u} + u \partial_x u
=\nu {\partial_{xx} u} \quad \text{ for } (x,t)\in (0,1) \times (0,1],
\\[5pt]
u(x, 0) = u_0(x) \text{ for } x\in (0,1),
\end{cases}
\end{equation}
and it is assumed that $u_0 \in C^0_{p} (\Omega) \cap L^2(\Omega)$. The operator to be learned is:
\[
T:  C^0_{p} (\Omega) \cap L^2(\Omega) \to C^0_{p} (\Omega) \cap H^1 (\Omega), \quad 
u_0(\cdot) \mapsto u(\cdot, 1).
\]
Following \cite{Li.Kovachki.ea:2021Fourier}, the initial data are prepared using a Gaussian Random Field (GRF) simulation $\sim \mathcal{N}(0, 25^2 (-\Delta + 25I)^{-2})$, and the system is solved using the \texttt{chebfun} package \cite{Driscoll.Hale.ea:2014Chebfun} with a spectral method using a very fine time step $\delta t \ll 1$ for the viscosity $\nu = (2\pi)^{-1}0.1$ on a grid with $n=8192$. Solutions and initial conditions in other two resolutions ($n=512,\; 2048$) are downsampled from this finest grid. Therefore, the discrete approximation the model learns is: for $h\in \{2^{-9}, 2^{-11}, 2^{-13}\}$
\[
T_h: \bb{X}_h \to \bb{X}_h,\quad I_h u_0(\cdot) \mapsto \Pi_h u_{h_{*}} (\cdot, 1),
\]
where $\bb{X}_h\simeq \R^n$ denotes a function space of which the degrees of freedom are defined on the grid with mesh size $h$, such as the space of piecewise linear functions. $I_h (\cdot)$ denotes the nodal value interpolation, and $\Pi_h(\cdot)$ denotes the downsampling operator from the space on the finest grid with mesh size $h_* = 2^{-13}$ to $\bb{X}_h$.

\begin{figure}[htb]
\centering
\includegraphics[width=0.8\textwidth]{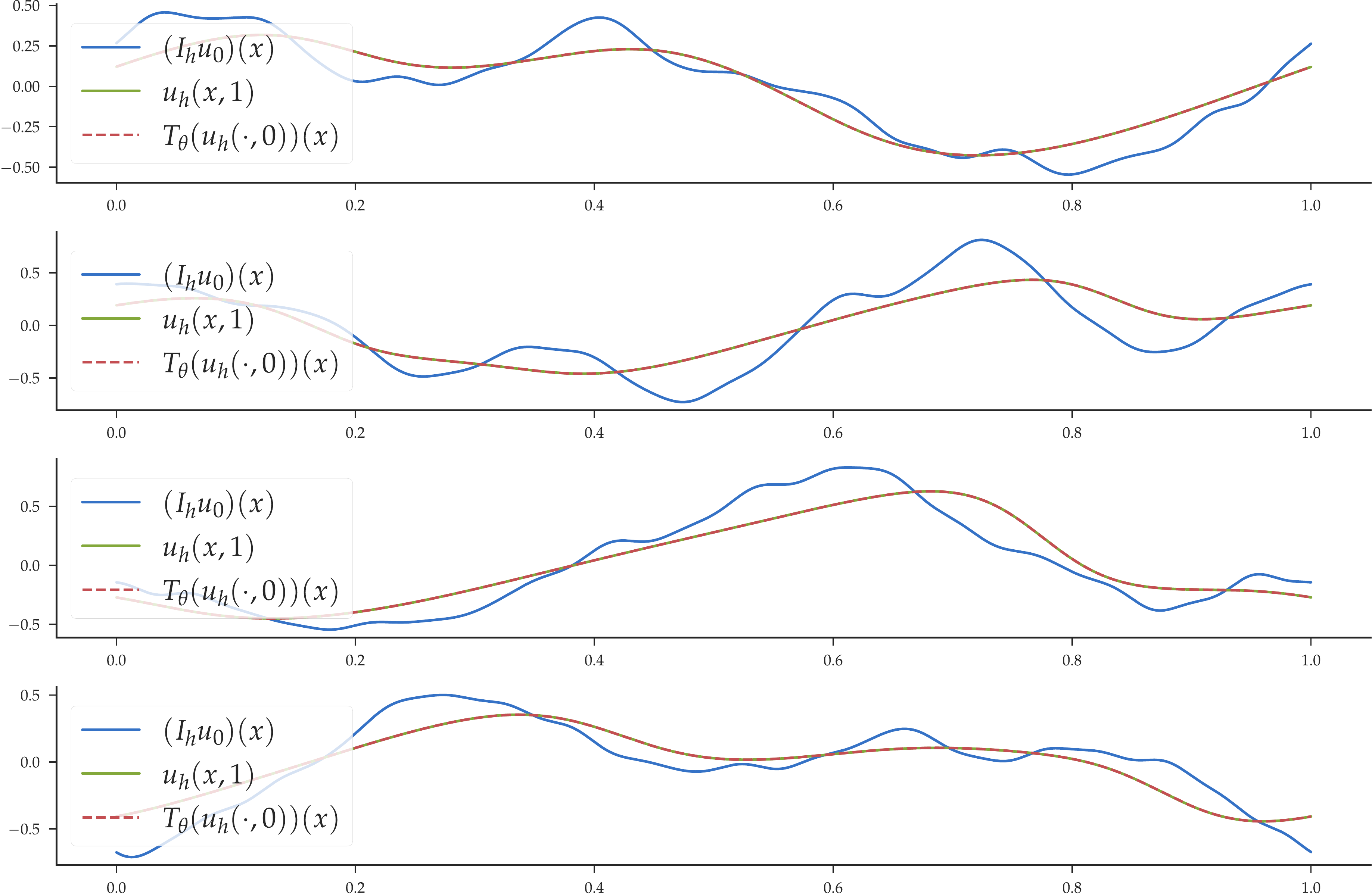}
\caption{Evaluation results for 4 randomly chosen samples in the test set; the average relative error $=1.079\times 10^{-3}$.}
\label{fig:burgers-init}
\end{figure}

\paragraph{Effect of the diagonal initialization.}
When using the regular layer normalization rule
\eqref{eq:attention-ln} to train the GT models (which is equivalent to the efficient attention proposed in \cite{Shen.Zhang.ea:2021Efficient}), 
it fails to converge under the \texttt{1cycle} learning rate scheduling with the $0.05$ dropout in FFN and $0.0$ dropout for the attention weights. 
We observe that the evaluation metric peaks after the warmup phase ($\approx 2\times 10^{-2}$ around epoch 30). 
The GT using the new rule \eqref{eq:attention-galerkin} is convergent almost unconditionally under the same initialization, which reaches the common $\approx 1\times 10^{-3}$ range 
if the diagonal initialization is employed 
(e.g., see Figure \ref{fig:ex1-conv}). 
The result reported in Table \ref{table:burgers} for GT with the regular layer normalization is obtained through imposing a $0.1$ dropout for the attention weights. 
A more detailed comparison can be found in Table \ref{table:burgers-init}. 
The training becomes divergent for a certain model if the best epoch shown in the table is not around 100. We conjecture that the success of the diagonal initialization is due to the input (initial conditions, blue curves in Figure \ref{fig:burgers-init}) being highly correlated spacially with the target (solutions at $t=1$, green curves in Figure \ref{fig:burgers-init}), despite the highly nonlinear mapping. 
Thus, in the encoder layers, what the attention operator learned is likely to be a perturbation of the identity operator in the latent Hilbert space, if a suitable basis can be found for $\bb{V}_h$ (or $\bb{Q}_h$ in the linear variant).

\paragraph{Effect of the Galerkin projection-type layer normalization scheme.} 
Multiplying $u$ on both sides of \eqref{eq:burgers}, the energy law of the Burgers' equation can be obtained through an integration by parts on $\Omega:=(0,1)$ with the periodic boundary condition:
\begin{equation}
\label{eq:burgers-energy}
\langle \partial_t u, u\rangle +\langle\partial_x(u^2), u\rangle/2 =
\nu\langle \partial_{xx} u, u\rangle \implies 
\dd \left( \|u\|_{L^2(\Omega)}^2\right)/\dd t = -\nu \|\partial_x u\|_{L^2(\Omega)}^2.
\end{equation}
Consequently, once the initial condition is given, integrating \eqref{eq:burgers-energy} from $t=0$ to a
fixed future time yields how much the energy of a single instance of $u$ has decayed, and this is a deterministic quantity. This indicates that the scale-preserving property of the Galerkin projection-type layer normalization would potentially learn this decaying property resulted by the operator, thus outperforms the regular layer normalization scheme that normalizes each position's feature vector. 
We also note that using an instance normalization \cite{Ulyanov.Vedaldi.ea:2017Instance} may appear to be more sensible as it normalizes $\|\bm{v}^j\|$ to $1$ ($1\leq j\leq d$) after the $1/n$ weight, however, we find that opting for the instance normalization deteriorates the training's stability and the dropout needs to be further dialed up. For more heuristics of the Galerkin-type layer normalization scheme please refer to Section \ref{sec:appendix-cea-generalizations}.

\begin{table}
\caption{Ablation study: evaluation relative error ($\times 10^{-3}$) of the Galerkin Transformer with the regular layer normalization \eqref{eq:attention-ln} and the new one \eqref{eq:attention-galerkin}, using various types of initialization; the baseline model is the GT with the default Xavier uniform initialization. $\text{GT}_{R}$: layer normalization \eqref{eq:attention-ln}; $\text{GT}_{A}$: layer normalization \eqref{eq:attention-galerkin}; $\delta$: the weight added to the diagonal; $\eta$: the gain of Xavier uniform initialization; $\zeta$: dropout for the attention weights.}
\label{table:burgers-init}
\begin{center}
\resizebox{0.8\textwidth}{!}{%
  \begin{tabular}{lcccccc}\toprule
& \multicolumn{2}{c}{$n=512$} & \multicolumn{2}{c}{$n=2048$}
& \multicolumn{2}{c}{$n=8192$ ($b=4$)}
\\
\cmidrule(lr){2-3}
\cmidrule(lr){4-5}
\cmidrule(lr){6-7}
&  Rel.~err  & Best ep.  &  Rel.~err & Best ep. &  Rel.~err  & Best ep.
\\
\midrule
$\text{GT}_{R}$ (ET in \cite{Shen.Zhang.ea:2021Efficient})                            & $200.4$ & 86  & $208.4$ & 39 & $217.5$ & 28  
\\
$\text{GT}_{R}$, $\zeta=0.1$                & $206.4$ & 46  & $205.9$ & 59 & $207.0$  & 75  
\\
$\text{GT}_{R}$, $\eta=10^{-2}$             & $1.406$ & 99  & $21.38$ & 14 &  $17.75$ & 16  
\\
$\text{GT}_{R}$, $\eta=10^{-2}$, $\zeta=0.1$ & $16.85$  & 20 & $11.19$  & 21  & $2.571$  & 98  
\\
$\text{GT}_{R}$, $\eta=\delta=10^{-2}$      & $15.14$ & 35  & $1.512$ & 97 &  $15.98$ & 34  
\\
$\text{GT}_{R}$, $\eta=\delta=10^{-2}$, $\zeta=0.1$  & $2.181$   & 100  & $13.96$ & 19 & $2.331$ & 92
\\
$\text{GT}_{A}$                             & $10.06$  & 96  & $10.12$ & 99 & $9.129$  & 100
\\
$\text{GT}_{A}$, $\eta=10^{-2}$             & $1.927$ & 95 & $2.453$ & 100 & $1.689$ & 99 
\\
$\text{GT}_{A}$, $\eta=\delta=10^{-2}$      & \b{1.203} & 99 & \b{1.150} & 100 & \b{1.025} & 100 
\\
\bottomrule
\end{tabular}
}
\end{center}
\end{table}

\subsection{Experimental details for Example \ref{sec:darcy}}
\label{sec:appendix-darcy}

\paragraph{Data preparation.}
Example \ref{sec:darcy} considers another well-known benchmark problem used in 
\cite{Li.Kovachki.ea:2021Fourier,Nelsen.Stuart:2020Random,Li.Kovachki.ea:2020Multipole}. The Darcy flow in porous media $\Omega := (0,1)^2$, in which the diffusion coefficient $a \in L^\infty(\Omega): x \mapsto \R^+$ represents the permeability of the media and has a sharp contrast within the domain.
\begin{equation}
\left\{
\begin{aligned}
  -\nabla \cdot (a\nabla u) = f & \text{ in } \;\Omega,
\\
u = 0 & \text{ on } \partial \Omega.
\end{aligned}
\right.
\end{equation}
For each sample in the training and validation data, $a(x)$ is generated according to $a \sim \nu:=\psi_{\sharp} \mathcal{N}(0,(-\Delta + 9I)^{-2})$, where within the covariance $-\Delta$ is defined on $(0,1)^2$ and has homogeneous Neumann boundary conditions. The mapping $\psi : \R \to \R$ is constructed by
\[
\psi(\rho)=12 \mathds{1}_{(0, \infty)}(\rho)+ 3 \mathds{1}_{(-\infty, 0)}(\rho).
\]
Thus, the resulting coefficient $a$ follows a pushforward probability measure $\nu$, and takes values of $12$ and $3$ almost surely within $\Omega$. 
The geometry of the interface exhibits a random pattern in $\Omega$ (see Figure \ref{fig:ex2-interface}). 
The forcing $f$ is fixed as $f \equiv 1$.

\begin{figure}[htbp]
\begin{center}
\begin{subfigure}[b]{0.4\linewidth}
      \centering
      \includegraphics[height=1.2in]{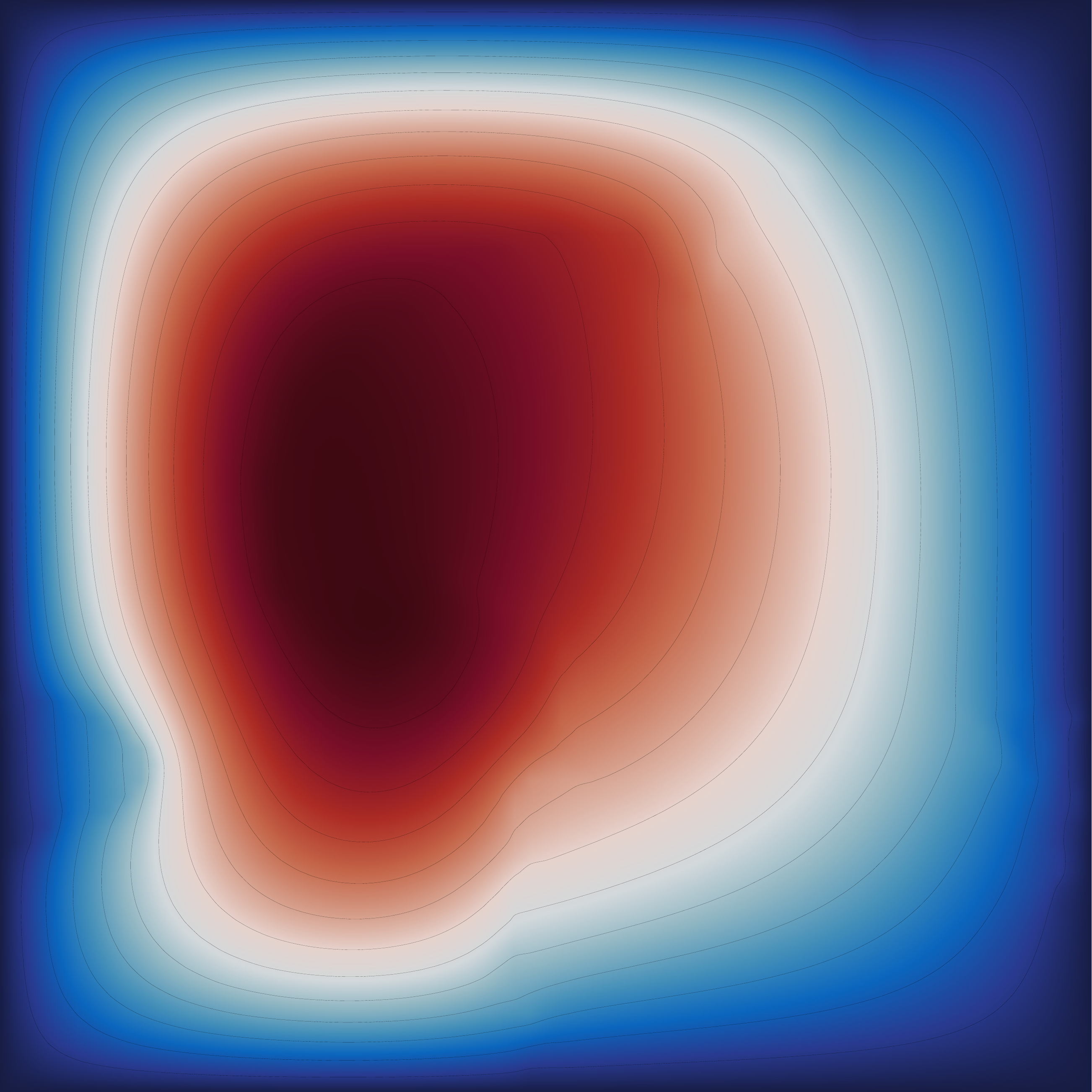}
      \caption{\label{fig:ex2-solution}}
\end{subfigure}
\;
\begin{subfigure}[b]{0.4\linewidth}
      \centering
      \includegraphics[height=1.2in]{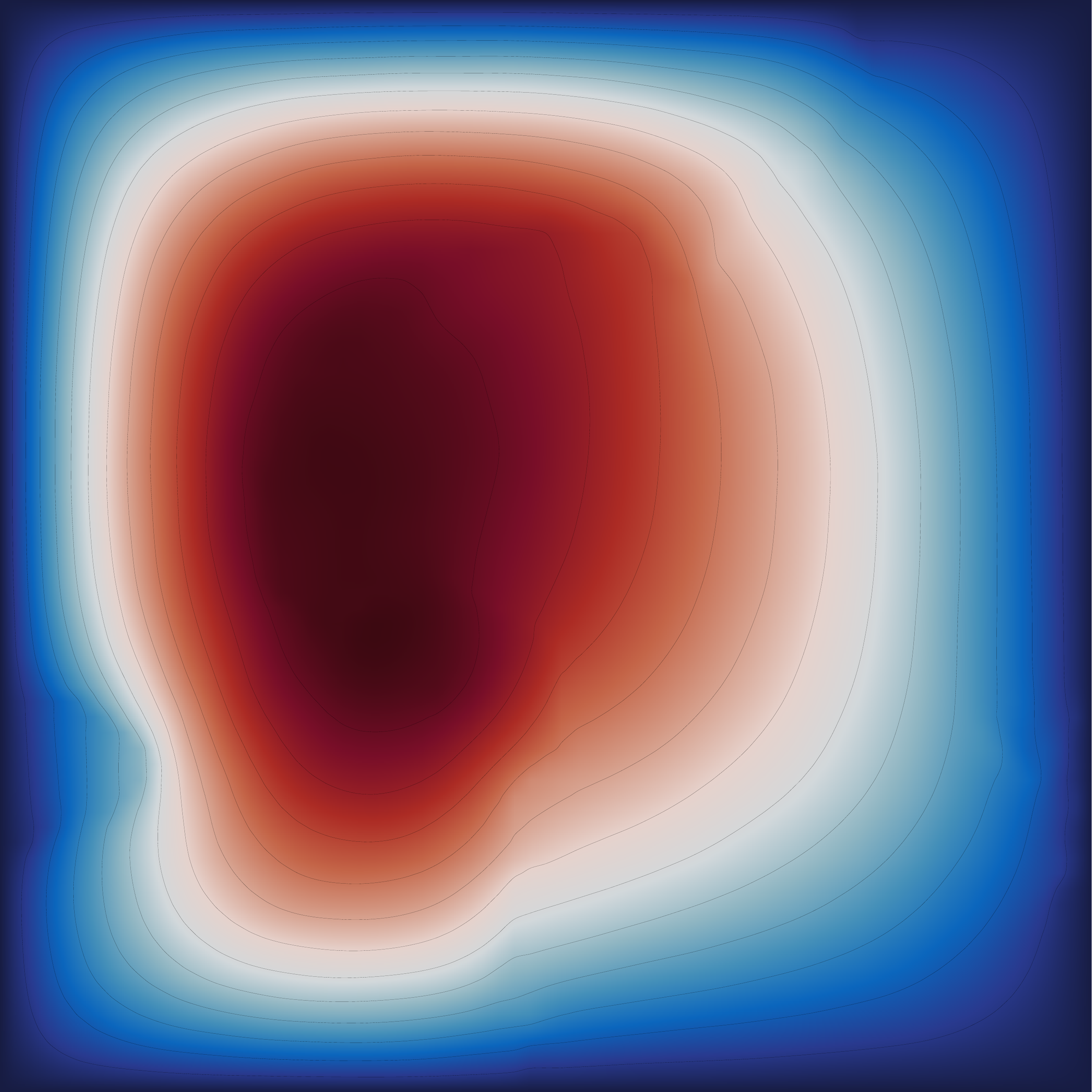}
      \caption{\label{fig:ex2-preds}}
\end{subfigure}
\\
\begin{subfigure}[b]{0.4\linewidth}
  \centering
  \hspace*{0.2in}\includegraphics[height=1.15in]{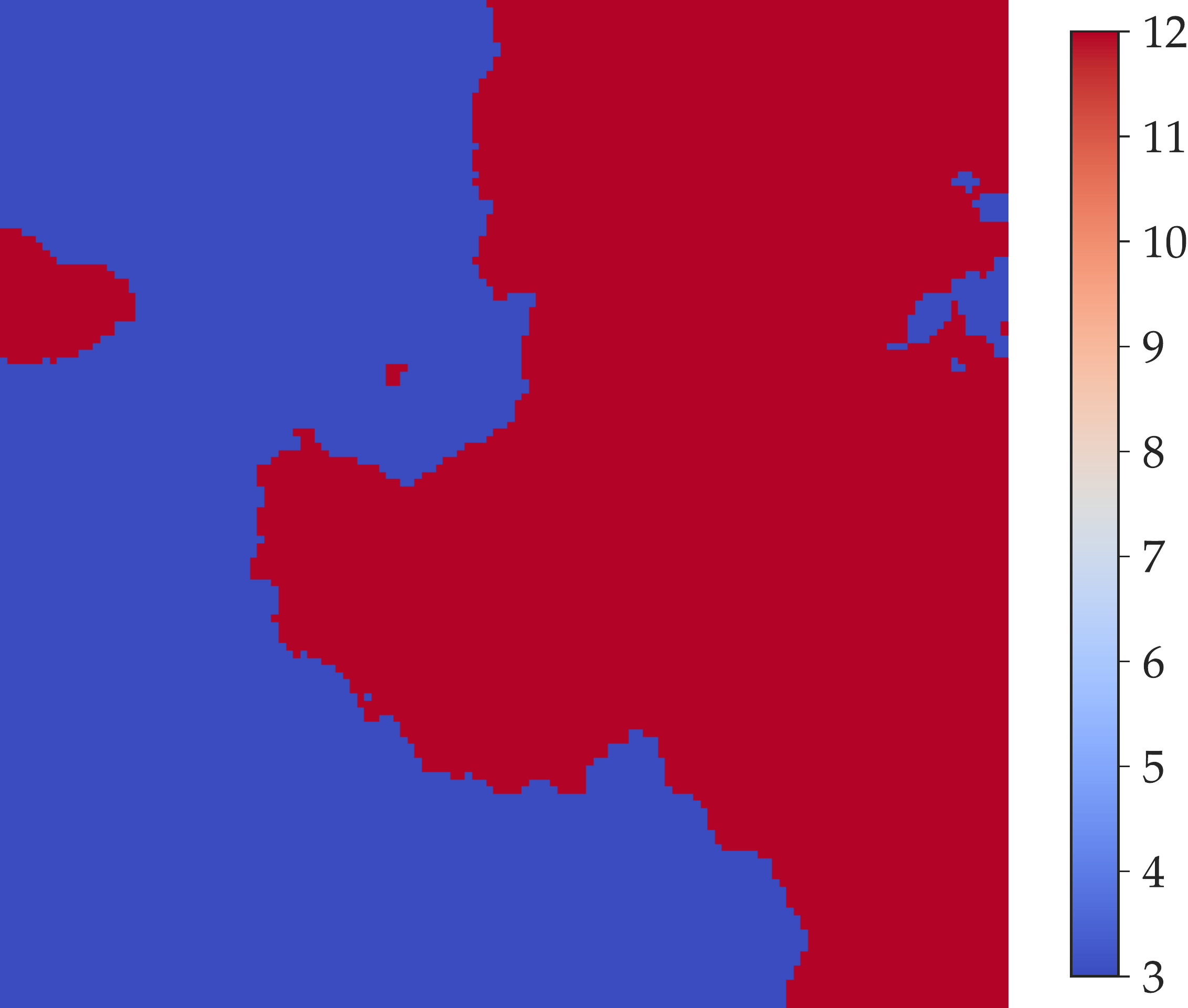}
  \caption{\label{fig:ex2-interface}}
\end{subfigure}%
\quad 
\begin{subfigure}[b]{0.4\linewidth}
  \centering
  \hspace*{0.3in}\includegraphics[height=1.15in]{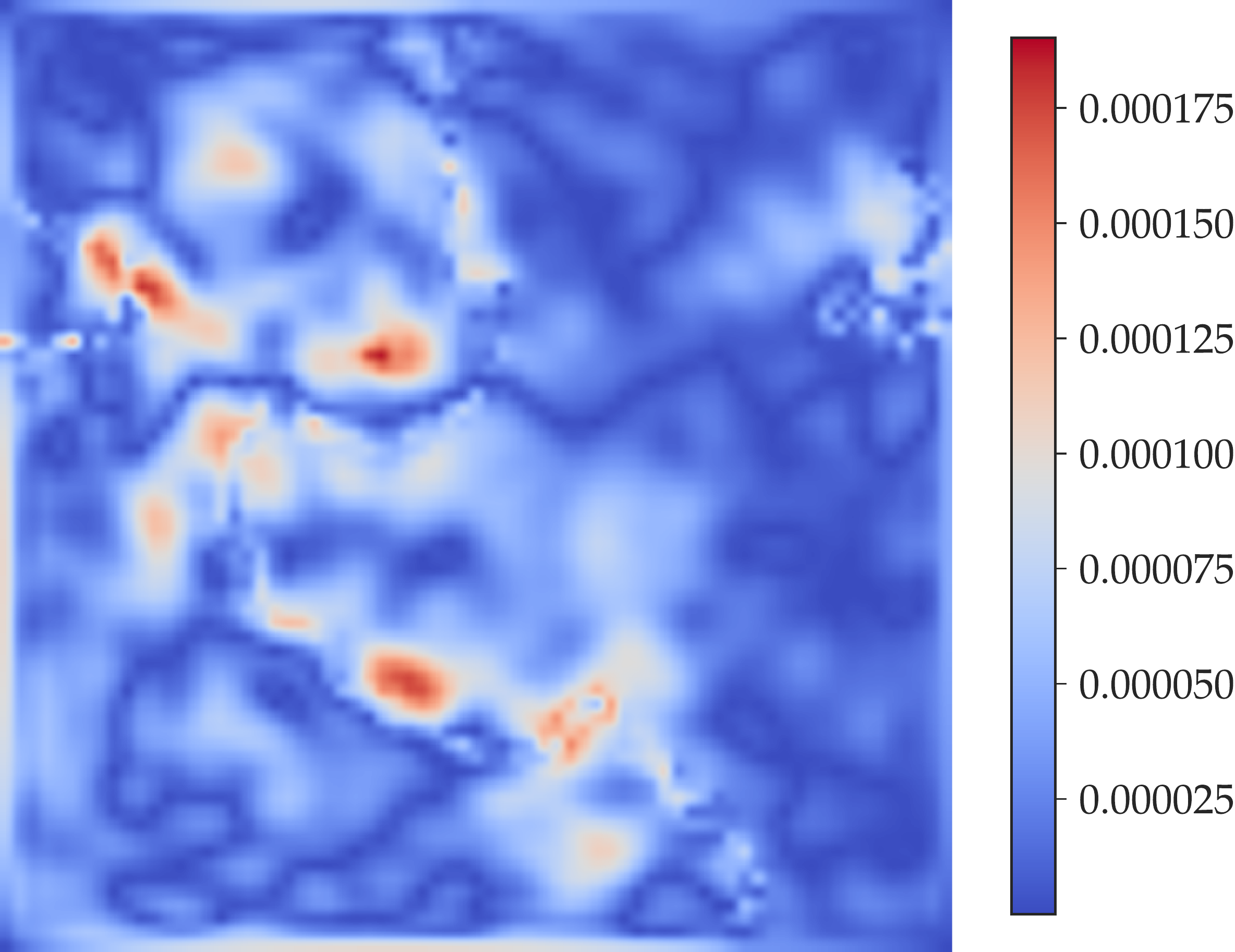}
  \caption{\label{fig:ex2-error}}
\end{subfigure}%
\end{center}
\caption{Interface Darcy flow in example \ref{sec:darcy}: a randomly chosen sample from the test dataset. (\subref{fig:ex2-solution}) the target being the finite difference approximation to the solution on a very fine grid; (\subref{fig:ex2-preds}) the inference approximation by model evaluation (relative $L^2$-error $6.454\times 10^{-3}$); (\subref{fig:ex2-interface}) the input $a(x)$; (\subref{fig:ex2-error}) the $L^{\infty}$-error distribution for the inference solution.}
\label{fig:ex2-darcy}
\end{figure}

The operator to be learned is between the diffusion coefficient and the unique weak solution:
\[
T: L^\infty(\Omega) \to H^1_0 (\Omega), \quad a\mapsto u.
\]
The finite dimensional approximation $u_h$'s are obtained using a 5-point stencil second-order finite difference scheme on a $421 \times 421$ grid. Therefore, the discrete operator $T_h$ to be learned is:
\begin{equation}
\label{eq:darcy-operator}
  T_h: \bb{A}_{h} \mapsto \bb{V}_h, \quad a_h \mapsto \Pi_h u_{h^*},
\end{equation}
where $\bb{A}_h$ and $\bb{V}_h$ are function spaces of which the degrees of freedom are defined on the fine grid points with mesh size $h$, $h^*=1/421$ is the finest grid size, and $a_h:=I_h a$ and $\Pi_h(\cdot)$ are defined accordingly in a similar fashion with Example \ref{sec:burgers} explained in Appendix \ref{sec:appendix-burgers}. Following the practice of \cite{Li.Kovachki.ea:2021Fourier}, a non-trainable Gaussian normalizer is built in the network to transform the input and the target to be $\sim \mathcal{N}(0,1)$ pointwisely on each grid point. 

By choosing $\bb{V}_h$ as the standard bilinear Lagrange finite element on a uniform Cartesian grid on $\Omega$, a standard summation by parts argument for $-\Delta_h$ and the discrete Poincar\'{e} inequality guarantees the well-posedness of problem using the Lax-Milgram lemma, i.e., given $a_h \in \bb{A}_h \simeq \R^{n_f\times n_f}$, the linear system of the $-a_h\Delta_h (\cdot)$ discretization has a unique solution $u_h$. Even though the inversion of the stiffness matrix in resulting linear system is a linear problem, the mapping in \eqref{eq:darcy-operator} is highly nonlinear between two spaces isomorphic to $\R^{n_f\times n_f}$.

\paragraph{Limitations.} We acknowledge that our method, despite surpassing the current best operator learner's benchmark evaluation accuracy, 
still does reach the accuracy of traditional discretization-based methods that aim to best the approximation for a single instance. For example, in Figure \ref{fig:ex2-error}, the error is more prominent at the location where the coefficient has sharp contrast. 
How to incorporate the adaptive methods (allocating more degrees of freedom based on the a posteriori local error) to data-driven operator learners will be a future study.

\subsection{Experimental details for Example \ref{sec:darcy-inverse}}
\label{sec:appendix-darcy-inv}
\paragraph{Inverse problems.}  
Playing a central role in many practical applications, the inverse problems are a class of important tasks in many scientific disciplines \cite{Kirsch:2011Introduction}. The problem summarizes to using the measurements to infer the material/physical coefficients.
In almost all cases, the inverse problem is much harder than solving the forward problem (e.g., solving for $u$ in problem $L_a(u) = f$), as the mapping from the solution (measurements) back to the coefficient is much less stable than the forward operator due to a much bigger Lipschitz constant. As a result, the inverse operator amplifies noises in measurements by a significant amount. 
For example, by \cite[Theorem 5.1]{Alessandrini:1986identification}, the error estimate of coefficient reconstruction indicates that in order that coefficient can be recovered, the measurements have to reach an accuracy with an error margin under $O(h)$ where $h$ denotes the mesh size. 
Meanwhile, standard iterative techniques that construct $a_{\epsilon}$ to approximate $a$ relies on the regularity of the coefficient $a$ itself \cite[Section 3]{Alessandrini:1986identification}. 
This regularity assumption is largely violated in our problem setting, as the $a$ has sharp material interfaces (see Figure \ref{fig:ex2-interface}), has no extra regularity, and is only in $L^{\infty}$. 

Having the measurements on the discrete grid, the ideal goal is to learn the following inverse map $T_h$, 
\begin{equation}
  T_h: \bb{X}_h \mapsto \bb{A}_{h}, \quad \Pi_h u_{h^*} \mapsto a_{h},
\end{equation}
However, in practice the measurements (solution) could have noise. 
Therefore, we aim to learn the following discrete operator in this example, i.e., to reconstruct the coefficient on the coarse grid based on 
a noisy measurement on the fine grid. We note that this operator is not well-posed anymore due to noise.
\begin{equation}
\label{eq:darcy-operator-inverse}
  T_h: \bb{X}_{h_f} \mapsto \bb{A}_{h_c}, 
\quad u_{h_f} + \epsilon \nu_{h_f} \mapsto \Pi_{h_c} a_{h},
\end{equation}
where $h_c, h_f$ denotes the mesh size of the coarse grid $n_c\times n_c$, and the fine grid $n_f\times n_f$, respectively. $\Pi_{h_c}$ denotes a map that restricts a function $\bb{X}_h$ defined on $n_f\times n_f$ to $n_c\times n_c$.
$ \epsilon$ is the strength of the noise, and $\nu_{h_f}(x_i)\sim \mathcal{N}(0, c_i)$ where $c_i$ is the variance of all training samples' value at $x_i$. If $\epsilon=0.1$, we have 10\% of noise in the solution measurements for the training and testing data (see Figure \ref{fig:ex3-darcy-inv}).
\begin{figure}[htbp]
\begin{center}
\begin{subfigure}[b]{0.32\linewidth}
      \centering
      \includegraphics[height=1.2in]{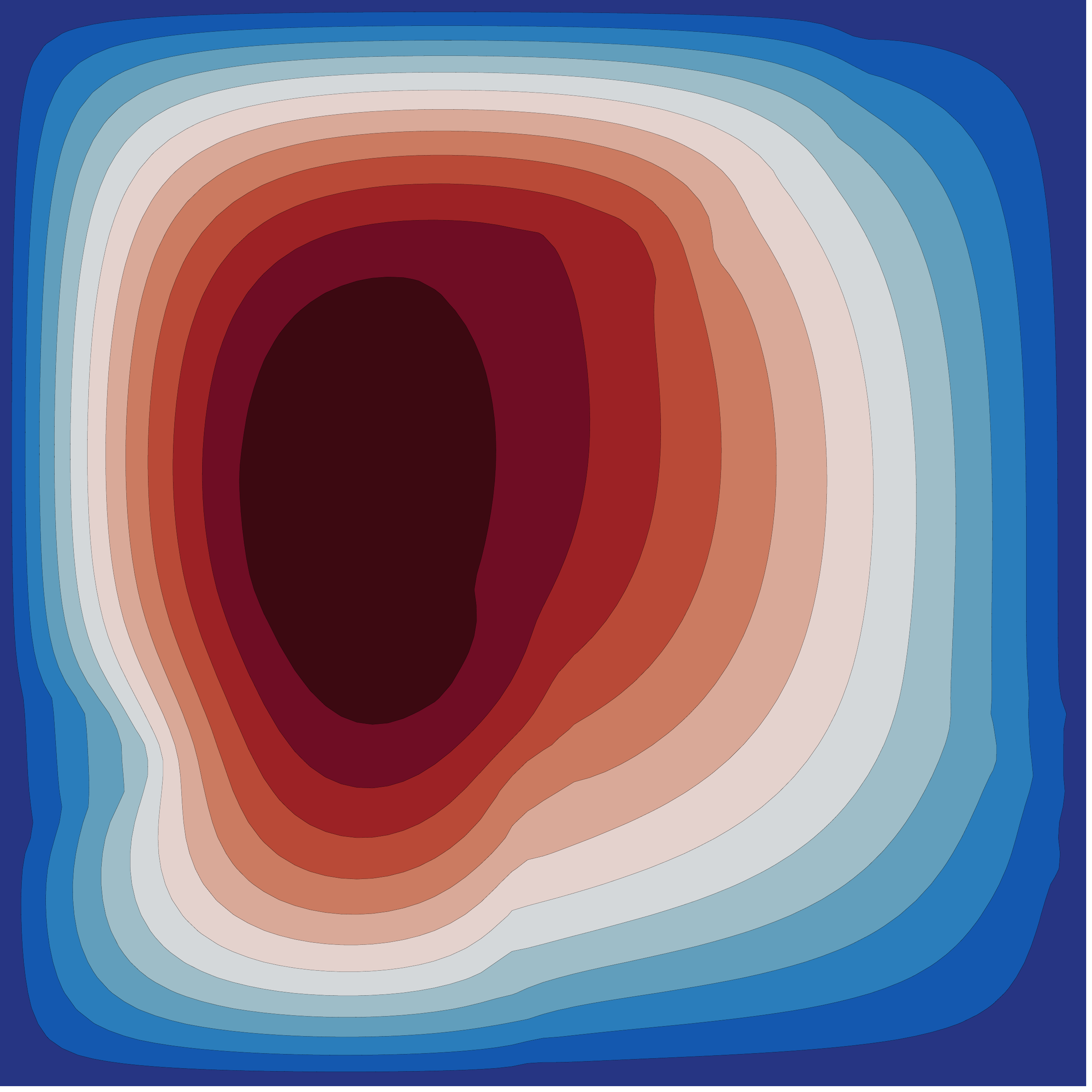}
      \caption{\label{fig:ex3-soln-0.0}}
\end{subfigure}
\;
\begin{subfigure}[b]{0.32\linewidth}
      \centering
      \includegraphics[height=1.2in]{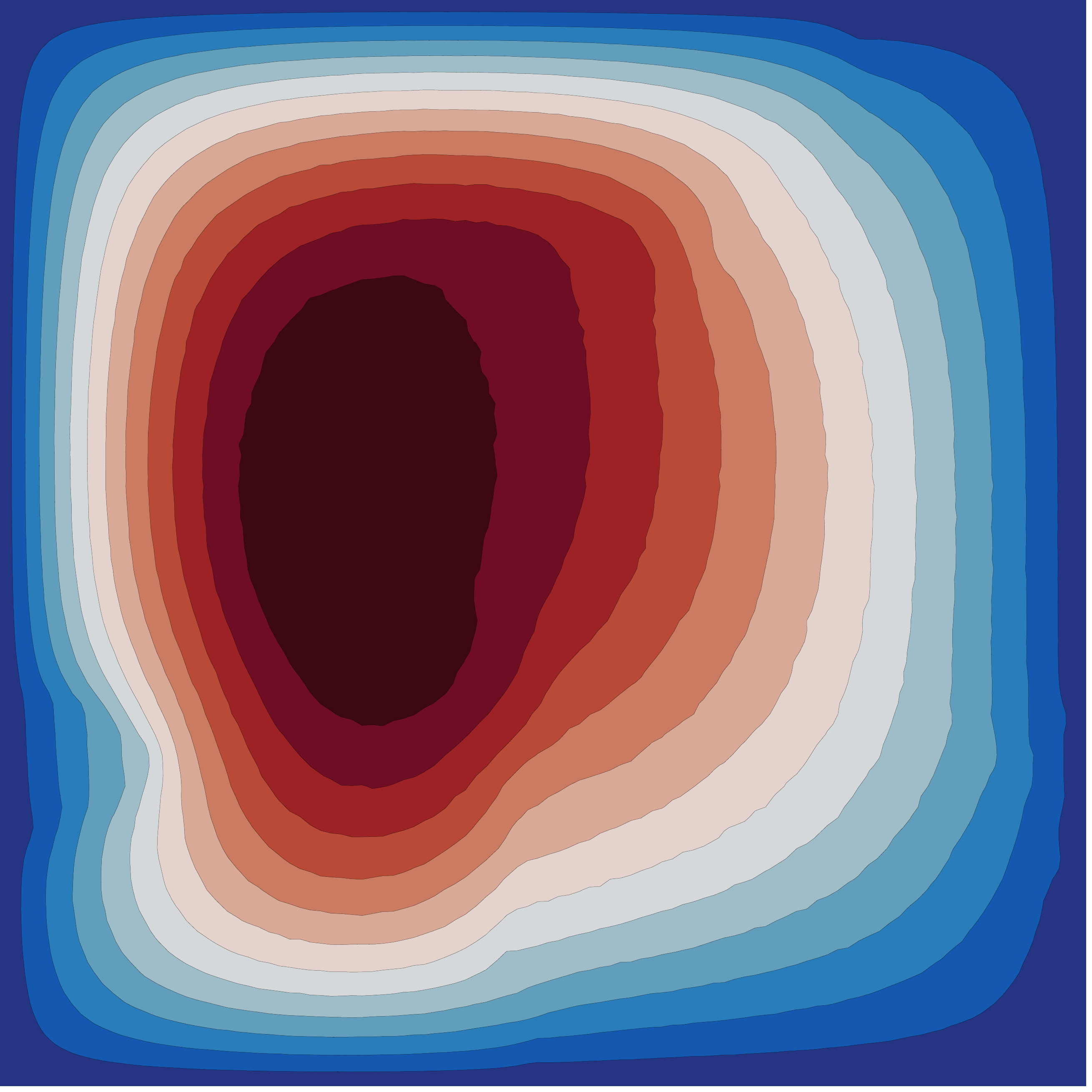}
      \caption{\label{fig:ex3-soln-0.01}}
\end{subfigure}
\;
\begin{subfigure}[b]{0.32\linewidth}
  \centering
  \includegraphics[height=1.2in]{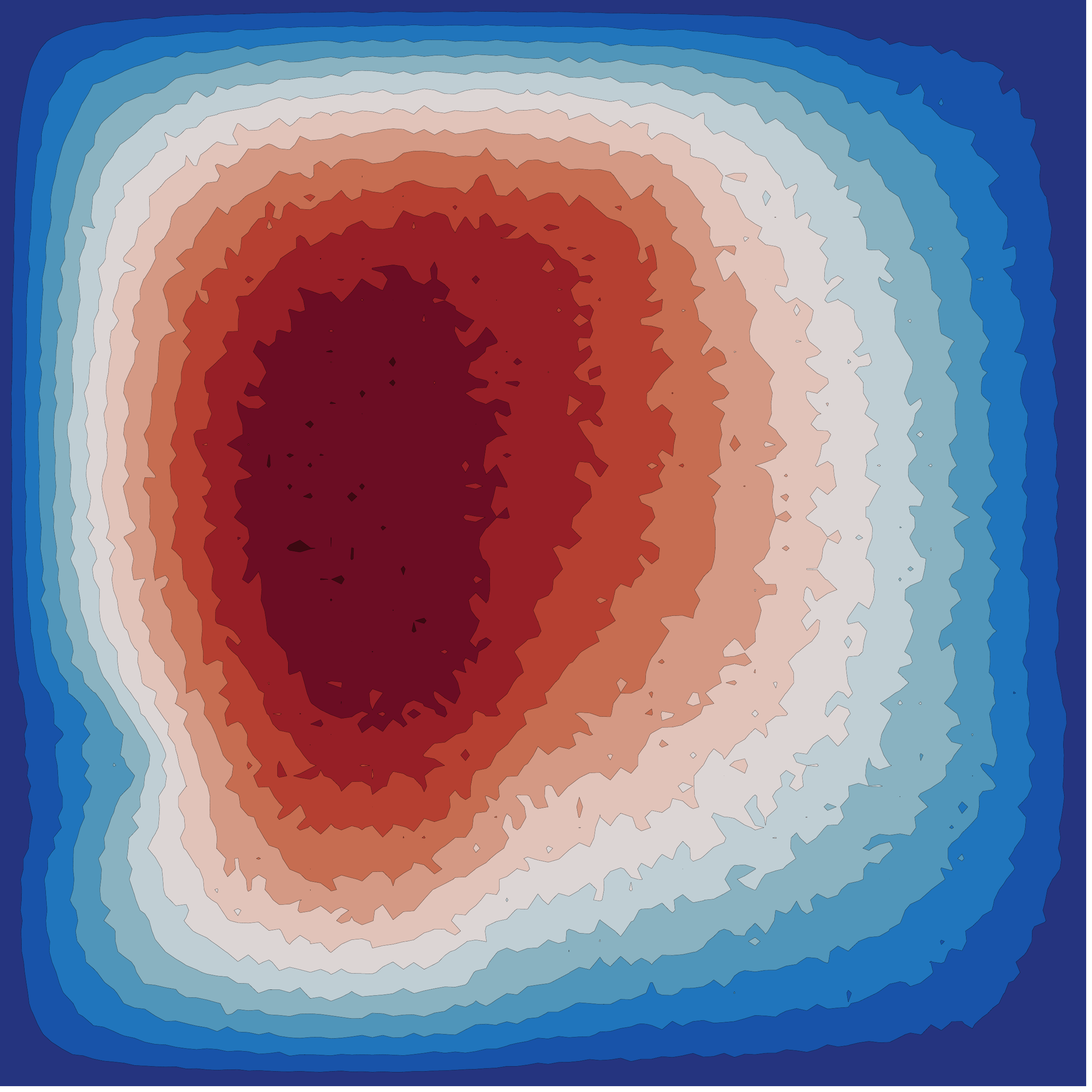}
  \caption{\label{fig:ex3-soln-0.1}}
\end{subfigure}%
\\ 
\begin{subfigure}[b]{0.32\linewidth}
  \centering
  \hspace*{0.1in}\includegraphics[height=1.16in]{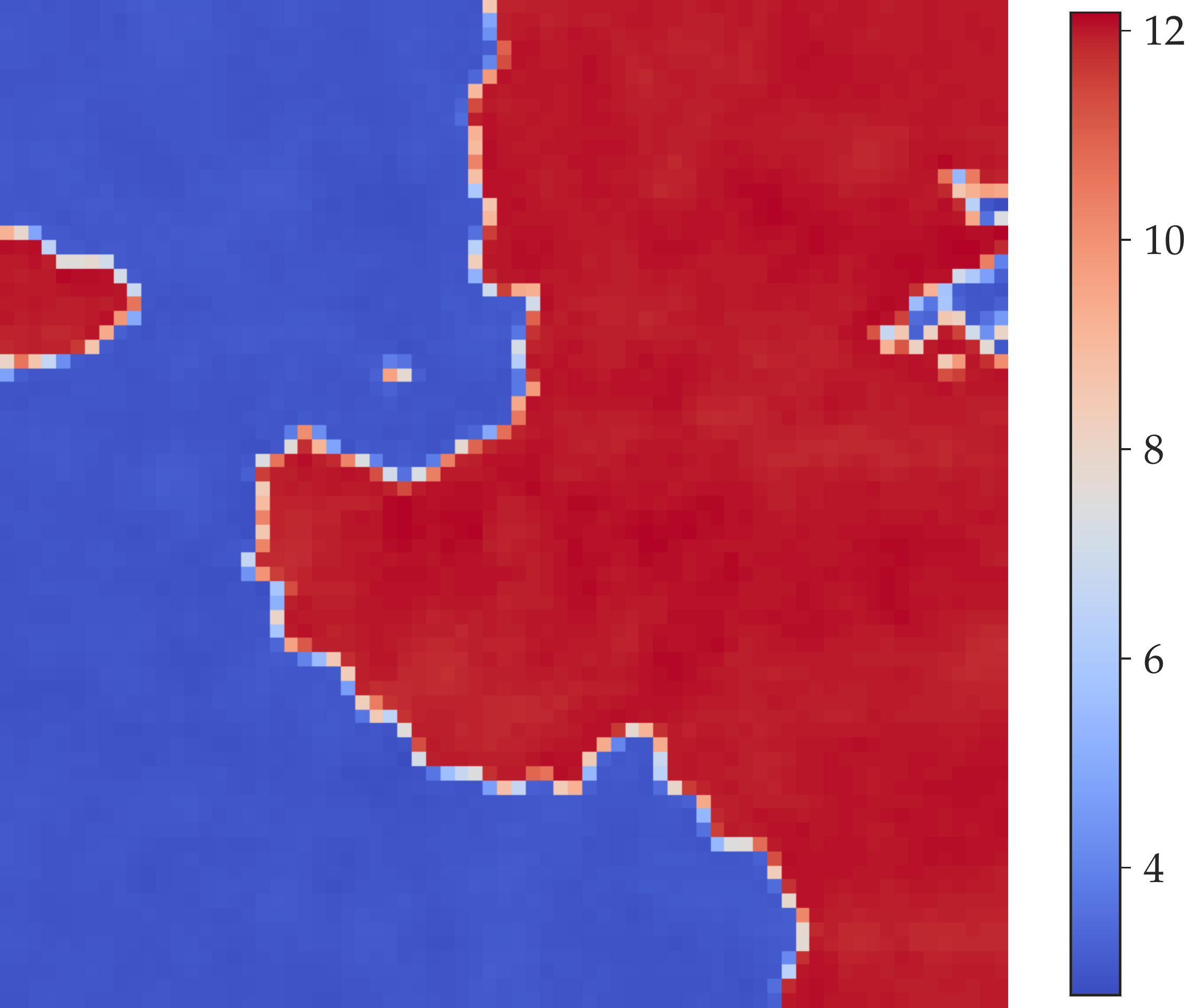}
  \caption{\label{fig:ex3-pred-0.0}}
\end{subfigure}%
\;
\begin{subfigure}[b]{0.32\linewidth}
  \centering
  \hspace*{0.15in}\includegraphics[height=1.16in]{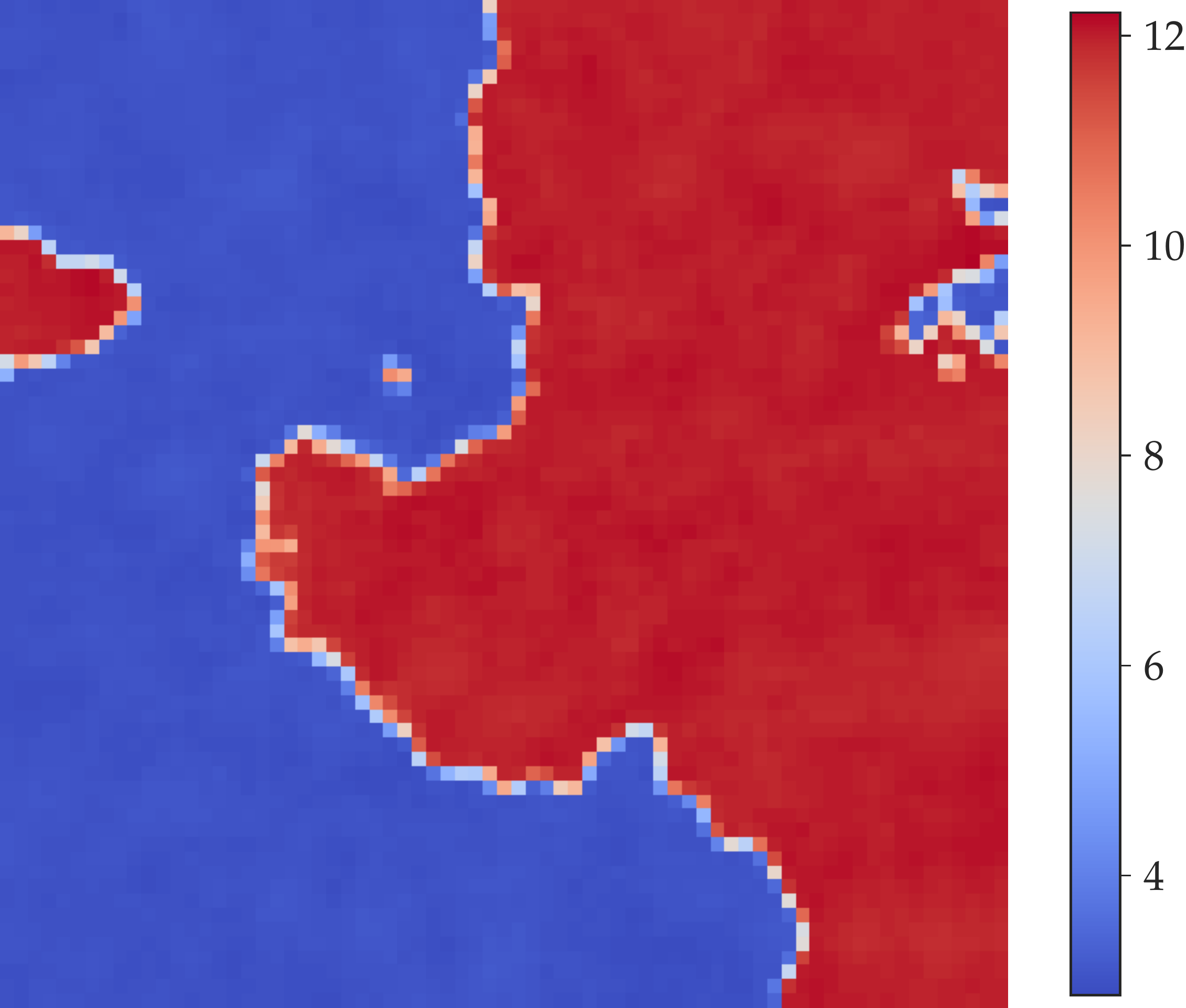}
  \caption{\label{fig:ex3-pred-0.01}}
\end{subfigure}%
\;
\begin{subfigure}[b]{0.32\linewidth}
  \centering
  \hspace*{0.18in}\includegraphics[height=1.16in]{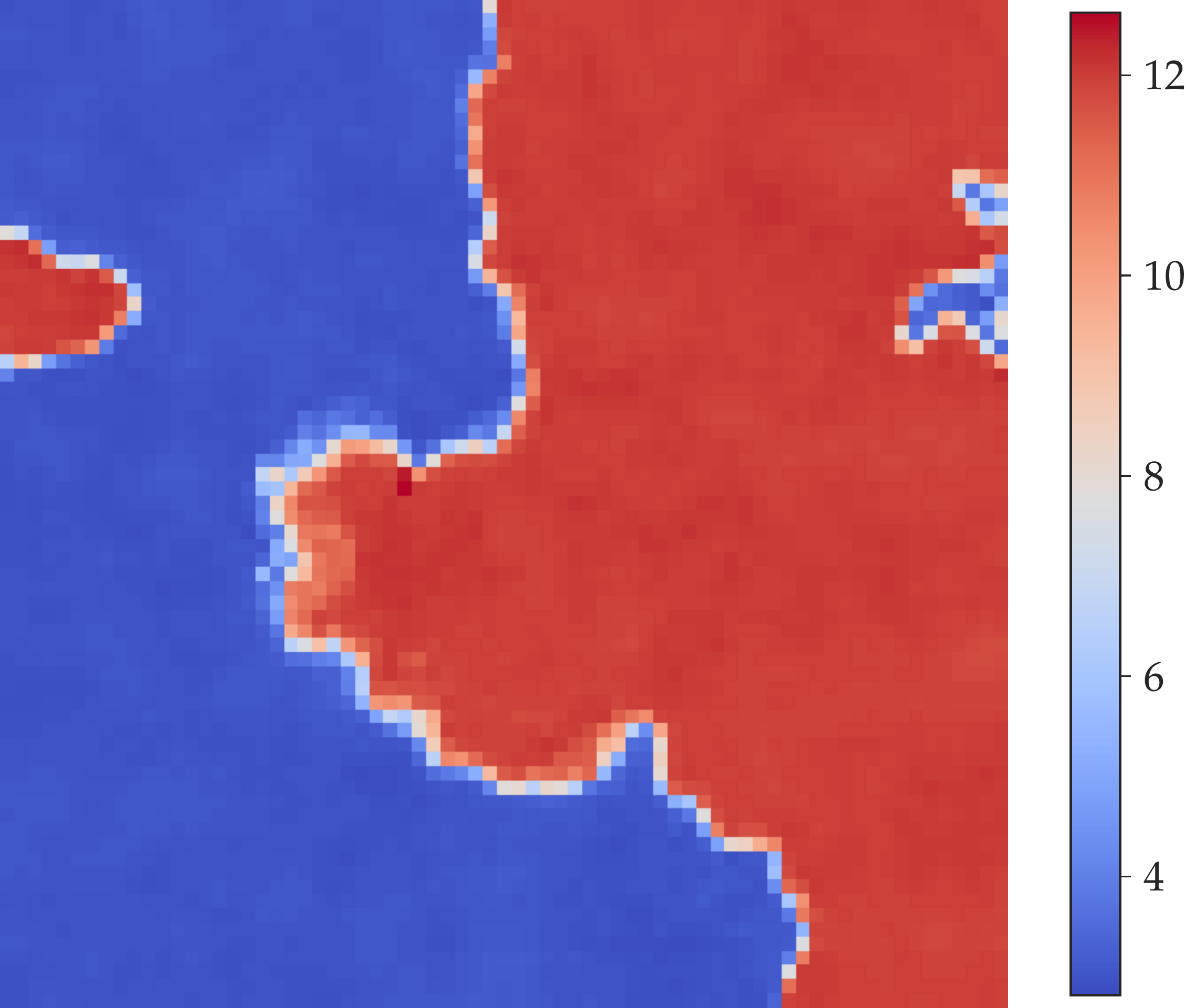}
  \caption{\label{fig:ex3-pred-0.1}}
\end{subfigure}%
\end{center}
\caption{Galerkin transformer evaluation for the inverse interface coefficient identification problem using the same sample with Figure \ref{fig:ex2-darcy}, model trained and evaluated under the same amount of noises: (\subref{fig:ex3-soln-0.0})--(\subref{fig:ex3-soln-0.1}) the input $u_h(x)$ with noise level $0$, $1\%$, and $10\%$ on $211\times 211$ grid; (\subref{fig:ex3-pred-0.0})--(\subref{fig:ex3-pred-0.1}) the recovered coefficient through evaluation with noise level $0$, $1\%$, and $10\%$ on $71\times 71$ grid with relative error being $0.0160$, $0.0292$, and $0.0885$, respectively.}
\label{fig:ex3-darcy-inv}
\end{figure}

\paragraph{Why not fine grid reconstruction?}
The reason we can only reconstruct the coarse grid coefficient is as follows. Since we use an upsampling interpolation from the coarse to fine grids, a limitation of the 2D operator learner structure in Figure \ref{fig:cnn-ft} is that it can approximate well if the target is smooth, and consists most combination of basis functions of lower frequencies. The low-frequency part, which can be roughly interpreted as the general trends, of the solution can be well-resolved by the coarse grid, then the operator learner benefits from the smoothing property of the operator a priori, as well as the approximation property of the interpolation operator. If the high frequency part of the target is prevailing due to low regularity (such as $L^{\infty}$), the model can only resolve the frequency up to of grid size $n_c\times n_c$ as the upsampling interpolation loses the approximation order (prolongation error estimate from coarse to fine grids, see e.g., \cite[Chapter 6.3]{Hackbusch:2013Multi}). 

\paragraph{More limitations.} Moreover, we do acknowledge another limitation of the proposed operator learner: it suffers from a common foe in many applications of deep learning, the instability with respect to the noise during evaluation. If the model is trained with clean data, in evaluation it becomes oversensitive to noises, especially in Example \ref{sec:darcy-inverse} due to the large Lipschitz constant in the original problem itself, which is further amplified by the black-box model. Therefore, we recommend adding certain amount of noises for inverse coefficient identification problems. See Figure~\ref{fig:ex3-darcy-inv-instable}.

\begin{figure}[htbp]
\begin{center}
\hspace*{0.03in}
\begin{subfigure}[b]{0.3\linewidth}
      \centering
      \hspace*{0.1in}\includegraphics[height=1.16in]{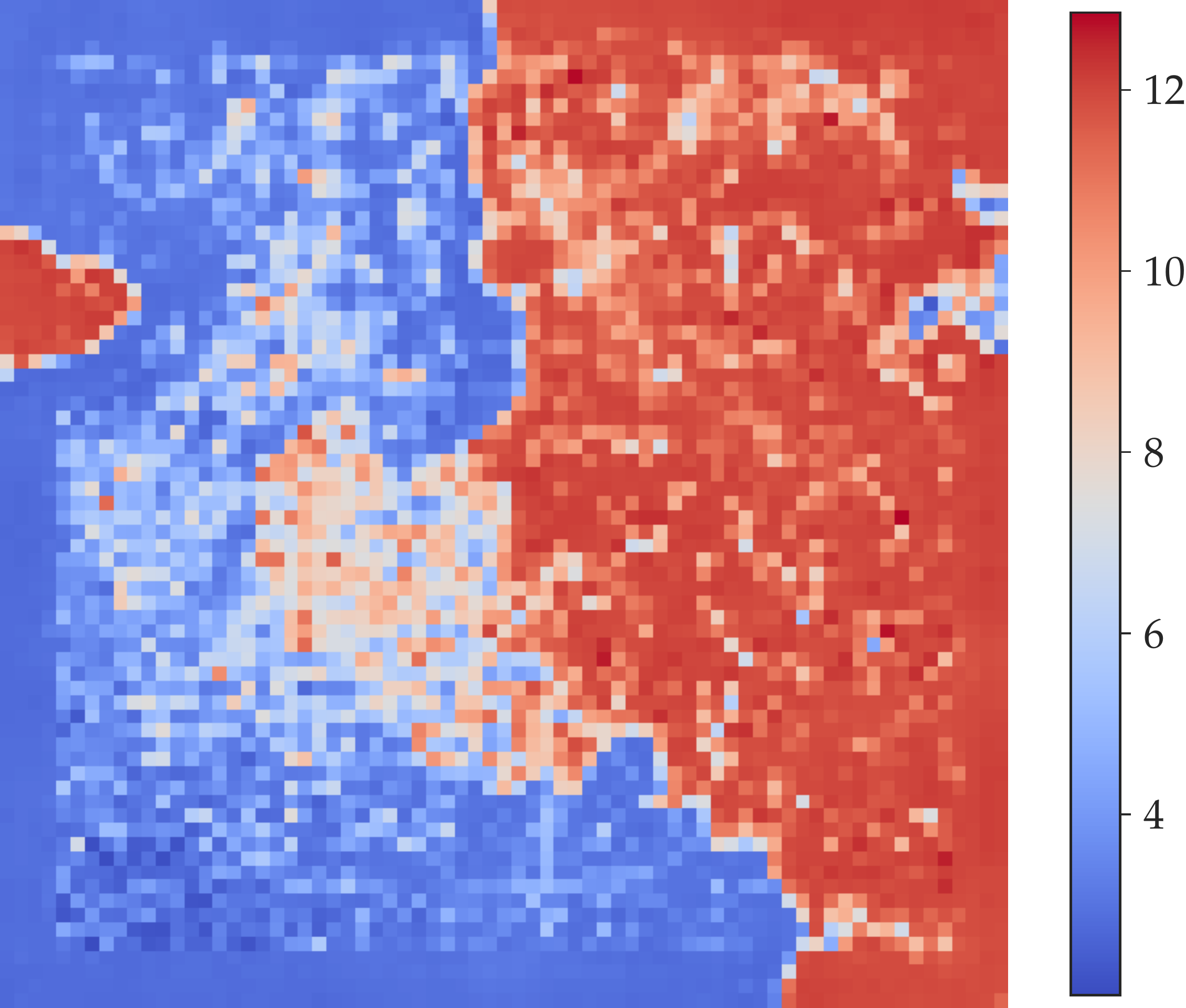}
      \caption{\label{fig:ex3-pred-0.01-train-0.0}}
\end{subfigure}
\hspace*{0.08in}
\begin{subfigure}[b]{0.3\linewidth}
      \centering
      \hspace*{0.08in}\includegraphics[height=1.16in]{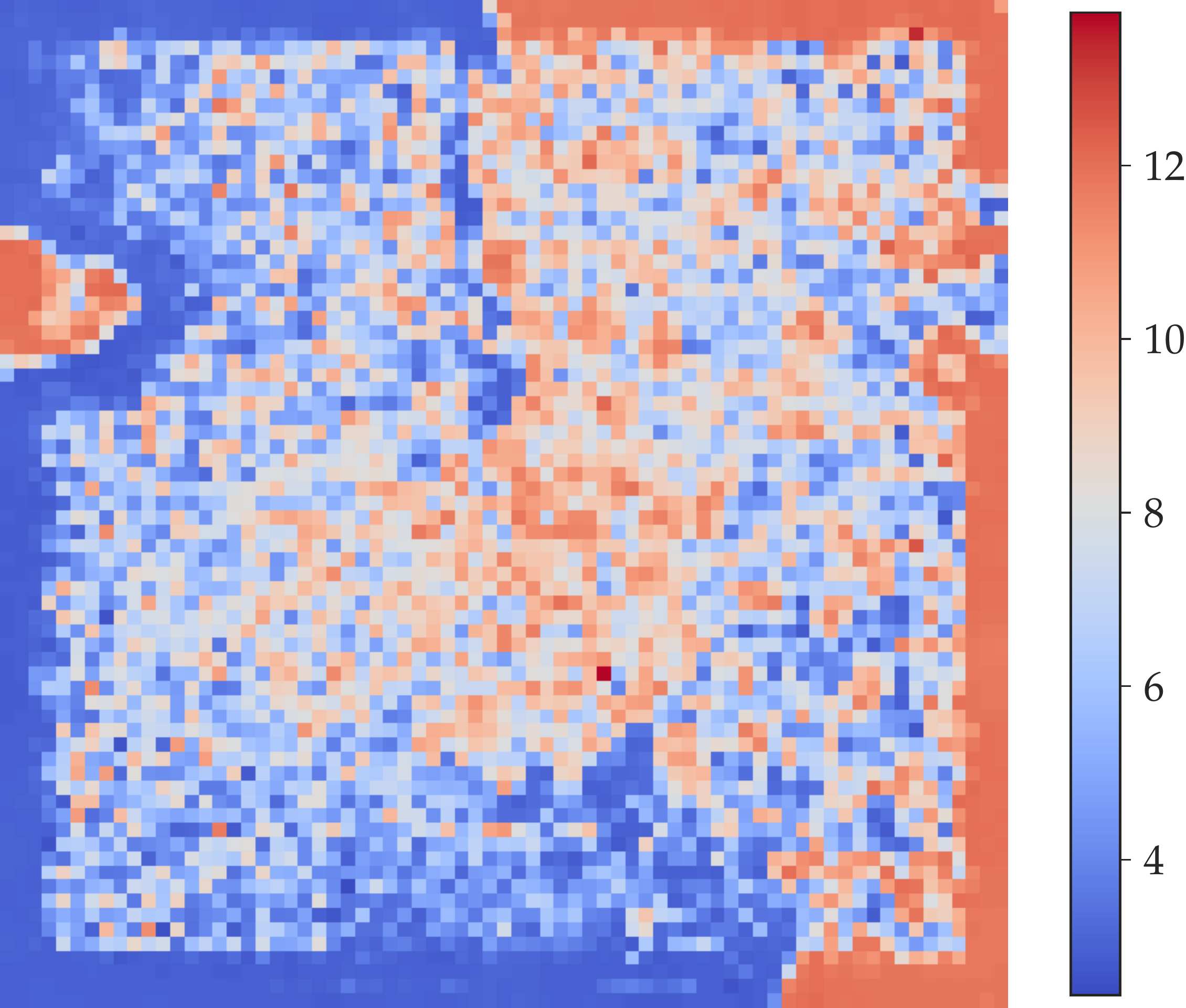}
      \caption{\label{fig:ex3-pred-0.02-train-0.0}}
\end{subfigure}
\hspace*{0.1in}
\begin{subfigure}[b]{0.3\linewidth}
  \centering
  \hspace*{0.1in}\includegraphics[height=1.16in]{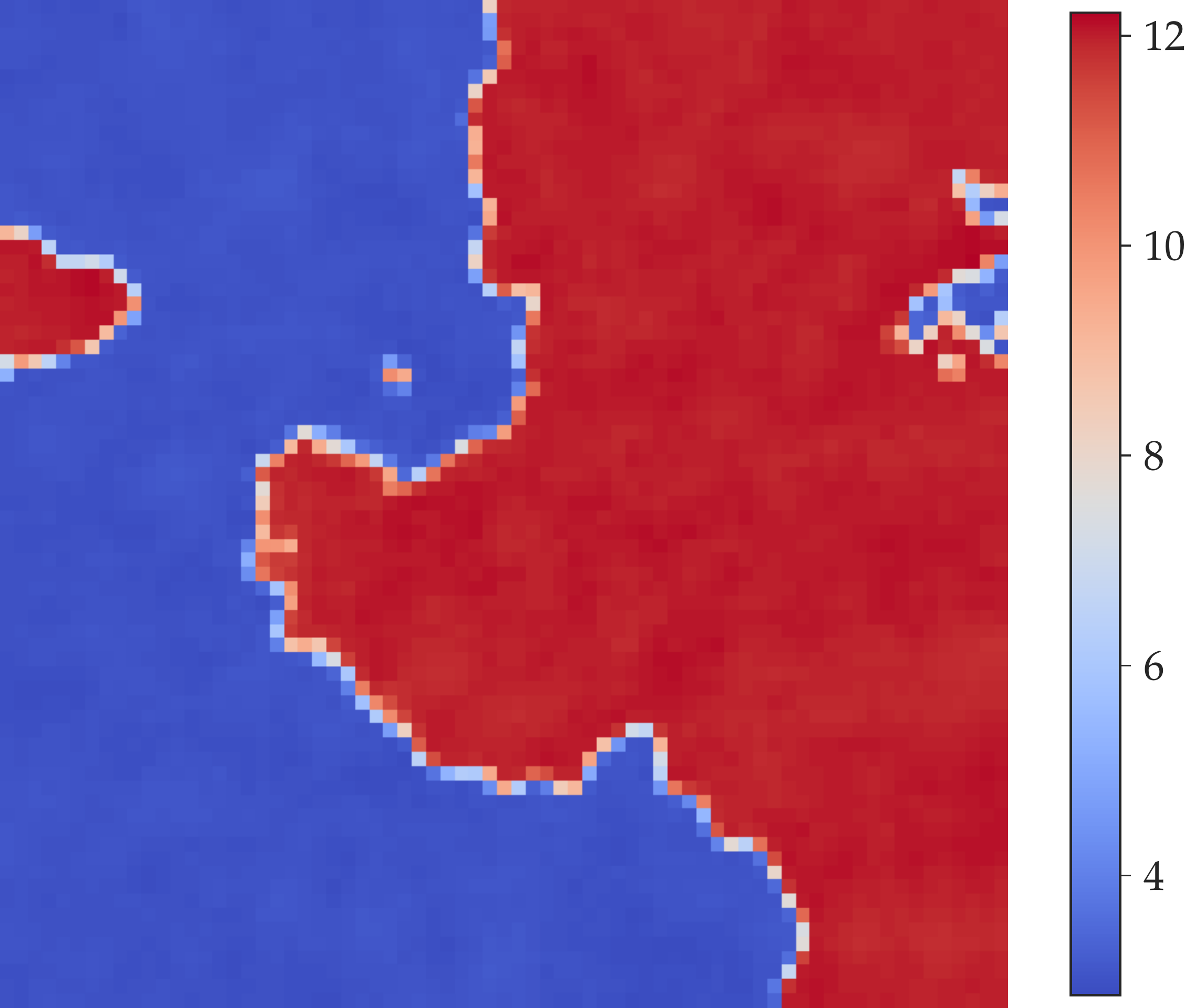}
  \caption{\label{fig:ex3-pred-0.0-train-0.01}}
\end{subfigure}%
\\
\begin{subfigure}[b]{0.3\linewidth}
  \centering
  \hspace*{0.1in}\includegraphics[height=1.16in]{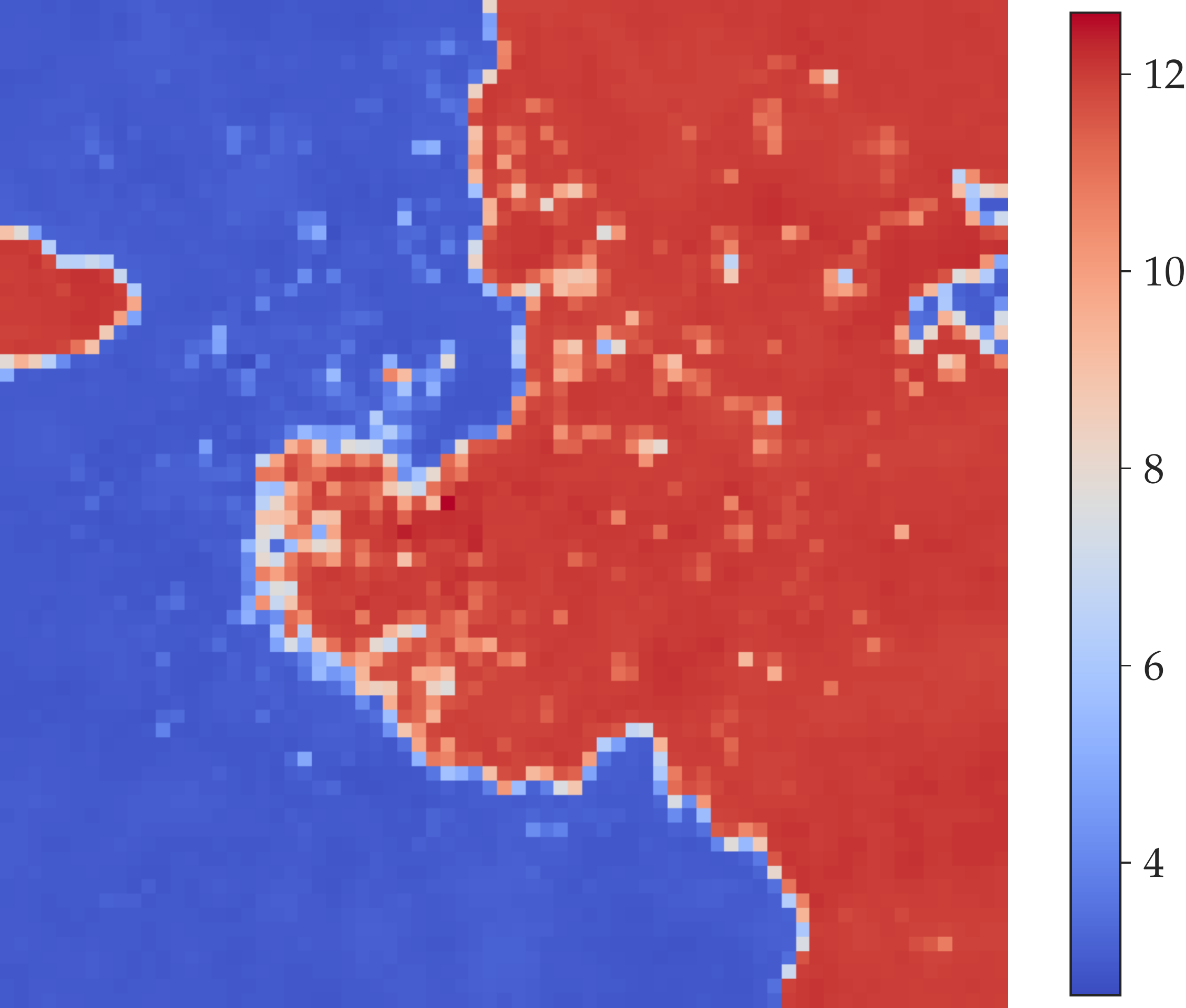}
  \caption{\label{fig:ex3-pred-0.02-train-0.01}}
\end{subfigure}%
\hspace*{0.1in}
\begin{subfigure}[b]{0.3\linewidth}
  \centering
  \hspace*{0.1in}\includegraphics[height=1.16in]{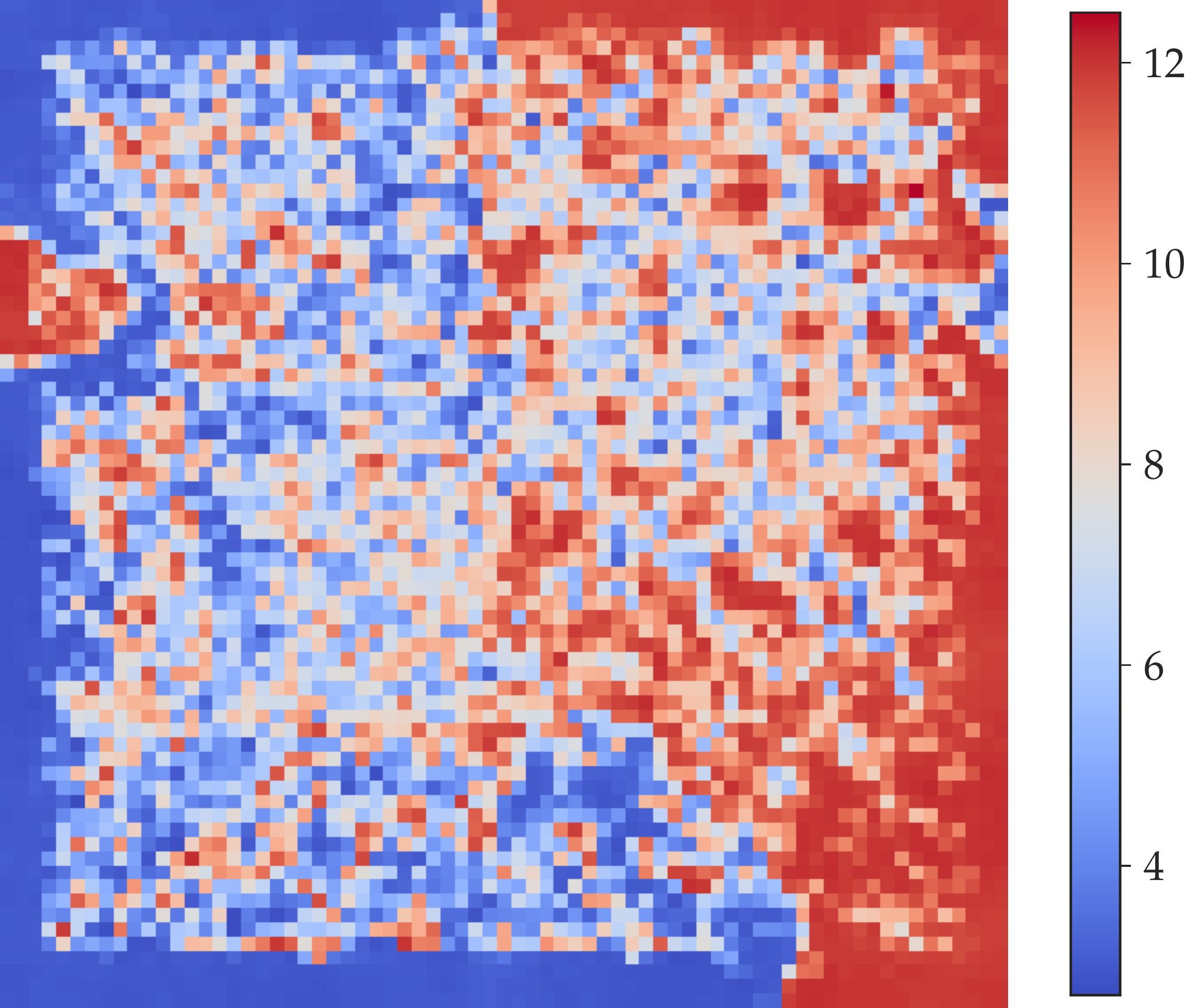}
  \caption{\label{fig:ex3-pred-0.05-train-0.01}}
\end{subfigure}%
\hspace*{0.1in}
\begin{subfigure}[b]{0.3\linewidth}
  \centering
  \hspace*{0.1in}\includegraphics[height=1.16in]{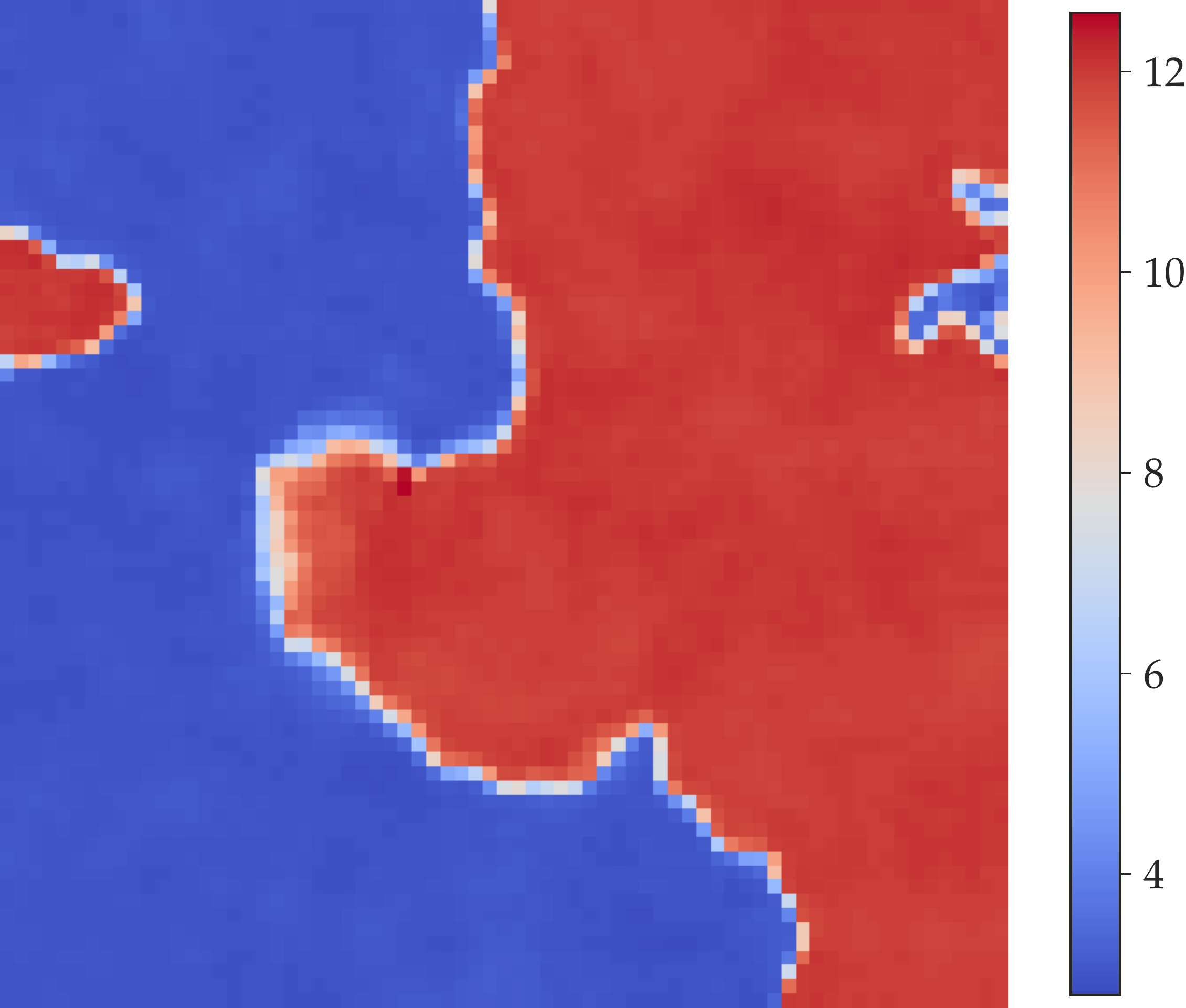}
  \caption{\label{fig:ex3-pred-0.05-train-0.1}}
\end{subfigure}%
\end{center}
\caption{Effect of noise in the inverse interface coefficient identification using the same sample with Figure \ref{fig:ex2-darcy} and \ref{fig:ex3-darcy-inv}, $\varepsilon$: relative error in $L^2$-norm. (\subref{fig:ex3-pred-0.01-train-0.0})--(\subref{fig:ex3-pred-0.02-train-0.0}) model trained with no noise, $1\%$ and $1.5\%$ noises in evaluation, $\varepsilon=0.194$ and $\varepsilon=0.416$; (\subref{fig:ex3-pred-0.0-train-0.01}) model trained with 1\% noise, no noise in evaluation, $\varepsilon=0.0235$; 
(\subref{fig:ex3-pred-0.02-train-0.01}) model trained with 1\% noise, 2\% in evaluation, $\varepsilon=0.0754$. (\subref{fig:ex3-pred-0.05-train-0.01}) model trained with 1\% noise, 5\% in evaluation, 
$\varepsilon=0.403$. (\subref{fig:ex3-pred-0.05-train-0.1}) model trained with 10\% noise, 5\% in evaluation, $\varepsilon=0.0691$.}
\label{fig:ex3-darcy-inv-instable}
\end{figure}

\newpage
\section{Proof of The Approximation Capacity of A Linear Attention}
\label{sec:appendix-cea}

In this section, we prove Theorem \ref{theorem:cea-lemma}, which shows that the linear attention variant we proposed in \eqref{eq:attention-galerkin}, Galerkin-type attention operator (nonlinear), is capable of replicating explicitly a Petrov-Galerkin projection (linear) in the current latent representation subspace under a Hilbertian setup. To better elaborate our Galerkin projection-inspired modifications to the attention operator, in Section \ref{sec:appendix-cea-background}, the technical background that bridges \eqref{eq:attention-galerkin-projection} to a learnable Petrov-Galerkin projection is presented. 
Then, some historical contexts are provided in Section \ref{sec:appendix-cea-overview} for an overview of Theorem \ref{theorem:cea-lemma}, connecting the sequence-length invariant training of the attention operator to how important a theorem like Theorem \ref{theorem:cea-lemma} is for an operator approximation problem in traditional applied mathematics.
In Section \ref{sec:appendix-cea-proof}, the proof of Theorem \ref{theorem:cea-lemma} is shown with a full array of mathematically rigorous setting and assumptions. 
Thereafter, in Section \ref{sec:appendix-cea-generalizations} some possible generalizations are discussed, together with the role of removing the softmax in obtaining a sequence-length uniform bound, as well as the heuristic behind the choice of the Galerkin projection-type normalization in the scaled dot-product attention. Last but not least, technical lemmata that are needed to show Theorem \ref{theorem:cea-lemma} are proved in Section \ref{sec:appendix-cea-aux}.

\subsection{Background on Galerkin methods}
\label{sec:appendix-cea-background}
The huge success of many Galerkin-type methods in approximating solutions to operator equations such as PDEs \cite{Ciarlet:2002finite} attributes partly to the following two fundamental properties of Hilbert spaces (see e.g., \cite[Chapter 4]{Ciarlet:2013Linear}): for $(\mathcal{H}, \langle\cdot,\cdot\rangle)$
\begin{itemize}[leftmargin=1.5em]
\item Let $\mathcal{Y}$ be a convex and complete subset of $\mathcal{H}$, and  $\mathcal{Y}$ is potentially infinite-dimensional. For any $f\in \mathcal{H}$, the projection $\Pi f \in \mathcal{Y}$ is uniquely determined by 
  \begin{equation}
    \|f - \Pi f\|_{\mathcal{H}} = \inf_{y\in \mathcal{Y}} \|f - y\|_{\mathcal{H}},
  \end{equation}
  i.e., the projection recovers the unique element in $\mathcal{Y}$ that is ``closest'' to $f$.
\item If $\mathcal{H}$ is infinitely-dimensional and separable, then there exists a set of orthogonal basis functions $\{q_l(\cdot)\}_{l=1}^{\infty}$ such that any $f\in \mathcal{H}$ can have its Fourier series expansion: 
  \begin{equation}
    f(\cdot) = \sum_{l=1}^{\infty} a_l q_l(\cdot) 
    := \sum_{l=1}^{\infty} 
    \frac{\langle f, q_l\rangle }{\langle q_l, q_l\rangle } q_l(\cdot),
  \end{equation}
  i.e., $\mathcal{H}$ can be identified by $\ell^2$ (the space that the Fourier coefficients $\{a_l\}_{l=1}^{\infty}$ are in). Thus, given any fixed tolerance under the norm induced by the inner product, any $f\in \mathcal{H}$ can be approximated using a finite number of basis $\{q_l(\cdot)\}_{l=1}^d$ (identified by a finite number of coefficients) thanks to the square summability.
\end{itemize}
Together, they rationalize the practice of using a finite dimensional vector (space) to approximate any element in $\mathcal{H}$, or a function in the solution subspace of an operator equation that is compact in $\mathcal{H}$. Consider the finite dimensional approximation space (trial space)
$\mathbb{Q}_h:= \operatorname{span}\{\tilde{q}_l(\cdot)\}_{l=1}^d$ (that is convex and closed), where $\tilde{q}_l(\cdot) := q_l(\cdot)/\|q_l\|_{\mathcal{H}}$. The projection $\Pi f$ onto $\mathbb{Q}_h$ is the best approximator in $\mathbb{Q}_h$ to $f\in \mathcal{H}$:
\begin{equation}
\label{eq:projection}
  \|f - \Pi f\|_{\mathcal{H}} = \inf_{q\in \mathbb{Q}_h} \|f - q\|_{\mathcal{H}}.
\end{equation}
Exploiting the definition of $\|\cdot\|_{\mathcal{H}}$, any perturbation $q\in \mathbb{Q}_h$ (test space) to the unique (local) minimizer $p:=\Pi f$ shall increase the difference in $\norm{\cdot}_{\mathcal{H}}$, thus 
\[
0 =\lim_{\tau\to 0}\frac{d}{d\tau} 
\|f -  (p+\tau q) \|_{\mathcal{H}}^2 =\lim_{\tau\to 0}\frac{d}{d\tau}
\langle f -  (p+\tau q), f -  (p+\tau q) \rangle,
\]
Choosing $q$ as $\tilde{q}_l$ ($l=1,\dots, d$), one shall obtain the following if assuming that $\mathcal{H} = L^2(\Omega)$ in a simple case
\begin{equation}
\label{eq:projection-galerkin}
\Pi f(x) = \sum_{l=1}^d \langle f, \tilde{q}_l\rangle \tilde{q}_l(x) 
= \sum_{l=1}^d \left( \int_{\Omega} f(\xi) \tilde{q}_l(\xi) \dd \xi \right) \tilde{q}_l(x), \quad \text{ for } x\in \Omega.
\end{equation}
When the set $\{f_j(\cdot)\}_{j=1}^d$ is projected onto $\mathbb{Q}_h$ element by element, \eqref{eq:projection-galerkin} carries a resoundingly similar form to that of \eqref{eq:attention-galerkin-projection}. In light of proving the approximation capacity of \eqref{eq:attention-galerkin-projection}, the differences are: 
\begin{itemize}[leftmargin=2.5em]
  \item[(a)] The test and trial function spaces are the same in the ideal case above \eqref{eq:projection-galerkin}, while in a Galerkin-type attention operator \eqref{eq:attention-galerkin-projection}, they are different and become learnable. This difference brings the Petrov-Galerkin projection into the picture. In Section \ref{sec:appendix-cea-proof}, we shall see that the minimization is done for a more general functional (dual) norm $\|\cdot\|_{\mathbb{V}_h'}$ (min-max problem \eqref{eq:min-max}), instead of $\|\cdot\|_{\mathcal{H}}$.
  \item[(b)] $\{\tilde{q}_l(\cdot)\}_{l=1}^d$ needs to be orthonormal to yield a compact formula as \eqref{eq:projection-galerkin}. In the forward propagation of the attention operations, there is no such guarantee unless certain orthogonalization/normalization is performed. We shall see in the proof in Section \ref{sec:appendix-cea-proof}, the Galerkin projection-type layer normalization acts as a cheap learnable alternative to the normalization shown in the explicit formula Petrov-Galerkin projection (inverse of the Gram matrices in \eqref{eq:attention-update}).
\end{itemize}

\subsection{Overview of Theorem \ref{theorem:cea-lemma}}
\label{sec:appendix-cea-overview}
\paragraph{Historical context.}  
\label{paragraph:cea-history}
Theorem \ref{theorem:cea-lemma} resembles the famous C\'{e}a's lemma (e.g., see \cite[Theorem 2.8.1]{Brenner.Scott:2008mathematical}, \cite[Theorem 2.4.1]{Ciarlet:2002finite}). 
It is one of the most fundamental theorems 
in approximating an operator equation such as a PDE under the Hilbertian framework. 
Define the operator norm to be the induced norm from the original Hilbertian norm, the C\'{e}a's ``lemma'' reads: if the norm of the operator associated with the bilinear form is bounded below (either by the Lax-Milgram lemma or the Ladyzhenskaya–Babu\v{s}ka–Brezzi inf-sup condition) in an approximation subspace of the original Hilbert space, which implies its invertibility, then this mere invertibility implies the quasi-optimality (optimal up to a constant) of the Galerkin-type projection to a given function. 
This says: in the current approximation space, measured by the distance of the Hilbertian norm, the Galerkin-type or the Petrov-Galerkin-type projection (see e.g., \cite{Stern:2015Banach}) is equivalent to the closest possible approximator to any given target function (see \eqref{eq:projection}).

As might be expected, whether this approximation space has enough approximation power, such that this closest approximator is actually close to the target, is another story. This closeness, either in terms of distance or structure, is usually referred as a part of ``consistency'' in the context of approximating a PDE's solution operator. 
Any method with a sufficient approximation power together with the invertibility has ``convergence'', and this is also known as the Lax equivalence principle \cite{Lax.Richtmyer:1956Survey}. 

\paragraph{Heuristic comparison: convergence of traditional numerical methods versus a data-driven operator learner.} 
\label{paragraph:cea-convergence}
For a traditional numerical method that approximates an operator equation to be successful, in that scientists and engineers can trust the computer-aided simulations, the aforementioned convergence is indispensable. One key difference of a traditional numerical method to an attention-based operator learner is how the ``convergence'' is treated.
\begin{itemize}[leftmargin=1.5em]
  \item A traditional numerical method: once the discretization is fixed (fixed degrees of freedom such as grids, radial bases, Fourier modes, etc.), the approximation power of this finite dimensional approximation space is fixed. The method seeks the best approximator to a single instance in this fixed space. The convergence refers to the process of error decreasing as one continuously enlarges the approximation subspaces (e.g., grid refinement, more Fourier modes).
  \item An attention-based operator learner: the approximation space is not fixed, and is constantly replenished with new bases during optimization (e.g., see the remarks in \ref{paragraph:ffn}), thus able to approximate an operator's responses in a subset/subspace in a much more dynamic manner. A possible exposition of the ``convergence'' is more problem-dependent as one obtains a better set of basis to characterize the operator of interest progressively through the stochastic optimization.
\end{itemize}

\paragraph{Positional encoding and a dynamic feature/basis generation.}
\label{paragraph:dynamic-basis-update}
Even though \cite{Vaswani;Shazeer;Parmar:2017Attention} opens a new era in NLP by introducing the state-of-the-art Transformer, in an operator learning problem, we find that its explanation unsatisfactory on the absolute necessity to include the positional encodings in the attention mechanism. 
From the proof of Theorem \ref{theorem:cea-lemma} we see that, if we ought to learn the identity operator from $\mathbb{Q}_h$ to $\mathbb{Q}_h$ (with a few other caveats), i.e., the target $f$ itself is in the approximation subspace $\mathbb{Q}_h$, then the attention mechanism has certainly the capacity to learn this operator exactly without any positional encodings. 
We want to emphasize that in addition to the remarks in \ref{paragraph:ffn}, in our interpretation, the utmost importance of the positional encoding is to make the approximation subspaces dynamic. Otherwise, the Galerkin-type attention (or a linear attention) is simply a linear combination of the current approximation subspace (or a convex combination in its softmax-normalized siblings). Consequently, the approximation power of the subspaces cannot enjoy this dynamic update mechanism through optimizations. In this sense, we can view the attention mechanism a universal ``dynamic feature/basis generator'' as well. For an empirical evidence of this layer-wise dynamic basis update in our experiments of the Darcy interface flow please refer to Figure \ref{fig:darcy-latent}.

\begin{figure}[htbp]
\begin{center}
\hspace*{0.03in}
\begin{subfigure}[b]{0.3\linewidth}
      \centering
      \hspace*{0.1in}\includegraphics[height=1.16in]{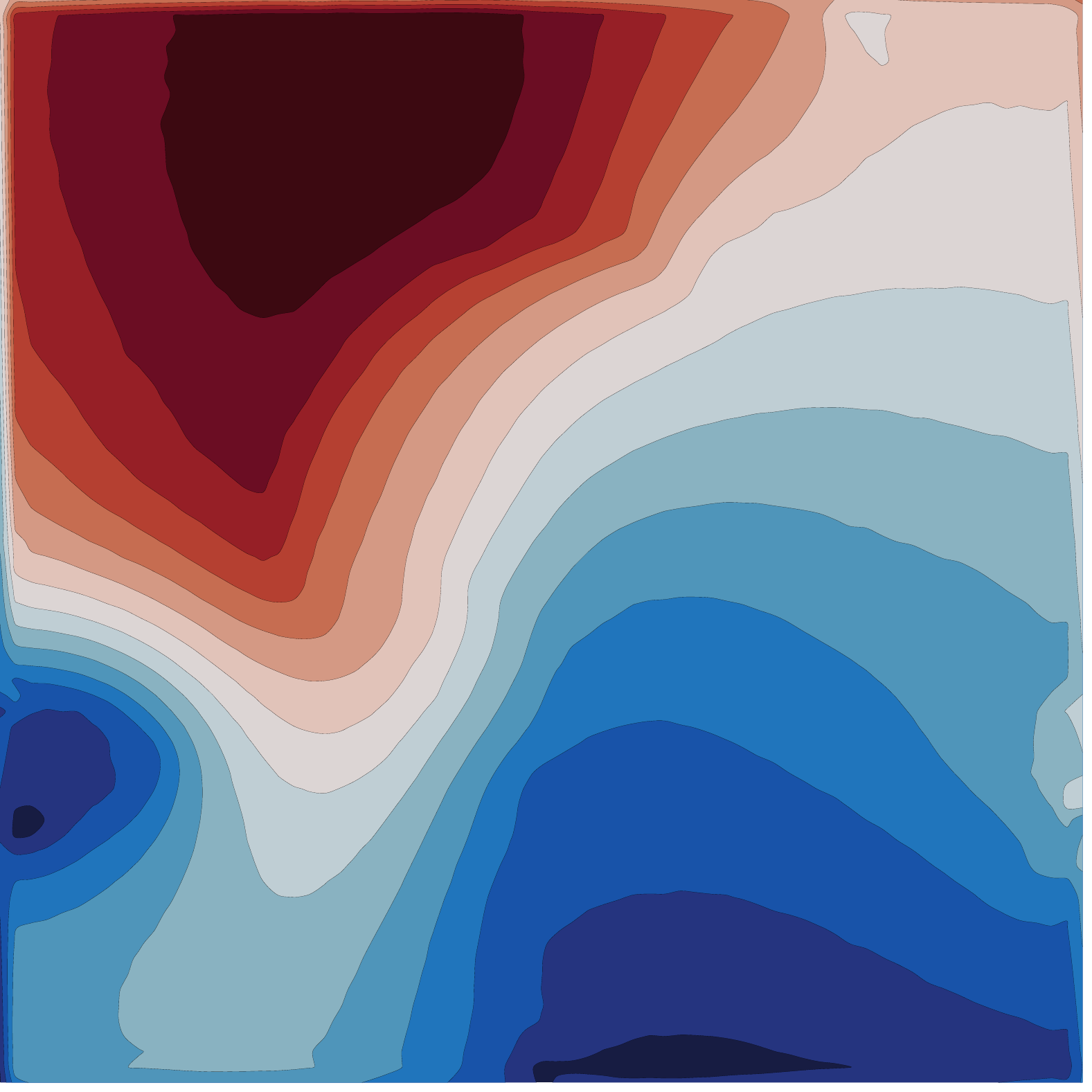}
      \caption{\label{fig:ex3-l1-f21}}
\end{subfigure}
\hspace*{0.08in}
\begin{subfigure}[b]{0.3\linewidth}
      \centering
      \hspace*{0.08in}\includegraphics[height=1.16in]{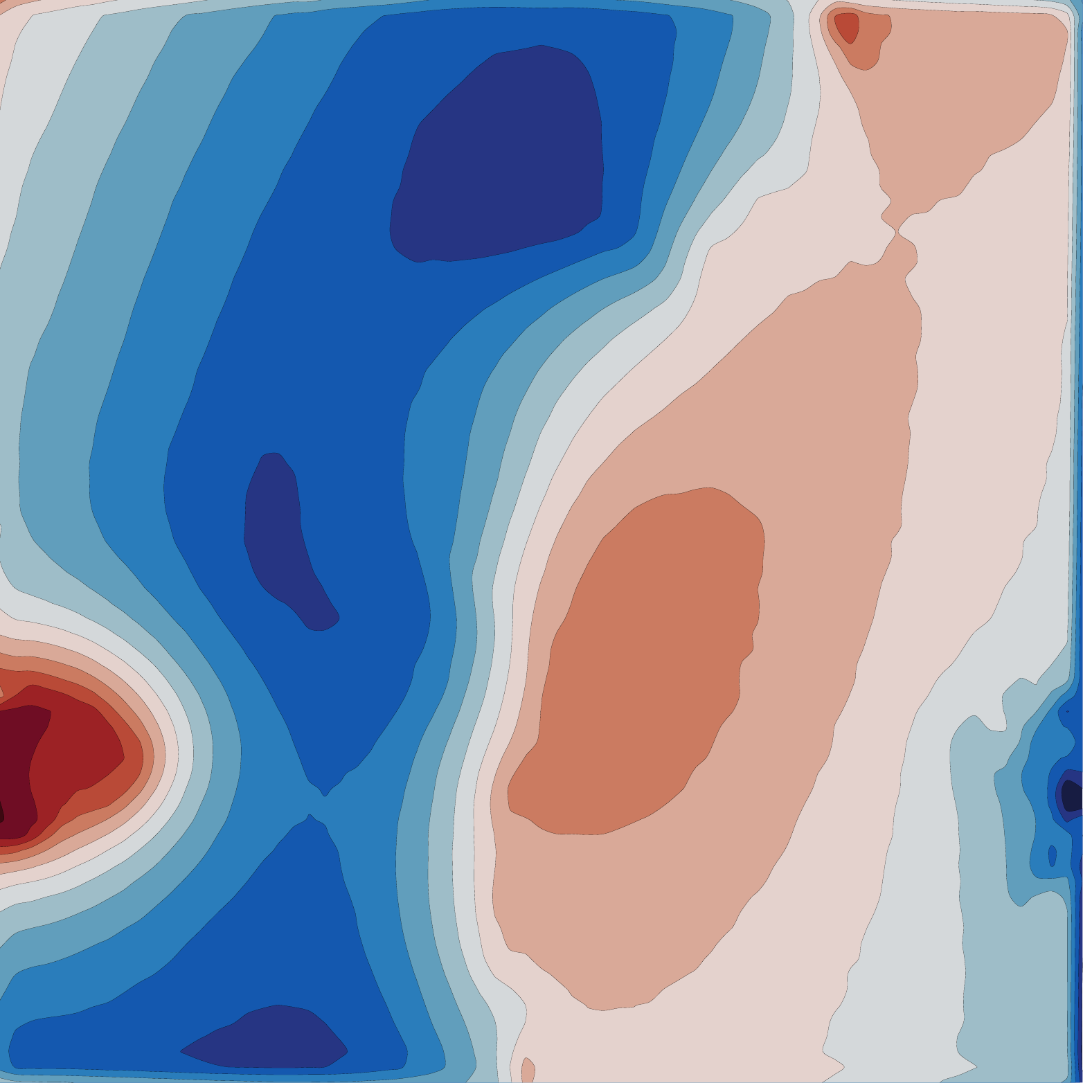}
      \caption{\label{fig:ex3-l1-f28}}
\end{subfigure}
\hspace*{0.1in}
\begin{subfigure}[b]{0.3\linewidth}
  \centering
  \hspace*{0.1in}\includegraphics[height=1.16in]{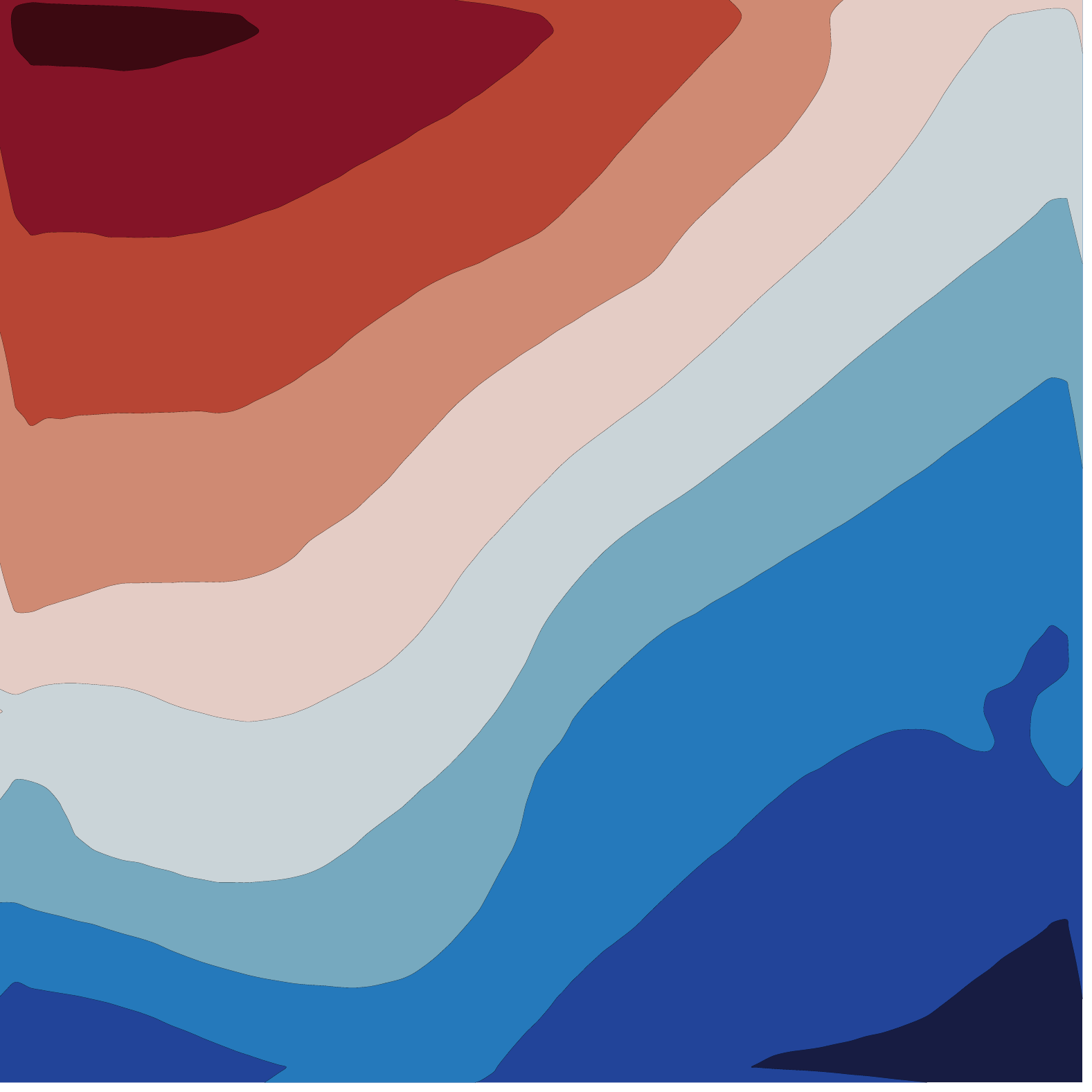}
  \caption{\label{fig:ex3-l4-f26}}
\end{subfigure}%
\\
\begin{subfigure}[b]{0.3\linewidth}
  \centering
  \hspace*{0.1in}\includegraphics[height=1.16in]{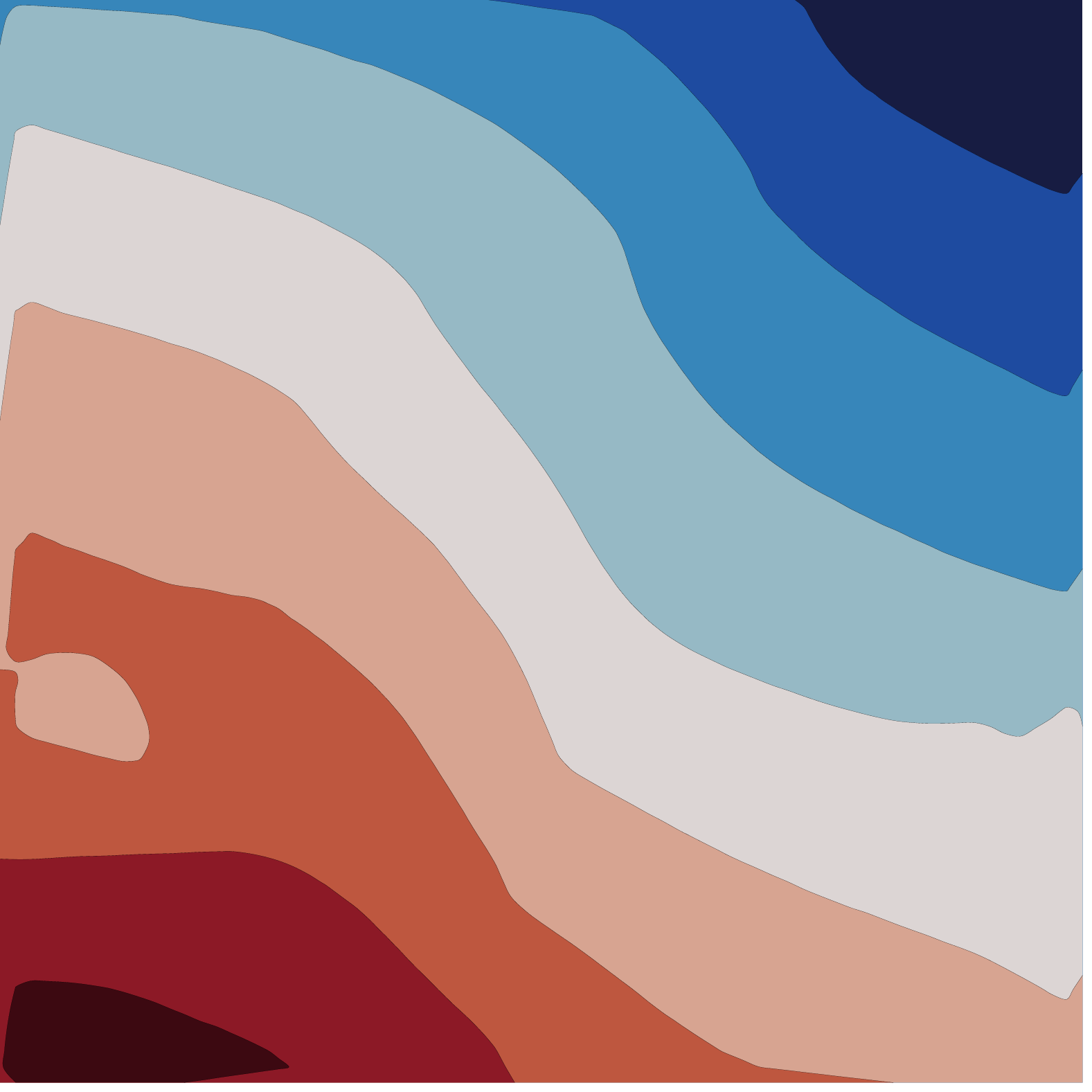}
  \caption{\label{fig:ex3-l4-f29}}
\end{subfigure}%
\hspace*{0.1in}
\begin{subfigure}[b]{0.3\linewidth}
  \centering
  \hspace*{0.1in}\includegraphics[height=1.16in]{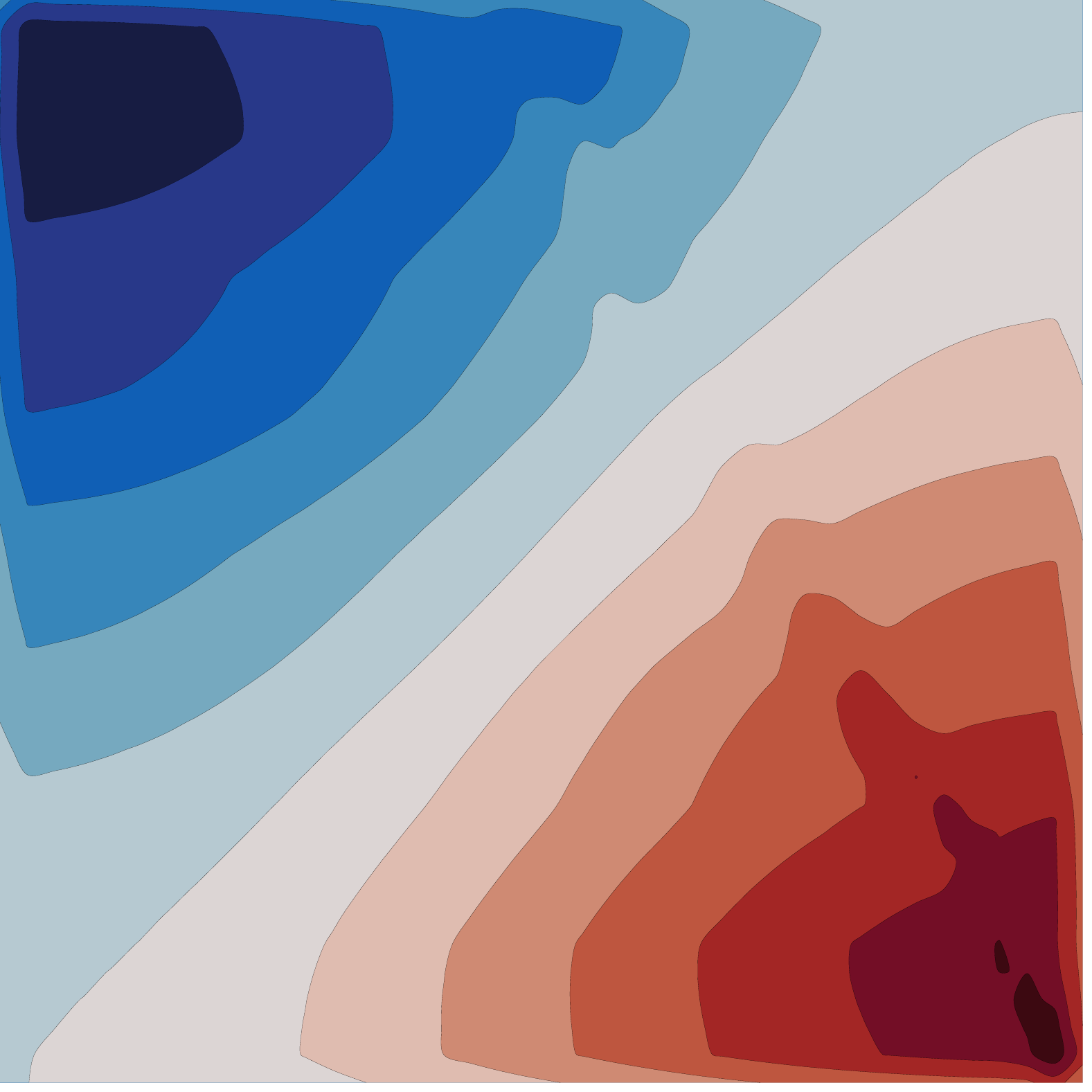}
  \caption{\label{fig:darcy-eigenvec-1}}
\end{subfigure}%
\hspace*{0.1in}
\begin{subfigure}[b]{0.3\linewidth}
  \centering
  \hspace*{0.1in}\includegraphics[height=1.16in]{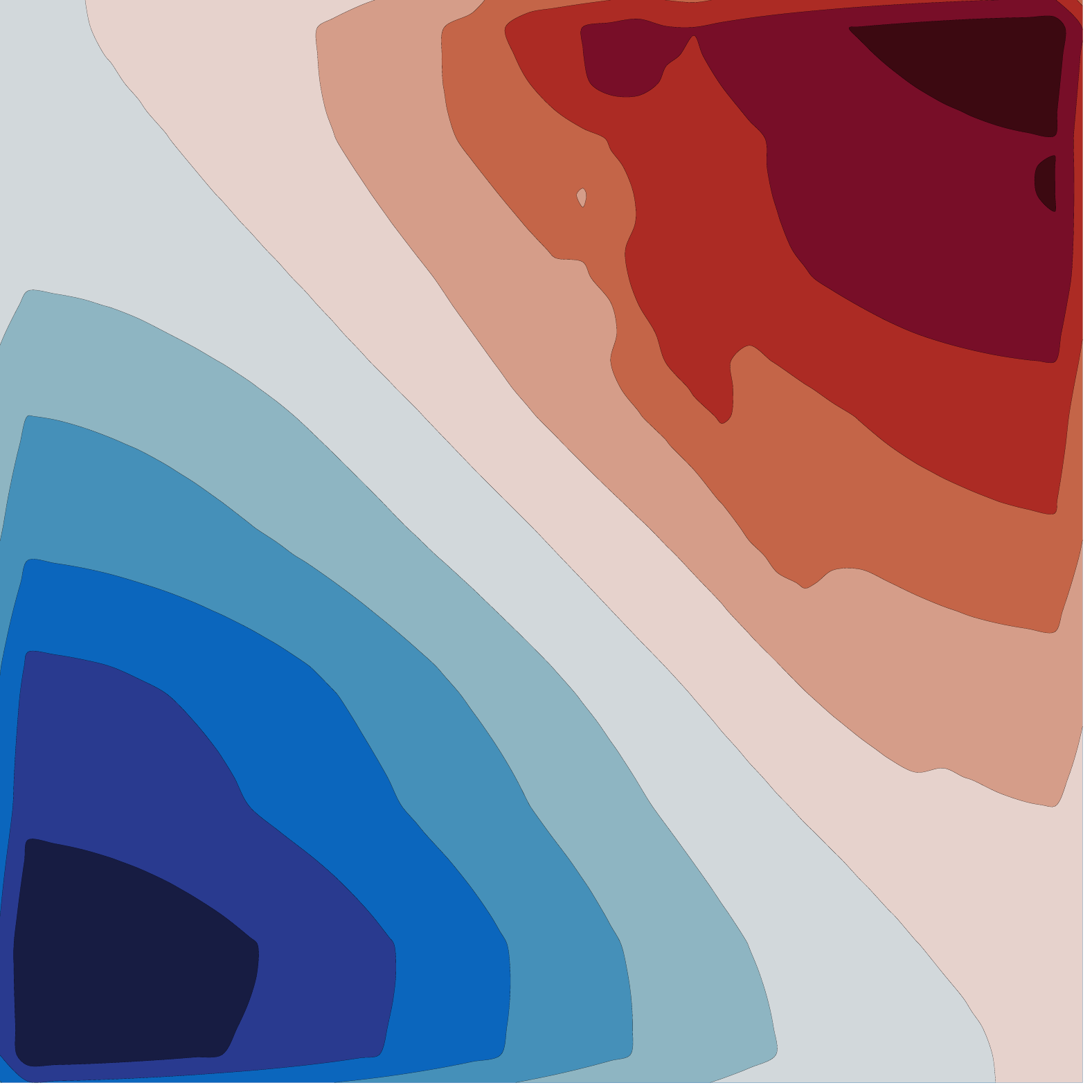}
  \caption{\label{fig:darcy-eigenvec-2}}
\end{subfigure}%
\end{center}
\caption{Extracted latent representation sequences reshaped to $n_c\times n_c$ from the encoder layers in the Galerkin Transformer using the same sample with Figure \ref{fig:ex2-darcy} and \ref{fig:ex3-darcy-inv} in evaluation;  (\subref{fig:ex3-l1-f21})--(\subref{fig:ex3-l1-f28}): two basis functions from the first encoder layer; (\subref{fig:ex3-l4-f26})--(\subref{fig:ex3-l4-f29}): two other basis functions from the fourth encoder layer, and we note that visually all basis functions in the fourth layers are smoother than those in the first; (\subref{fig:darcy-eigenvec-1})--(\subref{fig:darcy-eigenvec-2}): $-\nabla\cdot(a\nabla(\cdot))$'s first two eigenfunction approximations using the bilinear finite element on the $n_c\times n_c$ grid, i.e., the first two Fourier bases associated with this self-adjoint operator.}
\label{fig:darcy-latent}
\end{figure}

\subsection{Proof of Theorem \ref{theorem:cea-lemma}}
\label{sec:appendix-cea-proof}
\paragraph{Notations.} Before presenting the proof, to avoid confusion of the notions for various finite dimensional function spaces, the settings for different spaces are paraphrased from the condensed versions in Table \ref{table:notations}. The caveat is that for the same function in the latent approximation space, it has two vector representations: 
(i) nodal values at the grid points which can be used to form the columns of $Q, K, V$, this vector is in $\R^n$; 
(ii) the vector representation using degrees of freedom (coefficient functional) in an expansion using certain set of basis, this vector is in $\R^d$ or $\R^r$ ($r\leq d<n$). 

We also note that in the context of using the Galerkin-type attention 
(or other linear attentions) to approximate functions under a Hilbertian framework, in that the output of the attention-based map can represent a Petrov-Galerkin projection, $Q$ stands for values, $K$ for query, and $V$ for keys. In the proof of Theorem \ref{theorem:cea-lemma}, we shall refer the discrete approximation space generated by $Q$ as a ``value space'', and that of $V$ as a ``key space''. Meanwhile, $\Omega$ and $\Omega^*$ can be seen as spacial/temporal domain and frequency domain, respectively.

\begin{assumption}[assumptions and settings for Theorem \ref{theorem:cea-lemma}]
\label{assumption:cea-lemma}
The following notations and assumptions are used throughout the proof of the quasi-optimal approximation result in Theorem \ref{theorem:cea-lemma}:
\begin{enumerate}[align=left, leftmargin=10pt,
                  label=$\emph{(}$\subscript{D}{{\arabic*}}$\emph{)}$\; ]
  \item $(\mathcal{H}, \langle\cdot, \cdot\rangle_{\mathcal{H}})$ is a Hilbert space. For $f\in \mathcal{H}$, $f: \Omega\to \R$. $\mathcal{H}\hookrightarrow C^0(\Omega)$. $\Omega\subset \R^m$ is a bounded domain, discretized by $\{x_i\}_{i=1}^n$ with a mesh size $h$. 
  
  \item \label{asp:d-2} $\bb{Y}_h\subset \mathcal{H}$ is an approximation space associated with $\{x_i\}_{i=1}^n$, such that for any $y\in \bb{Y}_h$, $y(\cdot) = \sum_{i=1}^n y(x_i) \phi_{x_i}(\cdot)$ where $\{\phi_{x_i}(\cdot)\}_{i=1}^n$ form a set of nodal basis for $\bb{Y}_h$ in the sense that $\phi_{x_i}(x_j) = \delta_{ij}$, and the support of every nodal basis $\phi_{x_i}(\cdot)$ is of $O(h^m)$.
  
  \item $(\mathcal{V}, \langle\cdot, \cdot\rangle_{\mathcal{V}})$ is a latent Hilbert space. For $v\in \mathcal{V}$, $v: \Omega^*\to \R$. $\mathcal{V}\hookrightarrow C^0(\Omega^*)$. $\Omega^*\simeq\Omega$ and is discretized by $\{\xi_i\}_{i=1}^n$ with a mesh size $h$. 
  
  \item \label{asp:d-4} $\bb{W}_h\subset \mathcal{V}$ is an approximation space associated with $\{\xi_i\}_{i=1}^n$, i.e., for any $w\in \bb{W}_h$, $w(\cdot) = \sum_{i=1}^n w(\xi_i) \psi_{\xi_i}(\cdot)$ where $\{\psi_{\xi_i}(\cdot)\}_{i=1}^n$ form a set of nodal basis for $\bb{W}_h$ in the sense that $\psi_{\xi_i}(\xi_j) = \delta_{ij}$, and the support of every nodal basis $\psi_{\xi_i}(\cdot)$ is of $O(h^m)$.
  
  \item $\mathbf{y}\in \R^{n\times d}$ is the current input latent representation. $W^{Q}, W^K, W^V$ denote the current projection matrices. $n>d>m$ and $\op{rank}\mathbf{y} = d$. 
  
  \item \label{asp:d-6} $\bb{Q}_h\subset  \bb{Y}_h\subset \mathcal{Q}$ is the current value space from $Q$. $\mathcal{Q}$ is a subspace of $\mathcal{H}$ with the same topology. 
  $\bb{Q}_h$ is the spanned by basis functions whose degrees of freedom associated with $\{x_i\}_{i=1}^n$ form the columns of $Q:=\mathbf{y}W^Q\in \R^{n\times d}$, i.e., 
  $\bb{Q}_h = \operatorname{span}\{q_j(\cdot) \in \bb{Y}_h:  q_j(x_i) = Q_{ij}, \; 1\leq i\leq n, \;  1\leq j\leq d\}$.
  
  \item \label{asp:d-7} $\bb{V}_h\subset \bb{W}_h \subset\mathcal{V}$ is the current key space from $V$. 
  $\bb{V}_h$ is the spanned by basis functions whose degrees of freedom associated with $\{\xi_i\}_{i=1}^n$ form the columns of $V:=\mathbf{y}W^V\in \R^{n\times d}$, i.e., 
  $\bb{V}_h = \operatorname{span}\{v_j(\cdot) \in \bb{W}_h:  v_j(\xi_i) = V_{ij}, \; 1\leq i\leq n, \; 1\leq j\leq d\}$.
  
  \item \label{asp:d-8} $\bb{V}_h'$ is the dual space of $\bb{V}_h$ consisting of all bounded linear functionals defined on $\bb{V}_h$;   $\|g(\cdot)\|_{\bb{V}_h'}:=\sup_{v\in \bb{V}_h} 
  |g(v)|/\|v\|_{\mathcal{V}}\,$ for $g\in \bb{V}_h'$.
  
  \item $\dim \bb{Q}_h=r\leq \dim\bb{V}_h=d$, i.e., the key space is bigger than the value space.
  
  \item \label{asp:d-10} For $w\in \bb{V}_h$, $w(\cdot) = \sum_{j=1}^d \mu_{v_j}(w) v_j(\cdot)$ is the expansion in $\{v_j(\cdot)\}_{j=1}^d$, where $\mu_{v_j}(\cdot) \in \bb{V}_h'$ is the coefficient functional; in this case, $w$ can be equivalently identified by its vector representation $\bm{\mu}(w):= (\mu_{v_1}(w), \cdots, \mu_{v_d}(w))^{\top}\in \R^d$.

  \item \label{asp:d-11} For $p\in \bb{Q}_h$, $p(\cdot) = \sum_{j=1}^r \lambda_{q_j}(p) q_j(\cdot)$ is the expansion in $\{q_j(\cdot)\}_{j=1}^r$, where $\lambda_{q_j}(\cdot) \in \bb{Q}_h'$ is the coefficient functional; in this case, $p$ can be equivalently identified by its vector representation $\bm{\lambda}(p):= (\lambda_{q_1}(p), \cdots,  \lambda_{q_r}(q))^{\top}\in \R^r$.
  
  \item \label{asp:d-12} $\mathfrak{b}(\cdot, \cdot): \mathcal{V}\times \mathcal{Q}\to \R$ is a continuous bilinear form, i.e., 
  $|\mathfrak{b}(v, q) |\leq c_{0} \|v\|_{\mathcal{V}} \|q\|_{\mathcal{H}} $ for any $v\in \mathcal{V}$,  $q\in \mathcal{Q}$. For $(w, y)\in \bb{W}_h \times \bb{Y}_h \subset \mathcal{V}\times \mathcal{Q}$, $\mathfrak{b}(w, y):  = h^m\sum_{i=1}^n w(\xi_i)y(x_i) $.
  
  
  \item $g_{\theta}(\cdot): \R^{n\times d}\to \bb{Q}_h, \mathbf{y} \mapsto z$ is a learnable map that is the composition of the Galerkin-type attention operator \eqref{eq:attention-galerkin} with an updated set of $\{\widetilde{W}^{Q}, \widetilde{W}^{K}, \widetilde{W}^{V}\}$ and a pointwise universal approximator; $\theta$ denotes all the trainable parameters within.
\end{enumerate}
\end{assumption}

\begin{manualtheorem}{4.3}[C\'{e}a-type lemma, general version]
For any $f\in \mathcal{H}$, under Assumption \ref{assumption:cea-lemma}, for $f_h\in \bb{Q}_h$ being the best approximator of $f$ in $\|\cdot\|_{\mathcal{H}}$, we have:
\begin{equation}
\label{eq:approximation}
\min_{\theta} 
\|f- g_{\theta}(\mathbf{y})\|_{\mathcal{H}}
\leq c^{-1} \min_{q\in \bb{Q}_h} \max_{v\in \bb{V}_h}
\frac{|\mathfrak{b}(v, f_h - q)|}{\|v\|_{\mathcal{V}}}
+\|f-f_h\|_{\mathcal{H}}.
\end{equation}
\end{manualtheorem}

\begin{proof}
By triangle inequality, inserting the best approximation $f_h\in \bb{Q}_h$
\begin{equation}
\label{eq:approximation-split}
\|f- g_{\theta}(\mathbf{y})\|_{\mathcal{H}}\leq 
\|f_h - g_{\theta}(\mathbf{y}) \|_{\mathcal{H}} + \|f-f_h\|_{\mathcal{H}}.
\end{equation}
$f_h:= \operatorname{argmin}_{q\in \bb{Q}_h} \| f - f_h\|_{\mathcal{H}}$ describes the approximation capacity of the current value space $\bb{Q}_h$ and has nothing to do with $\theta$.

In the first part of the proof, we focus on bridging $\|f_h - g_{\theta}(\mathbf{y}) \|_{\mathcal{H}}$ in \eqref{eq:approximation-split} to the linear problem associated with seeking the Petrov-Galerkin projection of $f_h$. 
By Lemma \ref{lemma:lower-bound}, the continuous linear functional defined by 
$\mathfrak{b}(\cdot, q): \mathcal{V} \to \R, \,v\mapsto \mathfrak{b}(v, q)$ is bounded below on the current key space, i.e., there exists $c>0$ independent of $q$ or the discretization (mesh size $h$) such that,  
\begin{equation}
\label{eq:inf-sup}
\|\mathfrak{b}(\cdot, q)\|_{\bb{V}_h'}\geq c \|q \|_{\mathcal{H}}, \; \text{ for any fixed } q\in \bb{Q}_h.
\end{equation}
By the definition of $\|\cdot\|_{\bb{V}_h'}$ in \hyperref[asp:d-8]{$(D_{8})$}, we have
\[
\|f_h - g_{\theta}(\mathbf{y}) \|_{\mathcal{H}} \leq 
c^{-1} \sup_{v\in \bb{V}_h} 
\frac{\left|\mathfrak{b}(v, f_h - g_{\theta}(\mathbf{y})) \right|}{\norm{v}_{\mathcal{V}}}
= c^{-1} \max_{v\in \bb{V}_h} 
\frac{\left|\mathfrak{b}(v, f_h - g_{\theta}(\mathbf{y})) \right|}{\norm{v}_{\mathcal{V}}}.
\]
In the rest of the proof, the goal is to establish
\begin{equation}
\label{eq:approximation-capacity}
\min_\theta \max_{v\in \bb{V}_h} 
\frac{\left|\mathfrak{b}(v, f_h - g_{\theta}(\mathbf{y})) \right|}{\norm{v}_{\mathcal{V}}}
\leq 
 \min_{q\in \bb{Q}_h} \max_{v\in \bb{V}_h} 
\frac{\left|\mathfrak{b}(v, f_h - q) \right|}{\norm{v}_{\mathcal{V}}},
\end{equation}
i.e., the best approximation based on the attention operator, if exists, is on par with that offered by a Petrov-Galerkin projection.

By the continuity of $\mathfrak{b}(\cdot, \cdot)$ in \hyperref[asp:d-12]{$(D_{12})$}, given any fixed $q\in \bb{Q}_h$, $v\mapsto \mathfrak{b}(v, q)$ defines a bounded linear functional for any $v\in \bb{V}_h$. Moreover, since we use $\ell^2$-inner product to approximate $\langle\cdot, \cdot\rangle$ on $\bb{V}_h$, define 
\begin{equation}
\label{eq:norm-discrete}
  \langle u, v\rangle_h:= h^m\sum_{i=1}^n u(\xi_i)v(\xi_i) \approx \int_{\Omega^*} u(\xi)v(\xi)\dd \xi =:  \langle u, v\rangle_{\mathcal{V}}, \quad \text{for }u, v\in \bb{V}_h.
\end{equation}
Applying the Riesz representation theorem (e.g., see \cite[Theorem 4.6]{Ciarlet:2013Linear}), 
together with Lemma \ref{lemma:continuity}, implies that there exists a value-to-key linear map 
$\Phi$ for $f_h \in \bb{Q}_h$
\[
\Phi: \bb{Q}_h \to \bb{V}_h, \text{ such that } \mathfrak{b}(v, f_h) = \langle \Phi f_h, v\rangle_h \text{ for any } v\in \bb{V}_h.
\]
Thus, we have for the right-hand side of \eqref{eq:approximation-capacity} 
\begin{equation}
\label{eq:min-max}
  \min_{q\in \bb{Q}_h} \max_{v\in \bb{V}_h} 
\frac{\left|\mathfrak{b}(v, f_h - q) \right|}{\norm{v}_{\mathcal{V}}}
=
\min_{q\in \bb{Q}_h} \max_{v\in \bb{V}_h}
\frac{| \langle \Phi f_h, v\rangle_h - \mathfrak{b}(v,q )|}{\|v\|_{\mathcal{V}}}.
\end{equation}
Using \hyperref[asp:d-6]{$(D_{6})$},\hyperref[asp:d-7]{$(D_{7})$},\hyperref[asp:d-10]{$(D_{10})$},\hyperref[asp:d-11]{$(D_{11})$}, and Lemma \ref{lemma:mix}, the problem on the right-hand side of \eqref{eq:min-max} is equivalent to the block matrix form \eqref{eq:mix-matrix} as follows.
\begin{equation}
\label{eq:mix-3}
\begin{pmatrix}
  M & B^{\top} \\ B & 0
\end{pmatrix}
\begin{pmatrix}
 \bm{\mu}\\ \bm{\lambda}
\end{pmatrix}
= \begin{pmatrix} \bm{\zeta} \\ 0\end{pmatrix}.
\end{equation}
In the system above, 
$(\bm{\mu}, \bm{\lambda})\in \R^d \times \R^r$ are the vector representations of the critical point $(w, p)\in \bb{V}_h\times \bb{Q}_h$, if exists, under basis sets $\{v_j(\cdot)\}_{j=1}^d$ and $\{q_j(\cdot)\}_{j=1}^r$, respectively, and
\begin{equation}
\label{eq:mix-rhs}
\bm{\zeta}\in \R^d,\quad \text{where } \; (\bm{\zeta} )_j = \langle  \Phi f_h, v_j\rangle_h 
\approx \int_{\Omega^*} (\Phi f_h)(\xi) v_j(\xi)\dd \xi.
\end{equation}

In the next part of the proof, we shall seek the solution to \eqref{eq:mix-3}, which is the Petrov-Galerkin projection to the function of interest $f_h$. Since $\{v_j(\cdot)\}_{j=1}^d$ form a set of basis of $\bb{V}_h$, $M$ is invertible, and we can eliminate $\bm{\mu}$ by solving the first equation above and plugging it in the second to get
\begin{equation}
\label{eq:mix-solution}
B M^{-1} B^{\top} \bm{\lambda} = BM^{-1} \bm{\zeta}.
\end{equation}
Knowing that $M^{-1}$ is symmetric positive definite, to show 
\begin{equation}
\label{eq:mix-solution-1}
  \bm{\lambda} = (B M^{-1} B^{\top})^{-1} BM^{-1} \bm{\zeta},
\end{equation}
it suffices to show that $B$ is surjective (full row rank) if we ought to use Lemma \ref{lemma:matrix-invertible}. 
A simple argument by contradiction is as follows: suppose $B$ is not full row rank, then there exists a linear combination of the rows of $B$ being the 0 vector, i.e., there exists a set of nontrivial coefficients $\tilde{\bm{\lambda}}:= (\tilde{\lambda}_1, \dots, \tilde{\lambda}_{r})^{\top}$ such that
\[
\mathfrak{b} (v, \tilde{p}) = 0, \text{ for any }v\in \bb{V}_h  \text{ where } 0\not\equiv 
\tilde{p}(\cdot) := \sum_{l=1}^{r}\tilde{\lambda}_l q_l(\cdot).
\]
This is contradictory to the lower bound of $\mathfrak{b}(\cdot, \tilde{p})$ in \eqref{eq:inf-sup}:
\[
0< c\|\tilde{p}\|_{\mathcal{H}} \leq \|\mathfrak{b}(\cdot, \tilde{p})\|_{\bb{V}_h'} = 
\max_{v\in \bb{V}_h} \frac{|\mathfrak{b}(v, \tilde{p})|}{\|v\|_{\mathcal{V}} } = 0.
\]
Thus, \eqref{eq:mix-solution-1} holds and the critical point $p$ (or its vector representation) exists and is a local minimizer due to $M$ being positive definite. 

The last part of the proof is to show that this $p\in \bb{Q}_h$ is representable by the learnable map $g_{\theta}(\mathbf{y})$. To this end, we multiply a permutation matrix
$U\in \R^{d\times d}$ to $Q$, such that $QU$'s first $r$ columns $Q_0\in \R^{n\times r}$ form the nodal value vector $(q_j(x_1),\cdots, q_j(x_n))^{\top}$ for bases $\{q_j(\cdot)\}_{j=1}^r$ of $\bb{Q}_h$; see \hyperref[asp:d-2]{$(D_{2})$}. Then, multiplying the permuted basis matrix $QU$ with $(\bm{\lambda}\;\; 0)^{\top} \in \R^d$ yields the best approximator $p$'s vector representation at the grid points $\bm{p}:= (p(x_1),\cdots, p(x_n))^{\top}$
\begin{equation}
\label{eq:mix-solution-vec}
\bm{p} = \begin{pmatrix}
Q_0 & Q_1 
\end{pmatrix}\begin{pmatrix}
\bm{\lambda} \\ 0
\end{pmatrix} = QU\begin{pmatrix}
\bm{\lambda} \\ 0
\end{pmatrix}.
\end{equation}
Moreover, since $B_{ij} = \mathfrak{b}(v_j, q_i)$ with $\{q_i(\cdot)\}_{i=1}^r$ and $\{v_j(\cdot)\}_{j=1}^d$ being the sets of basis for $\bb{Q}_h$ and $\bb{V}_h$, it is straightforward to verify that $B = h^m Q_0^{\top} V$; see \hyperref[asp:d-12]{$(D_{12})$}. Here without loss of generality, we simply assume that the nodal value vector representation $(v_j(\xi_1),\cdots, v_j(\xi_n))^{\top} =: \bm{v}^j$ forms the $j$-th column of $V$ in a sequential order. Therefore, using \eqref{eq:mix-solution-1}, $(\bm{\lambda}\;\; 0)^{\top}\in \R^d$  can be written as the following block form:
\begin{equation}
\label{eq:mix-block}
\begin{pmatrix} \bm{\lambda} \\ 0 \end{pmatrix} =
\begin{pmatrix} (B M^{-1} B^{\top})^{-1} & \\   & 0 \end{pmatrix}
\begin{pmatrix} B \\ * \end{pmatrix} M^{-1} \bm{\zeta}
=
h^m\begin{pmatrix} (B M^{-1} B^{\top})^{-1} & \\   & 0 \end{pmatrix}
\begin{pmatrix} Q_0^{\top}\\Q_1^{\top} \end{pmatrix} V M^{-1} \bm{\zeta}.
\end{equation}
Consequently, the ideal updated $\{\widetilde{Q}, \widetilde{K}, \widetilde{V}\}$ that has capacity to obtain the current best approximator $p$, or its vector presentation $\bm{p}$ in \eqref{eq:mix-solution-vec}, are
\begin{equation}
\label{eq:attention-update}
  \begin{aligned}
&  \widetilde{Q}:= \mathbf{y} \widetilde{W}^Q \gets  \mathbf{y} W^Q U, 
\\
& \widetilde{K} := \mathbf{y} \widetilde{W}^K \gets  \mathbf{y} W^Q U \Lambda, 
\\
& \widetilde{V} := \mathbf{y} \widetilde{W}^V \gets  \mathbf{y} W^V M^{-1}, 
\\
 \text{and} & \;\; \bm{p} = h^m \widetilde{Q}(\widetilde{K}^T \widetilde{V}) \bm{\zeta},
  \end{aligned}
\end{equation}
where $\Lambda:= \operatorname{blkdiag}\left\{(B M^{-1} B^{\top})^{-1}, 0\right\}$ and is symmetric. 
This essentially implies that $g_\theta(\cdot)$ has capacity to learn this best approximator $p$ using the updated set of $\{\widetilde{Q}, \widetilde{K}, \widetilde{V}\}$ from \eqref{eq:attention-update}:
\[
g_\theta(\cdot):  \R^{n\times d}\to \bb{Q}_h, \mathbf{y} \mapsto p, \text{ such that } p(x_i) = \big(h^m\widetilde{Q}(\widetilde{K}^T \widetilde{V}) \bm{\zeta} \big)_i \;\text{ for } i=1,\cdots, n.
\]
Passing the inequality from minimizing in a bigger set yields the desired inequality \eqref{eq:approximation-capacity}, which, in turn, shows the approximation capacity indicated in the theorem.
\end{proof}

\paragraph{Dynamical basis update for a set of functions.}
\label{paragraph:dynamical-basis-update}
Despite the fact that Theorem \ref{theorem:cea-lemma} is for a single instance of $f\in \mathcal{H}$, it is not difficult to see that 
the form \eqref{eq:mix-block} easily extends to a latent representation in $\R^{n\times d}$ whose columns are $\{f_{j,h}(\cdot)\}_{j=1}^d$ sampled at the grid points $\{x_i\}_{i=1}^n$. Simply replacing $\bm{\zeta}$ in \eqref{eq:mix-rhs} by a matrix $W^Z := (\bm{\zeta}_1,  \cdots , \bm{\zeta}_d)\in \R^{d\times d}$, of which the $i$-th entry in the $j$-th column is $(\bm{\zeta}_j )_i = \langle  \Phi f_{j,h}, v_i\rangle_h$, we can verify that a pointwise FFN $g_\theta(\cdot)$ is fully capable of representing $W^Z$.

Using the argument above, we can further elaborate the dynamical basis update nature of the scaled dot-product attention \eqref{eq:attention-galerkin-projection} (see also the remarks in \ref{paragraph:approximation} and Appendix \hyperref[paragraph:dynamic-basis-update]{D.2.3}):
\begin{itemize}[leftmargin=1.5em]
\item Test the columns of $V$ (query) against the columns of $K$ (keys), and use the responses to seek a ``better'' set of basis using the columns of $Q$ (values).
\end{itemize}

To this end, we modify Assumption \ref{assumption:cea-lemma} slightly by appending (with a superscript $+$) or redefining (with a superscript $*$) certain entries as follows.

\begin{assumption}[assumptions and settings for Theorem \ref{theorem:basis-update}]
\label{assumption:basis-update}
Assumption \ref{assumption:cea-lemma} are assumed to hold with the following amendments to their respective numbered entries for the proof of Theorem \ref{theorem:basis-update}:
\begin{enumerate}[align=left, leftmargin=10pt]
\item[$(D_{6}^+)$] $\bb{K}_h\subset \bb{Y}_h\subset \mathcal{K}$ is the current key space from $K$ and is defined on $\Omega^*$. Define $\langle s, k\rangle_h:= h^m\sum_{i=1}^n s(\xi_i)k(\xi_i)$ as an inner product for $s,k\in \bb{K}_h$.

\item[$(D_{7}^*)$] $\bb{V}_h\subset \bb{W}_h\subset \mathcal{V} \subseteq \mathcal{Q}$ is the current query space from $V$ defined on $\Omega$.

\item[$(D_{12}^*)$] $\mathfrak{b}(\cdot, \cdot): \mathcal{K}\times \mathcal{Q}\to \R$ is continuous. $\mathfrak{b}(k, q):  = h^m\sum_{i=1}^n k(\xi_i)q(x_i)$ for $(k, q)\in \bb{Y}_h \times \bb{W}_h$.

\item[$(D_{12}^+)$] $\mathfrak{a}(\cdot, \cdot): \mathcal{V}\times \mathcal{K}\to \R$ is continuous. $\mathfrak{a}(w, k):  = h^m\sum_{i=1}^n w(x_i)k(\xi_i)$ for $(w, k)\in \bb{W}_h \times \bb{Y}_h$.


\end{enumerate}
\end{assumption}

After the introduction of $\mathfrak{a}(\cdot, \cdot)$, a new continuous functional can be defined as follows:
\begin{equation}
\label{eq:basis-update-functional}
\ell_{(w,q)}(\cdot): \mathcal{K}\to \R,\quad  k\mapsto \mathfrak{a}(w, k) - \mathfrak{b}(k, q).
\end{equation}

\begin{manualtheorem}{4.4}[layer-wise basis update of the scaled dot-product attention]
\label{theorem:basis-update}
Under Assumption \ref{assumption:basis-update}, and the columns of $V$ and $Q$ are formed by the DoFs of query and value functions, respectively, i.e., $V_{ij} = v_{j}(x_i)$ and $Q_{ij} = q_{j}(x_i)$, $1\leq i\leq n$, $1\leq j\leq d$. Then, there exists an updated set $\{\tilde{q}_l(\cdot)\}_{l=1}^d \subset \bb{Q}_h =\operatorname{span}\{q_l(\cdot)\}_{l=1}^d$ of value functions, such that $z_j(\cdot)\in \bb{Q}_h$ satisfying the basis update rule in \eqref{eq:attention-galerkin-projection}
\[
  z_j(\cdot) := \sum_{l=1}^d \mathfrak{a}(v_{j}, k_l) \, \tilde{q}_l(\cdot), \;
\text{ for } \,j=1,\cdots,d,
\]
is the minimizer to the following problem for every $j=1,\cdots, d$
\begin{equation}
\label{eq:basis-update-min-max}
\big\|\ell_{(v_{j}, z_j)}(\cdot)\big\|_{\bb{K}_h'} 
= \min_{q\in \bb{Q}_h} \big\| \mathfrak{a}(v_{j}, \cdot) - \mathfrak{b}(\cdot, q)\big\|_{\bb{K}_h'} 
= \min_{q\in \bb{Q}_h} \max_{k\in \bb{K}_h} 
\frac{\left|\mathfrak{a}(v_{j}, k) - \mathfrak{b}(k, q) \right|}{\norm{k}_{\mathcal{K}}}.
\end{equation}
\end{manualtheorem}

\begin{proof}
The proof repeats most parts from that of Theorem \ref{theorem:cea-lemma} in Appendix \ref{sec:appendix-cea-proof}. By Lemma \ref{lemma:mix}, $\mathfrak{b}(\cdot, q): k\mapsto \mathfrak{b}(k, q)$ is bounded below on the current key space $\bb{K}_h$, i.e., $\|\mathfrak{b}(\cdot, q)\|_{\bb{K}_h'}\geq c \|q \|_{\mathcal{H}}$ for any $q\in \mathbb{Q}_h$ fixed. Noting that the variational formulation corresponding to the min-max problem in \eqref{eq:basis-update-min-max} have the same form as the one in \eqref{eq:min-max-lemma}--\eqref{eq:mix} in Lemma \ref{lemma:mix} with the right-hand side switched from $\langle\cdot, \cdot\rangle_h$ to $\mathfrak{a}(\cdot, \cdot)$, therefore, we conclude that the following operator equation associated with the right-hand side of \eqref{eq:basis-update-min-max} has a unique solution $z_j\in \bb{Q}_h$:
\begin{equation}
\label{eq:basis-update-mix}
\left\{
\begin{aligned}
 \langle s, k \rangle_{h} + \mathfrak{b}(k, z_j) & = \mathfrak{a}(v_{j}, k), & 
\forall k\in \bb{K}_h,
\\
 \mathfrak{b}(s, q) &= 0, & \forall q \in \bb{Q}_h.
\end{aligned}
\right.
\end{equation}
It is straightforward to verify that $h^m(K^{\top}V)_{ij} = \mathfrak{a}(v_{j}, k_i)$ thus the matrix associated with $\mathfrak{a}(\cdot, \cdot)$ for all $\{v_j(\cdot)\}_{j=1}^d$ is $h^m (K^{\top}V)$.
Consequently, following the last part of the proof of Theorem \ref{theorem:cea-lemma} and writing in the form of scaled dot-product, we let $\mathbf{z}\in \R^{n\times d}$ be the latent representation with columns being $\{z_j(\cdot)\}_{j=1}^d$ at the grid points, i.e., $z_{j}(x_i) = \mathbf{z}_{ij}$, $1\leq i\leq n$, $1\leq j\leq d$, and carry over identical definitions for all matrices involved with the ones used in \eqref{eq:attention-update}, 
\[
\mathbf{z} = h^m Q U \widetilde{W}^Q (K^{\top} V),\; 
\text{ where } \widetilde{W}^Q:= (QU\Lambda)^{\top}(VM^{-1})\in \R^{d\times d}.
\]
The corollary is proved by setting $\{\tilde{q}_l(\cdot)\}_{l=1}^d$ as functions with DoFs being columns of $Q U \widetilde{W}^Q$.
\end{proof}

\paragraph{Heuristics on learning the latent representations.} 
In the dynamical basis update interpretation above, the scaled dot-product attention
is capable to bring the latent representation of query as close as possible to that of values in a functional distance, which is $\|\mathfrak{a}(v_{j}, \cdot) - \mathfrak{b}(\cdot, q)\big\|_{\bb{K}_h'}$ in \eqref{eq:basis-update-min-max}. Heuristically speaking in each encoder layer, an ideal operator learner could learn a latent subspace on which the input (query) and the target (values) are ``close'', and this closeness is measured by how they respond to a dynamically changing set of basis (keys).

\subsection{Interpretations, remarks, and possible generalizations}
\label{sec:appendix-cea-generalizations}
\paragraph{Inspirations from finite element methods.}
The first part of the proof for Theorem \ref{theorem:cea-lemma} follows closely to the finite element methods of the velocity-pressure formulation of a stationary 2D Stokesian flow \cite[Lemma 12.2.12]{Brenner.Scott:2008mathematical}, where the key is to establish the solvability of an inf--sup problem (min--max in our finite dimensional setting). 
The $n$-independent lower bound of $\mathfrak{b}(\cdot, q)$ plays a key role in the convergence theory of traditional finite element methods (see the remarks in Appendix \hyperref[paragraph:cea-history]{D.2.1}--\hyperref[paragraph:cea-convergence]{D.2.2}). In our interpretation of the scaled dot-product attention, this independence is essentially consequent on the diagonalizability of normalizations of the two matrices in this product (cf. the discussion in Appendix \hyperref[paragraph:cea-normalization]{D.4.4}; see also the proof in Lemma \ref{lemma:lower-bound}). The singular values of the attention matrix should be independent of the sequence length if the lower bound of $\mathfrak{b}(\cdot, q)$ is to be verified with a constant independent of $n$. The presence of softmax renders this verification impossible (see Remark \ref{remark:normalization}).

In traditional finite element methods \cite[Chapter 12]{Brenner.Scott:2008mathematical}, the dimension of the approximation subspaces are usually tied to the geometries (discretization, mesh), for example, $\dim \bb{V}_h = 2n = 2(\#\text{grid points})$, and $\dim \bb{Q}_h = \#\text{triangles} \simeq O(n)$. 
Inspired by \cite[Lemma 12.2.12]{Brenner.Scott:2008mathematical}, the introduction of the bilinear form $\mathfrak{b}(\cdot, \cdot): \mathcal{V}\times \mathcal{Q} \to \R$ here in Theorem \ref{theorem:cea-lemma} is to cater the possibility that the column spaces of $V$ and $Q$ may represent drastically different functions.

In our proof, we have shown that, using small subspaces (dimension $d$ and $r$) of these complete finite element approximation spaces (dimension $n\gg d\geq r$) suffices to yield the best approximator if this key-to-value map exists. This fits perfectly into our setting to use merely the column space of $Q$ to build the degrees of freedom (DoF) for the ``value'' functions: the DoF functional is simply a function integrating against the function which is discretized at the grid points, i.e., a dot-product in the sequence-length dimension represents an integral. Recently, we also become aware that the CV and NLP communities begin to exploit
this topological structure in the feature (channel) dimension by treating the vector in the same feature dimension as a function to exert operations upon, instead of mainly relying on position-wise
operations, see e.g., the work in \cite{Tolstikhin.Houlsby.ea:2021MLP,Lee-Thorp.Ainslie.ea:2021FNet} essentially acknowledges Assumption \ref{assumption:attn} implicitly.

It is also worth noting that, the subscript $h$ in the approximation subspaces is a choice of the exposition of the proof to facilitate the notion that ``dot product $\approx$ integral''. In computations, unlike finite element methods where the basis functions are locally supported, the learned basis functions are globally supported, dynamically updated, and not bound to a fixed resolution. As such, numerically the approximation subspaces are essentially mesh-free (e.g., see \eqref{eq:attention-galerkin-projection}, \eqref{eq:projection-galerkin}, and Figure \ref{fig:darcy-latent}), how this nature attributes to the capability of zero-shot super-resolution of kernelizable architectures (e.g., \cite{Li.Kovachki.ea:2021Fourier}) 
will be an interesting future study.

\paragraph{Difference with universal approximation theorems.} 
In the era of using a nonlinear approximator such as a deep neural network to approximate functions/operators, following the discussion in Section \ref{sec:appendix-cea-overview}, the ``consistency'' of an approximation class embodies itself as the Universal Approximation Theorem (UAT, e.g., see \cite{HornikStinchcombeWhite1989Multilayer,ChenChen1995Universal}, and \cite{LinJegelka2018ResNet,OkunoHadaShimodaira2018probabilistic,SiegelXu2020Approximation,YunBhojanapalliRawatEtAl2020Are} for modern expositions). 
A UAT roughly translates to ``a $d$-layer ($d\geq 2$) neural network with certain nonlinear activation can approximate any compactly supported continuous function'', with the approximation error 
depending on the width and the depth of a network.

In our work, Theorem \ref{theorem:cea-lemma} is in its spirit a ``stability'' result of the attention-based learner. Under this fixed general dot-product attention architecture, Theorem \ref{theorem:cea-lemma} tries to bridge the approximation capacity of a nonlinear operator (hard to quantify) to a linear one (easy to find). 
To further give a more refined bound in terms of the architectural hyperparameters (number of layers, number of learned basis, etc.), one shall invoke a UAT on top of the result in Theorem \ref{theorem:cea-lemma}. Theorem \ref{theorem:cea-lemma} states that the attention mechanism, in terms of architecture, allows practitioners to exploit simple non-learnable linear approximation results as well, for example the number of Fourier bases that one chooses to build the ``value'' space.

\paragraph{PDE-inspired attention architectures.}
The lower bound of $\mathfrak{b}(\cdot, q)$ gives a theoretical guideline on how to design a new dot-product attention between compatible subspaces.
For example, we can use the inter-position differences in certain direction (flux) as query to test against the key, which may provides more information on how to choose the best value. This aforementioned ``differential attention'' corresponds to $\langle q, \operatorname{div}\bm{v} \rangle$ in the velocity-pressure formulation of the Stokesian flow.

In the proof, we assume nothing such as $\bb{Q}_h\subset \bb{V}_h$ but only bridge them through a value-to-key map. The lower bound of this value-to-key map using the simple scaled dot-product is verified in Lemma \ref{lemma:lower-bound}.
In general, we have two guidelines: (1) the scaling shall be chosen such that the lower bound \eqref{eq:lower-bound} is independent of the sequence length; (2) the key space $\bb{V}_h$ (functions to test the responses against) is bigger than the value space $\bb{Q}_h$, which contains functions to approximate the target 
or to form a better set of latent basis. 
In practice, for example, $Q, V\in \R^{n\times r}, \R^{n\times d}$ with $d\geq r$ in \cite{Choromanski.Likhosherstov.ea:2021Rethinking}.
We remark that the proof is done for $\dim\bb{V}_h=d$, but in general it applies to $ \dim\bb{V}_h\leq d$ as the final block form matrices \eqref{eq:mix-block} can be easily adjusted.

From the proof, it is also straightforward to see that the columns of $Q, K, V$ merely act as the degrees of freedom representations associated with $x_i$ or $\xi_i$, thus not necessarily the pointwise value at $x_i$ or $\xi_i$. The essential requirement is that, there exists two sets of locally supported nodal basis functions in $\bb{Q}_h$ and $\bb{V}_h$ associated with $x_i$ or $\xi_i$ to build the globally supported learned bases. These nodal basis functions are assumed to be supported in an $O(h^m)$-neighborhood of $x_i$ or $\xi_i$, which implies the sparse connectivity of the grids in the discretization. 
Here the degrees of freedom being the pointwise values is assumed merely for simplicity, thus imposing $\mathcal{H}\hookrightarrow C^0(\Omega)$ and $\mathcal{V}\hookrightarrow C^0(\Omega^*)$ to ensure the existence of the pointwise values is just a technicality and not essential. As a result, this broadens the possibility of a more structure-preserving feature map. For example, the bilinear form $\mathfrak{b}(\cdot, \cdot)$ can incorporate information from the higher derivatives, DoFs for splines, etc.; another example is that a single entry in a column of $Q,K,V$ may stand for an edge DoF vector representation in a graph, and the edge--edge interaction
is extensively studied for the simulation of Maxwell's equations \cite{Monk.others:2003Finite,CaoChenGuo2021Virtual,Arndt.Bangerth.ea:2020dealII} on manifolds \cite{Whitney:2015Geometric}.

\paragraph{Galerkin-type layer normalization scheme.}
\label{paragraph:cea-normalization}
In the proposed Galerkin-type attention \eqref{eq:attention-galerkin}, the layer normalization is pre-dot-product but post-projection; cf. the pre-LN scheme in \cite{Xiong.Yang.ea:2020Layer} is pre-dot-product and pre-projection. 
From the proof of Theorem \ref{theorem:cea-lemma}, it is natural to impose such a normalization from a mathematical perspective. As in \eqref{eq:attention-update}, the updated $\widetilde{K}$ and $\widetilde{V}$ have $(BM^{-1}B^{\top})^{-1}$ and $M^{-1}$ as their normalizations, respectively. While $M$ is the Gram matrix for $V$, the inverse of $M$'s Schur complement can be understood as the inverse of the Riesz map from $Q$ to $V$. Since a given layer normalization module is shared by every position of a latent representation upon which it is exerted, the weight acting as the variance (in $\R^d$) in the layer normalization acts a cheap learnable diagonal matrix approximation to $M^{-1}\in \R^{d\times d}$. This offers a reasonable approximation if one assumes the self-similarity of the bases $\|v_j\|^2_{\mathcal{V}}$ outweighs the inter-basis inner product $\langle v_i, v_j \rangle$ for $i\neq j$.

\subsection{Auxiliary results}
\label{sec:appendix-cea-aux}

\begin{lemma}
  \label{lemma:continuity}
Consider $\Omega^*\subset \R^m$ a bounded domain, we assume $(\mathcal{V}, \langle\cdot, \cdot\rangle_{\mathcal{V}})\hookrightarrow C^0(\Omega^*)$, where $\langle u, v\rangle_{\mathcal{V}}:= \int_{\Omega^*} u(\xi)v(\xi) \dd\xi$. $\Omega$ is discretized by $\{\xi_i\}_{i=1}^n$. $\bb{Y}_h\subset \mathcal{V}$ is an approximation space defined on $\{\xi_i\}_{i=1}^n$ and $\bb{Y}_h \simeq \R^n$. $\bb{Y}_h= \operatorname{span} \{\phi_{\xi_1},\cdots, \phi_{\xi_n}\}$ such that the degree of freedom for the $i$-th nodal basis is defined as $\chi_{\xi_i}(v):= v(\xi_i)$. $\bb{V}_h = \operatorname{span} \{v_{1}\cdots, v_d\}$ for $d<n$ with $d$ linearly independent $v_j(\cdot) \in \bb{Y}_h$. Then,
$\langle\cdot, \cdot\rangle_h: \mathcal{V}\times \mathcal{V} \to \R$ defines a continuous inner product in $\bb{V}_h$, where
\begin{equation}
\label{eq:inner-product-discrete}
\langle u, v\rangle_h:= h^m\sum_{i=1}^n u(\xi_i)v(\xi_i), \quad \text{for }u, v\in \bb{V}_h.
\end{equation}
\end{lemma}
\begin{proof}
First we show that $\langle v, v\rangle_h = 0$ implies that $v\equiv 0$. By \eqref{eq:inner-product-discrete}, obviously this $v$ satisfies $v(\xi_i)=0$ for $i=1,\cdots, n$. Now expanding $v$ in $\{\phi_{\xi_i}\}_{i=1}^n$ we have
\[
v(\cdot) = \sum_{i=1}^n \chi_{\xi_i}(v) \phi_{\xi_i}(\cdot) = \sum_{i=1}^n v(\xi_i) \phi_{\xi_i}(\cdot) \equiv 0.
\]
Thus, $\|\cdot\|_{\bb{V}_h}^2 := \langle \cdot, \cdot\rangle_h$ defines a norm, and it is equivalent to $\|\cdot\|_{L^2(\Omega^*)}$ restricted on $\bb{V}_h$ due to being finite dimensional. The desired result follows from applying the Cauchy-Schwarz inequality.

\end{proof}

\begin{lemma}
\label{lemma:matrix-invertible}
If $r\leq d$, $B\in \R^{r\times d}$ has full row rank, $G\in \R^{d\times d}$ is symmetric positive definite, then $BGB^{\top} \in \R^{r\times r}$ is invertible.
\end{lemma}
\begin{proof}
With $G$ being symmetric positive definite, upon applying a Cholesky factorization $G = CC^{\top}$, we have $BGB^{\top} = (BC) (BC)^{\top}$. Now it is straightforward to see that $\rank (BC) = r$ as $\rank (BC) \leq \min \{r,d\}$ and $\rank (BC) \geq r$ by the Sylvester's inequality. Applying the same argument on the product of $BC$ with $(BC)^{\top}$ yields the desired result of $\rank (BGB^{\top})  = r$.
\end{proof}

\begin{lemma}
\label{lemma:mix}
Under the same setting with Theorem \ref{theorem:cea-lemma} (Assumption \ref{assumption:cea-lemma}), given a $u\in \bb{V}_h$, looking for a saddle point of 
\begin{equation}
\label{eq:min-max-lemma}
\min_{q\in \bb{Q}_h} \max_{v\in \bb{V}_h}
\frac{|\langle u, v\rangle_h - \mathfrak{b}(v, q)|}{\|v\|_{\mathcal{V}}}
\end{equation}
is equivalent to solving the following operator equation system: 
find $p\in \bb{Q}_h, w\in \bb{V}_h$
\begin{equation}
\label{eq:mix}
\left\{
\begin{aligned}
 \langle w, v\rangle_{h} + \mathfrak{b}(v, p) & = \langle u, v\rangle_h, & \forall v\in \bb{V}_h,
\\
 \mathfrak{b}(w, q) &= 0, & \forall q \in \bb{Q}_h.
\end{aligned}
\right.
\end{equation}
It is further equivalent to solve the following linear system if 
$\{v_j(\cdot)\}_{j=1}^d$ and $\{q_j(\cdot)\}_{j=1}^r$ form sets of basis for $\bb{V}_h$ and $\bb{Q}_h$, respectively:
\begin{equation}
\label{eq:mix-matrix}
\begin{pmatrix}
  M & B^{\top} \\ B & 0
\end{pmatrix}
\begin{pmatrix}
 \bm{\mu}\\ \bm{\lambda}
\end{pmatrix}
= \begin{pmatrix} \bm{\zeta} \\ 0\end{pmatrix},
\end{equation}
where $M \in \R^{d\times d}$  with $M_{ij} = \langle v_j, v_i\rangle_h$. $B\in \R^{r\times d}$ with $B_{ij} = \mathfrak{b}(v_j, q_i)$.  $\bm{\zeta}\in \R^d$ with $(\bm{\zeta} )_j = \langle u, v_j\rangle_h$. For $w(\cdot) \in \bb{V}_h$, its vector representation is $\R^d\ni \bm{\mu} := \bm{\mu}(w) = (\mu_{v_1}(w), \dots, \mu_{v_d}(w))^{\top}$ for 
$w(\cdot) = \sum_{j=1}^d \mu_{v_j}(w) v_j(\cdot)$. Similar notion applies to 
$\bm{\lambda}:=\bm{\lambda}(p) = (\lambda_{q_1}(p), \cdots, \lambda_{q_{r}}(p))^{\top}\in \R^{r}$ being the vector representation for $p(\cdot) = \sum_{j=1}^{r} \lambda_{q_j}(p) q_j(\cdot)$ where $ p(\cdot) \in \bb{Q}_h$.

\end{lemma}

\begin{proof}
This lemma is a finite dimensional rephrasing of the commonly-known variational formulation of a saddle point problem in infinite dimensional Hilbert spaces \cite[Chapter I $\S$ 4.1]{Girault.Raviart:1979Finite}, for completeness we include the proof here for the convenience of readers.

Define $\eta(\cdot): \bb{V}_h\to \R$ such that 
$\eta(v) := \langle u, v\rangle_h - \mathfrak{b}(v, q)$, using Lemma \ref{lemma:continuity} and the assumptions \hyperref[asp:d-11]{$(D_{11})$} \hyperref[asp:d-12]{$(D_{12})$} on $\mathfrak{b}(\cdot,\cdot)$ in Assumption \ref{assumption:cea-lemma}, 
then clearly
$\eta\in \bb{V}_h'$ for $(\bb{V}_h, \langle\cdot, \cdot\rangle_{h})$.  By Riesz representation theorem, there exists an isomorphism $R: \bb{V}_h \to \bb{V}_h'$ such that $w := R^{-1}(\eta)\in \bb{V}_h$
and
\begin{equation}
\label{eq:mix-1}  
\eta(v) = \langle w, v\rangle_h =  \langle u, v\rangle_h - \mathfrak{b}(v, q).
\end{equation}
Then, \eqref{eq:min-max-lemma} is equivalent to find the minimizer for the following problem, define $\|\cdot\|_{\bb{V}_h}^2 := \langle\cdot, \cdot\rangle_h$:
\[
\min_{q\in \bb{Q}_h} \|\eta(\cdot)\|_{\bb{V}_h'}^2
= \min_{q\in \bb{Q}_h} \|w\|_{\bb{V}_h}^2
= 
\min_{q\in \bb{Q}_h} \|R^{-1}(\langle u, \cdot\rangle_h - \mathfrak{b}(\cdot, q) )\|_{\bb{V}_h}^2
=:\min_{q\in \bb{Q}_h} J(q).
\]
Taking the Gateaux derivative $\lim_{\tau\to 0} \dd J(p+\tau q)/ \dd \tau$ in order to find the critical point(s) $p\in \bb{Q}_h$, we have for any perturbation $q\in \bb{Q}_h$ such that $p+\tau q\in \bb{Q}_h$
\[
0 =\lim_{\tau\to 0}\frac{d}{d\tau} 
\left\langle
R^{-1}(\langle u, \cdot\rangle_h - \mathfrak{b}(\cdot, p+\tau q),
R^{-1}(\langle u, \cdot\rangle_h - \mathfrak{b}(\cdot, p+\tau q) )
\right\rangle_h,
\]
and applying $R^{-1}$ on $\mathfrak{b}(\cdot, p)\in \bb{V}_h'$, it reads for any $q\in \bb{Q}_h$
\begin{equation}
\label{eq:mix-2}
\left\langle
R^{-1}\big(\langle u, \cdot\rangle_h - \mathfrak{b}(\cdot, q) \big),
R^{-1}(\mathfrak{b}(\cdot, q) ) \right\rangle_h
= \left\langle w, R^{-1}(\mathfrak{b}(\cdot, q) )\right\rangle_h 
= \mathfrak{b}(w, q) = 0.
\end{equation}
Thus, \eqref{eq:mix-1} and \eqref{eq:mix-2} form the desired system \eqref{eq:mix}. Lastly, 
applying an expansion of $w$ and $p$ in $\{v_i(\cdot)\}_{i=1}^d$ and $\{q_i(\cdot)\}_{i=1}^r$ with degrees of freedom 
$ \mu_{v_j}(\cdot)$ and $\lambda_{q_j}(\cdot)$ respectively, and choosing the test functions to be each $v=v_j$ for $j=1,\cdots, d$ and $q=q_j$ for $j=1,\cdots, r$, yield the system \eqref{eq:mix-matrix} with $i$ and $j$ representing the row index and column index, respectively. 
\end{proof}

\begin{lemma}[verification of the lower bound of $\mathfrak{b}(\cdot,p)$]
\label{lemma:lower-bound}
Under the setting of Assumption \ref{assumption:cea-lemma}, for any given $p\in \bb{Q}_h$, we have
\begin{equation}
\label{eq:lower-bound}
c\|p\|_{\mathcal{H}} \leq 
\max_{w\in \bb{V}_h} \frac{|\mathfrak{b}(w, p)|}{\|w\|_{\mathcal{V}} },
\end{equation}
where the constant $c = c_V^{-1} c_{Q}^{-1} \min_{j=1,\dots,r}|\sigma_j| $. $\{\sigma_j\}_{j=1}^r$ are the singular values of the matrix $B\in \R^{r\times d}$ with $B_{ij} = \mathfrak{b}(v_j, q_i)$ under the sets of basis $\{v_j(\cdot)\}_{j=1}^d$ and $\{q_j(\cdot)\}_{j=1}^r$. 
$c_V$, $c_Q$ are the norm equivalence constants between functions $w(\cdot)\in \bb{V}_h$, $p(\cdot)\in \bb{Q}_h$ and their vector representation using the DoF functionals $\bm{\mu}(w)$ and $\bm{\lambda}(p)$ defined in Lemma \ref{lemma:mix} under the sets of basis $\{v_j(\cdot)\}_{j=1}^d$ and $\{q_j(\cdot)\}_{j=1}^r$: for any $w\in \bb{V}_h$, and any $p\in \bb{Q}_h$, it is assumed that the following norm equivalence holds
\begin{equation}
\label{eq:norm-equiv}
\|w\|_{\mathcal{V}}\leq c_V \|\bm{\mu}(w)\| \; 
\text{ and }\; \|p\|_{\mathcal{H}}\leq c_Q \|\bm{\lambda}(p)\|.
\end{equation}
\end{lemma}
\begin{proof}
The proof is straightforward by exploiting the singular value decomposition (SVD) to the matrix representation of the bilinear form $\mathfrak{b}(\cdot,\cdot)$. When $r\leq d$, $\rank(B)=r$.
Consider an SVD of $B$: for $D = (\operatorname{diag}\{\sigma_1, \cdots, \sigma_r\},\;  0)$ with $\sigma_j\neq 0$, $U_Q, U_V$ being orthonormal,
\[
B = U_Q D U_V^{\top}, \; 
\text{ where } U_Q\in \R^{r\times r}, 
D\in \R^{r\times d}, \text{ and } U_V\in \R^{d\times d}.
\]
Hence, $\mathfrak{b}(w, p)$ can be equivalent written as the following for $\bm{\mu}:= \bm{\mu}(w)$ 
and $\bm{\lambda}:= \bm{\lambda}(q)$
\[
\mathfrak{b}(w, p) = \mathfrak{b}\left( 
\sum_{j=1}^d \mu_{v_j}(w) v_j(\cdot),
\sum_{j=1}^{r} \lambda_{q_j}(p) q_j(\cdot) 
\right) 
= \bm{\lambda}^{\top} B \bm{\mu} = (U_Q^{\top} \bm{\lambda})^{\top}
D (U_V^{\top} \bm{\mu}).
\]
By the norm equivalence \eqref{eq:norm-equiv} and above,
\[
\max_{w\in \bb{V}_h} \frac{|\mathfrak{b}(w, p)|}{\|w\|_{\mathcal{V}} }
\geq c_V^{-1} \max_{\bm{\mu}\in \R^d} \frac{\left|(U_Q^{\top} \bm{\lambda})^{\top}
D (U_V^{\top} \bm{\mu})\right|}{\|\bm{\mu}\|} 
= c_V^{-1} \max_{\bm{\mu}\in \R^d} \frac{\left|(U_Q^{\top} \bm{\lambda})^{\top}
D (U_V^{\top} \bm{\mu})\right|}{\left\|U_V^{\top}\bm{\mu}\right\|}.
\]
Since $U_V^{\top}$ is surjective, we choose a specific $U_V^{\top}\bm{\mu}\in \R^d$ to pass the lower bound: let the first $r$ entries of $U_V^{\top}\bm{\mu}$ be $U_Q^{\top}\bm{\lambda}$, we have
\begin{equation}
  \begin{aligned}
& c_V^{-1} \max_{\bm{\mu}\in \R^d} \frac{\left|(U_Q^{\top} \bm{\lambda})^{\top}
D (U_V^{\top} \bm{\mu})\right|}{\left\|U_V^{\top}\bm{\mu}\right\|}
\geq c_V^{-1} \frac{\left|(U_Q^{\top} \bm{\lambda})^{\top}
\operatorname{diag}\{\sigma_1, \cdots, \sigma_r\} 
(U_Q^{\top} \bm{\lambda})\right|}{\big\|U_Q^{\top} \bm{\lambda}\big\|}
\\
\geq & \; c_V^{-1} \min_{j=1,\cdots,r} |\sigma_j| \big\|U_Q^{\top} \bm{\lambda}\big\|
= c_V^{-1} \min_{j=1,\cdots,r} |\sigma_j| \;\| \bm{\lambda}\|
\geq c_V^{-1}c_Q^{-1} \min_{j=1,\cdots,r} |\sigma_j|\;  \|p\|_{\mathcal{H}}.
\end{aligned}
\end{equation}
\end{proof}

\begin{remark}[Constants $c_V, c_Q$ and the potential impact of softmax thereon]
\label{remark:normalization}
The norm equivalence constants bridging the integral-based $\|\cdot\|_{\mathcal{H}}$ and the $\ell^2$-norm $\|\cdot\|$ depend on the topology of the approximation spaces. 

If the basis functions are globally supported, such as the Fourier-type basis from the eigenfunctions of the self-adjoint operator, or the learned bases shown in Figure \ref{fig:darcy-latent}, orthonormal, and defined on the same discretization on the spacial domain, then it is easy to see that $c_V$ and $c_Q$ are approximately $1$ due to the Parseval identity, minus the caveat of approximating an integral on a discrete grid.

Even if the basis functions $\{v_j(\cdot)\}$ and $\{q_j(\cdot)\}$ for $\bb{V}_h$ and $\bb{Q}_h$ are locally supported, such as the nodal basis in \hyperref[asp:d-2]{$(D_{2})$} and \hyperref[asp:d-4]{$(D_{4})$}, the $h^m$-weight in \eqref{eq:inner-product-discrete} will make $c_V$ and $c_Q$ be of $O(h^{m/2})$, or the inverse square root of the sequence length $1/\sqrt{n} = O(h^{m/2})$, see e.g., \cite[Section 11]{Xu.Zikatanov:2017Algebraic}. 
Nevertheless, the final bound \eqref{eq:lower-bound} will be sequence length-independent because now the minimum singular value of $B$ will scale as $O(h^m)$ the same with a mass matrix in the finite element method (see e.g., \cite[Section 4.4.2]{Ern.Guermond:2004Theory}). 
Consequently, these two constants depend on the number of explicit connections (not learned) that a single position has in the discretization. In our examples, the Euclidean coordinate positional encodings yield a sparse connection (tri-diagonal in 1D, 5-point stencil in 2D) thus independent of $n$.

If a softmax normalization $\mathcal{S}(\cdot)$, such as the one in \cite{Shen.Zhang.ea:2021Efficient}, is applied on the key matrix in the sequence length dimension, the norm equivalence constant $c_V$ is not $n$-independent anymore. To illustrate this, without loss of generality, let $\mathcal{V}\subset \mathcal{H}=L^2(\Omega)$, $\Omega=\Omega^* = (0,1)$. By the definition of softmax, it is straightforward to verify that the test functions essentially change to $\widetilde{w}:=\mathcal{S}(w)\approx \left(n\int_{\Omega} e^{w} dx\right)^{-1} e^{w} $. Following the argument of Lemma \ref{lemma:lower-bound}, the lower bound can be proved for $\widetilde{w}$, i.e., $\max_{w\in \bb{V}_h} |\mathfrak{b}(w, p)|/ \|\widetilde{w}\|_{\mathcal{H}}  \geq c\|p\|_{\mathcal{H}}$. However, by an exponential Sobolev inequality \cite[Theorem 7.21]{Gilbarg.Trudinger2001Elliptic} if we further assume that $w$ has a weak derivative with a bounded $L^1(\Omega)$-norm, $\|w\|_{\mathcal{H}}\leq c_{\Omega}n \|\widetilde{w}\|_{\mathcal{H}}$, thus passing the inequality to try to obtain an estimate like \eqref{eq:lower-bound} yields an inevitable constant related to $n$.

\end{remark}


{
\small
\begin{thebibliography}{100}
\providecommand{\natexlab}[1]{#1}
\providecommand{\url}[1]{\texttt{#1}}
\expandafter\ifx\csname urlstyle\endcsname\relax
  \providecommand{\doi}[1]{doi: #1}\else
  \providecommand{\doi}{doi: \begingroup \urlstyle{rm}\Url}\fi

\bibitem[Adler and {\"O}ktem(2017)]{Adler.Oektem:2017Solving}
Jonas Adler and Ozan {\"O}ktem.
\newblock Solving ill-posed inverse problems using iterative deep neural
  networks.
\newblock \emph{Inverse Problems}, 33\penalty0 (12):\penalty0 124007, 2017.

\bibitem[Al-Rfou et~al.(2019)Al-Rfou, Choe, Constant, Guo, and
  Jones]{AlRfouChoeConstantGuoEtAl2019Character}
Rami Al-Rfou, Dokook Choe, Noah Constant, Mandy Guo, and Llion Jones.
\newblock Character-level language modeling with deeper self-attention.
\newblock \emph{Proceedings of the AAAI Conference on Artificial Intelligence},
  33\penalty0 (01):\penalty0 3159--3166, Jul. 2019.

\bibitem[Alessandrini(1986)]{Alessandrini:1986identification}
Giovanni Alessandrini.
\newblock An identification problem for an elliptic equation in two variables.
\newblock \emph{Annali di matematica pura ed applicata}, 145\penalty0
  (1):\penalty0 265--295, 1986.

\bibitem[Alet et~al.(2019)Alet, Jeewajee, Villalonga, Rodriguez, Lozano-Perez,
  and Kaelbling]{Alet.Jeewajee.ea:2019Graph}
Ferran Alet, Adarsh~Keshav Jeewajee, Maria~Bauza Villalonga, Alberto Rodriguez,
  Tomas Lozano-Perez, and Leslie Kaelbling.
\newblock Graph element networks: adaptive, structured computation and memory.
\newblock In \emph{International Conference on Machine Learning}, pages
  212--222. PMLR, 2019.

\bibitem[Anandkumar et~al.(2020)Anandkumar, Azizzadenesheli, Bhattacharya,
  Kovachki, Li, Liu, and Stuart]{Li.Kovachki.ea:2020Neural}
Anima Anandkumar, Kamyar Azizzadenesheli, Kaushik Bhattacharya, Nikola
  Kovachki, Zongyi Li, Burigede Liu, and Andrew Stuart.
\newblock Neural operator: Graph kernel network for partial differential
  equations.
\newblock In \emph{ICLR 2020 Workshop on Integration of Deep Neural Models and
  Differential Equations}, 2020.

\bibitem[Arndt et~al.(2020)Arndt, Bangerth, Blais, Clevenger, Fehling, Grayver,
  Heister, Heltai, Kronbichler, Maier, Munch, Pelteret, Rastak, Thomas,
  Turcksin, Wang, and Wells]{Arndt.Bangerth.ea:2020dealII}
Daniel Arndt, Wolfgang Bangerth, Bruno Blais, Thomas~C. Clevenger, Marc
  Fehling, Alexander~V. Grayver, Timo Heister, Luca Heltai, Martin Kronbichler,
  Matthias Maier, Peter Munch, Jean-Paul Pelteret, Reza Rastak, Ignacio Thomas,
  Bruno Turcksin, Zhuoran Wang, and David Wells.
\newblock The \texttt{deal.II} library, version 9.2.
\newblock \emph{Journal of Numerical Mathematics}, 28\penalty0 (3):\penalty0
  131--146, 2020.

\bibitem[Bahdanau et~al.(2015)Bahdanau, Cho, and
  Bengio]{Bahdanau.Cho.ea:2016Neural}
Dzmitry Bahdanau, Kyung~Hyun Cho, and Yoshua Bengio.
\newblock Neural machine translation by jointly learning to align and
  translate.
\newblock In \emph{3rd International Conference on Learning Representations,
  ICLR 2015}, 2015.

\bibitem[Berrut and Trummer(1987)]{Berrut.Trummer:1987Equivalence}
Jean-Paul Berrut and Manfred~R Trummer.
\newblock {Equivalence of Nystr{\"o}m’s method and Fourier methods for the
  numerical solution of Fredholm integral equations}.
\newblock \emph{Mathematics of computation}, 48\penalty0 (178):\penalty0
  617--623, 1987.

\bibitem[Bhatnagar et~al.(2019)Bhatnagar, Afshar, Pan, Duraisamy, and
  Kaushik]{Bhatnagar.Afshar.ea:2019Prediction}
Saakaar Bhatnagar, Yaser Afshar, Shaowu Pan, Karthik Duraisamy, and Shailendra
  Kaushik.
\newblock Prediction of aerodynamic flow fields using convolutional neural
  networks.
\newblock \emph{Computational Mechanics}, 64\penalty0 (2):\penalty0 525--545,
  2019.

\bibitem[Bhattacharya et~al.(2020)Bhattacharya, Hosseini, Kovachki, and
  Stuart]{Bhattacharya.Hosseini.ea:2020Model}
Kaushik Bhattacharya, Bamdad Hosseini, Nikola~B Kovachki, and Andrew~M Stuart.
\newblock {Model reduction and neural networks for parametric PDEs}.
\newblock \emph{arXiv preprint arXiv:2005.03180}, 2020.

\bibitem[Blanc and Rendle(2018)]{BlancRendle2018adaptive}
Guy Blanc and Steffen Rendle.
\newblock Adaptive sampled softmax with kernel based sampling.
\newblock In \emph{International Conference on Machine Learning}, pages
  590--599. PMLR, 2018.

\bibitem[Bracewell and Bracewell(1986)]{Bracewell.Bracewell:1986Fourier}
Ronald~Newbold Bracewell and Ronald~N Bracewell.
\newblock \emph{The Fourier transform and its applications}, volume 31999.
\newblock McGraw-Hill New York, 1986.

\bibitem[Brenner and Scott(2008)]{Brenner.Scott:2008mathematical}
Susanne~C Brenner and Ridgway Scott.
\newblock \emph{The mathematical theory of finite element methods}, volume~15.
\newblock Springer, 2008.

\bibitem[Brent(1976)]{Brent1976Multiple}
Richard~P Brent.
\newblock Multiple-precision zero-finding methods and the complexity of
  elementary function evaluation.
\newblock In \emph{Analytic computational complexity}, pages 151--176.
  Elsevier, 1976.

\bibitem[Cao et~al.(to appear, 2021)Cao, Chen, and Guo]{CaoChenGuo2021Virtual}
Shuhao Cao, Long Chen, and Ruchi Guo.
\newblock A virtual finite element method for two dimensional {M}axwell
  interface problems with a background unfitted mesh.
\newblock \emph{Mathematical Models and Methods in Applied Sciences}, to
  appear, 2021.

\bibitem[Chan and Tai(2003)]{Chan.Tai:2003Identification}
Tony~F Chan and Xue-Cheng Tai.
\newblock Identification of discontinuous coefficients in elliptic problems
  using total variation regularization.
\newblock \emph{SIAM Journal on Scientific Computing}, 25\penalty0
  (3):\penalty0 881--904, 2003.

\bibitem[Chen(2008)]{Chen:2008iFEM}
Long Chen.
\newblock {$i$FEM}: an integrated finite element methods package in {MATLAB}.
\newblock Technical report, 2008.

\bibitem[Chen and Chen(1995)]{ChenChen1995Universal}
Tianping Chen and Hong Chen.
\newblock Universal approximation to nonlinear operators by neural networks
  with arbitrary activation functions and its application to dynamical systems.
\newblock \emph{IEEE Transactions on Neural Networks}, 6\penalty0 (4):\penalty0
  911--917, 1995.

\bibitem[Choromanski et~al.(2021)Choromanski, Likhosherstov, Dohan, Song, Gane,
  Sarlos, Hawkins, Davis, Mohiuddin, Kaiser, Belanger, Colwell, and
  Weller]{Choromanski.Likhosherstov.ea:2021Rethinking}
Krzysztof~Marcin Choromanski, Valerii Likhosherstov, David Dohan, Xingyou Song,
  Andreea Gane, Tamas Sarlos, Peter Hawkins, Jared~Quincy Davis, Afroz
  Mohiuddin, Lukasz Kaiser, David~Benjamin Belanger, Lucy~J Colwell, and Adrian
  Weller.
\newblock Rethinking attention with {P}erformers.
\newblock In \emph{International Conference on Learning Representations
  (ICLR)}, 2021.

\bibitem[Ciarlet(2002)]{Ciarlet:2002finite}
Philippe~G Ciarlet.
\newblock \emph{The finite element method for elliptic problems}.
\newblock SIAM, 2002.

\bibitem[Ciarlet(2013)]{Ciarlet:2013Linear}
Philippe~G Ciarlet.
\newblock \emph{Linear and nonlinear functional analysis with applications},
  volume 130.
\newblock SIAM, 2013.

\bibitem[Courant et~al.(1967)Courant, Friedrichs, and
  Lewy]{Courant.Friedrichs.ea:1967partial}
Richard Courant, Kurt Friedrichs, and Hans Lewy.
\newblock On the partial difference equations of mathematical physics.
\newblock \emph{IBM journal of Research and Development}, 11\penalty0
  (2):\penalty0 215--234, 1967.

\bibitem[Csord\'as et~al.(2021)Csord\'as, Irie, and
  Schmidhuber]{CsordasIrieSchmidhuber2021Devil}
R\'obert Csord\'as, Kazuki Irie, and J\"urgen Schmidhuber.
\newblock The devil is in the detail: Simple tricks improve systematic
  generalization of transformers.
\newblock In \emph{Proc. Conf. on Empirical Methods in Natural Language
  Processing (EMNLP)}, Punta Cana, Dominican Republic, November 2021.

\bibitem[Dal~Santo et~al.(2020)Dal~Santo, Deparis, and
  Pegolotti]{DalSantoDeparisPegolotti2020Data}
Niccol{\`o} Dal~Santo, Simone Deparis, and Luca Pegolotti.
\newblock Data driven approximation of parametrized pdes by reduced basis and
  neural networks.
\newblock \emph{Journal of Computational Physics}, 416:\penalty0 109550, 2020.

\bibitem[de~Br{\'e}bisson and Vincent(2016)]{Brebisson.Vincent:2016Cheap}
Alexandre de~Br{\'e}bisson and Pascal Vincent.
\newblock A cheap linear attention mechanism with fast lookups and fixed-size
  representations.
\newblock \emph{arXiv preprint arXiv:1609.05866}, 2016.

\bibitem[Dehghani et~al.(2019)Dehghani, Gouws, Vinyals, Uszkoreit, and
  Kaiser]{DehghaniGouwsVinyalsUszkoreitEtAl2019Universal}
Mostafa Dehghani, Stephan Gouws, Oriol Vinyals, Jakob Uszkoreit, and Lukasz
  Kaiser.
\newblock Universal transformers.
\newblock In \emph{International Conference on Learning Representations}, 2019.

\bibitem[Devlin et~al.(2019)Devlin, Chang, Lee, and
  Toutanova]{DevlinChangLeeToutanova2019BERT}
Jacob Devlin, Ming{-}Wei Chang, Kenton Lee, and Kristina Toutanova.
\newblock {BERT:} pre-training of deep bidirectional transformers for language
  understanding.
\newblock In \emph{NAACL-HLT 2019, Minneapolis, MN, USA, June 2-7, 2019, Volume
  1}, pages 4171--4186. Association for Computational Linguistics, 2019.

\bibitem[Driscoll et~al.(2014)Driscoll, Hale, and
  Trefethen]{Driscoll.Hale.ea:2014Chebfun}
Tobin~A Driscoll, Nicholas Hale, and Lloyd~N Trefethen.
\newblock Chebfun guide, 2014.

\bibitem[Ern and Guermond(2004)]{Ern.Guermond:2004Theory}
Alexandre Ern and Jean-Luc Guermond.
\newblock \emph{Theory and Practice of Finite Elements}.
\newblock Springer, 2004.

\bibitem[Fox(1961)]{Fox:1961functions}
Charles Fox.
\newblock {The $G$ and $H$ functions as symmetrical Fourier kernels}.
\newblock \emph{Transactions of the American Mathematical Society}, 98\penalty0
  (3):\penalty0 395--429, 1961.

\bibitem[Fuchs et~al.(2020)Fuchs, Worrall, Fischer, and
  Welling]{Fuchs.Worrall.ea:2020SE3}
Fabian Fuchs, Daniel Worrall, Volker Fischer, and Max Welling.
\newblock {SE(3)-Transformers: 3D Roto-Translation Equivariant Attention
  Networks}.
\newblock In \emph{Advances in Neural Information Processing Systems},
  volume~33, pages 1970--1981, 2020.

\bibitem[Gilbarg and Trudinger(2001)]{Gilbarg.Trudinger2001Elliptic}
D.~Gilbarg and N.S. Trudinger.
\newblock \emph{Elliptic Partial Differential Equations of Second Order}.
\newblock Classics in Mathematics. Springer Berlin Heidelberg, 2001.
\newblock ISBN 9783540411604.

\bibitem[Girault and Raviart(1979)]{Girault.Raviart:1979Finite}
Vivette Girault and P-A Raviart.
\newblock {Finite element approximation of the Navier-Stokes equations}.
\newblock \emph{Lecture Notes in Mathematics, Berlin Springer Verlag}, 749,
  1979.

\bibitem[Glorot and Bengio(2010)]{Glorot.Bengio:2010Understanding}
Xavier Glorot and Yoshua Bengio.
\newblock Understanding the difficulty of training deep feedforward neural
  networks.
\newblock In \emph{Proceedings of the Thirteenth International Conference on
  Artificial Intelligence and Statistics}, volume~9 of \emph{Proceedings of
  Machine Learning Research}, pages 249--256, Chia Laguna Resort, Sardinia,
  Italy, 13--15 May 2010. PMLR.

\bibitem[Guo and Jiang(2021)]{GuoJiang2021Construct}
Ruchi Guo and Jiahua Jiang.
\newblock Construct deep neural networks based on direct sampling methods for
  solving electrical impedance tomography.
\newblock \emph{SIAM Journal on Scientific Computing}, 43\penalty0
  (3):\penalty0 B678--B711, 2021.

\bibitem[Guo et~al.(2016)Guo, Li, and Iorio]{Guo.Li.ea:2016Convolutional}
Xiaoxiao Guo, Wei Li, and Francesco Iorio.
\newblock Convolutional neural networks for steady flow approximation.
\newblock In \emph{Proceedings of the 22nd ACM SIGKDD international conference
  on knowledge discovery and data mining}, pages 481--490, 2016.

\bibitem[Gupta et~al.(2021)Gupta, Xiao, and
  Bogdan]{GuptaXiaoBogdan2021Multiwavelet}
Gaurav Gupta, Xiongye Xiao, and Paul Bogdan.
\newblock Multiwavelet-based operator learning for differential equations.
\newblock \emph{arXiv preprint arXiv:2109.13459}, 2021.

\bibitem[Hackbusch(2013)]{Hackbusch:2013Multi}
W.~Hackbusch.
\newblock \emph{Multi-Grid Methods and Applications}.
\newblock Springer Series in Computational Mathematics. Springer Berlin
  Heidelberg, 2013.
\newblock ISBN 9783662024270.

\bibitem[Harris et~al.(2020)Harris, Millman, van~der Walt, Gommers, Virtanen,
  Cournapeau, Wieser, Taylor, Berg, Smith, Kern, Picus, Hoyer, van Kerkwijk,
  Brett, Haldane, del R{\'{i}}o, Wiebe, Peterson, G{\'{e}}rard-Marchant,
  Sheppard, Reddy, Weckesser, Abbasi, Gohlke, and
  Oliphant]{Harris.Millman.ea:2020Array}
Charles~R. Harris, K.~Jarrod Millman, St{\'{e}}fan~J. van~der Walt, Ralf
  Gommers, Pauli Virtanen, David Cournapeau, Eric Wieser, Julian Taylor,
  Sebastian Berg, Nathaniel~J. Smith, Robert Kern, Matti Picus, Stephan Hoyer,
  Marten~H. van Kerkwijk, Matthew Brett, Allan Haldane, Jaime~Fern{\'{a}}ndez
  del R{\'{i}}o, Mark Wiebe, Pearu Peterson, Pierre G{\'{e}}rard-Marchant,
  Kevin Sheppard, Tyler Reddy, Warren Weckesser, Hameer Abbasi, Christoph
  Gohlke, and Travis~E. Oliphant.
\newblock Array programming with {NumPy}.
\newblock \emph{Nature}, 585\penalty0 (7825):\penalty0 357--362, September
  2020.

\bibitem[He and Xu(2019)]{He.Xu:2019MgNet}
Juncai He and Jinchao Xu.
\newblock Mgnet: A unified framework of multigrid and convolutional neural
  network.
\newblock \emph{{Science China Mathematics}}, 62\penalty0 (7):\penalty0
  1331--1354, 2019.

\bibitem[He et~al.(2016)He, Zhang, Ren, and Sun]{He.Zhang.ea:2016Identity}
Kaiming He, Xiangyu Zhang, Shaoqing Ren, and Jian Sun.
\newblock Identity mappings in deep residual networks.
\newblock In \emph{European conference on computer vision}, pages 630--645.
  Springer, 2016.

\bibitem[Hornik et~al.(1989)Hornik, Stinchcombe, and
  White]{HornikStinchcombeWhite1989Multilayer}
Kurt Hornik, Maxwell Stinchcombe, and Halbert White.
\newblock Multilayer feedforward networks are universal approximators.
\newblock \emph{Neural networks}, 2\penalty0 (5):\penalty0 359--366, 1989.

\bibitem[Hunter(2007)]{Hunter:2007Matplotlib}
J.~D. Hunter.
\newblock Matplotlib: A 2d graphics environment.
\newblock \emph{Computing in Science \& Engineering}, 9\penalty0 (3):\penalty0
  90--95, 2007.

\bibitem[Hutchinson et~al.(2021)Hutchinson, Le~Lan, Zaidi, Dupont, Teh, and
  Kim]{Hutchinson.Lan.ea:2021LieTransformer}
Michael~J Hutchinson, Charline Le~Lan, Sheheryar Zaidi, Emilien Dupont,
  Yee~Whye Teh, and Hyunjik Kim.
\newblock {LieTransformer: equivariant self-attention for Lie groups}.
\newblock pages 4533--4543, 2021.

\bibitem[Inc.(2015)]{Inc.:2015plotly}
Plotly~Technologies Inc.
\newblock plotly, 2015.

\bibitem[{JGraph}(2021)]{drawio}
{JGraph}.
\newblock draw.io, 2021.

\bibitem[Jiang et~al.(2021)Jiang, Li, and Guo]{JiangLiGuo2021learn}
Jiahua Jiang, Yi~Li, and Ruchi Guo.
\newblock Learn an index operator by cnn for solving diffusive optical
  tomography: a deep direct sampling method.
\newblock \emph{arXiv preprint arXiv:2104.07703}, 2021.

\bibitem[Jumper et~al.(2021)Jumper, Evans, Pritzel, Green, Figurnov,
  Ronneberger, Tunyasuvunakool, Bates, {\v{Z}}{\'\i}dek, Potapenko,
  et~al.]{JumperEvansPritzelGreenEtAl2021Highly}
John Jumper, Richard Evans, Alexander Pritzel, Tim Green, Michael Figurnov,
  Olaf Ronneberger, Kathryn Tunyasuvunakool, Russ Bates, Augustin
  {\v{Z}}{\'\i}dek, Anna Potapenko, et~al.
\newblock {Highly accurate protein structure prediction with AlphaFold}.
\newblock \emph{Nature}, 596\penalty0 (7873):\penalty0 583--589, 2021.

\bibitem[Karniadakis et~al.(2021)Karniadakis, Kevrekidis, Lu, Perdikaris, Wang,
  and Yang]{Karniadakis.Kevrekidis.ea:2021Physics}
George~Em Karniadakis, Ioannis~G. Kevrekidis, Lu~Lu, Paris Perdikaris, Sifan
  Wang, and Liu Yang.
\newblock Physics-informed machine learning.
\newblock \emph{Nature Reviews Physics}, 2021.

\bibitem[Katharopoulos et~al.(2020)Katharopoulos, Vyas, Pappas, and
  Fleuret]{Katharopoulos.Vyas.ea:2020Transformers}
Angelos Katharopoulos, Apoorv Vyas, Nikolaos Pappas, and Fran{\c{c}}ois
  Fleuret.
\newblock {Transformers are RNNs: Fast autoregressive transformers with linear
  attention}.
\newblock In \emph{International Conference on Machine Learning}, pages
  5156--5165. PMLR, 2020.

\bibitem[Kirsch(2011)]{Kirsch:2011Introduction}
A.~Kirsch.
\newblock \emph{An Introduction to the Mathematical Theory of Inverse
  Problems}.
\newblock Applied Mathematical Sciences. Springer New York, 2011.
\newblock ISBN 9781441984746.

\bibitem[Klein et~al.(2017)Klein, Kim, Deng, Senellart, and
  Rush]{Klein.Kim.ea:2017OpenNMT}
Guillaume Klein, Yoon Kim, Yuntian Deng, Jean Senellart, and Alexander~M. Rush.
\newblock {OpenNMT: Open-Source Toolkit for Neural Machine Translation}.
\newblock In \emph{Proc. ACL}, 2017.

\bibitem[Lax and Richtmyer(1956)]{Lax.Richtmyer:1956Survey}
Peter~D Lax and Robert~D Richtmyer.
\newblock Survey of the stability of linear finite difference equations.
\newblock \emph{Communications on pure and applied mathematics}, 9\penalty0
  (2):\penalty0 267--293, 1956.

\bibitem[Lee-Thorp et~al.(2021)Lee-Thorp, Ainslie, Eckstein, and
  Ontanon]{Lee-Thorp.Ainslie.ea:2021FNet}
James Lee-Thorp, Joshua Ainslie, Ilya Eckstein, and Santiago Ontanon.
\newblock {FNet}: Mixing tokens with {F}ourier transforms.
\newblock \emph{arXiv preprint arXiv:2105.03824}, 2021.

\bibitem[Li et~al.(2020{\natexlab{a}})Li, Zhang, and Zhao]{LiZhangZhao2020data}
Sijing Li, Zhiwen Zhang, and Hongkai Zhao.
\newblock A data-driven approach for multiscale elliptic {PDEs} with random
  coefficients based on intrinsic dimension reduction.
\newblock \emph{Multiscale Modeling \& Simulation}, 18\penalty0 (3):\penalty0
  1242--1271, 2020{\natexlab{a}}.

\bibitem[Li et~al.(2020{\natexlab{b}})Li, Kovachki, Azizzadenesheli, Liu,
  Stuart, Bhattacharya, and Anandkumar]{Li.Kovachki.ea:2020Multipole}
Zongyi Li, Nikola Kovachki, Kamyar Azizzadenesheli, Burigede Liu, Andrew
  Stuart, Kaushik Bhattacharya, and Anima Anandkumar.
\newblock Multipole graph neural operator for parametric partial differential
  equations.
\newblock In \emph{Advances in Neural Information Processing Systems},
  volume~33, pages 6755--6766, 2020{\natexlab{b}}.

\bibitem[Li et~al.(2021)Li, Kovachki, Azizzadenesheli, liu, Bhattacharya,
  Stuart, and Anandkumar]{Li.Kovachki.ea:2021Fourier}
Zongyi Li, Nikola~Borislavov Kovachki, Kamyar Azizzadenesheli, Burigede liu,
  Kaushik Bhattacharya, Andrew Stuart, and Anima Anandkumar.
\newblock Fourier neural operator for parametric partial differential
  equations.
\newblock In \emph{International Conference on Learning Representations}, 2021.

\bibitem[Lin and Jegelka(2018)]{LinJegelka2018ResNet}
Hongzhou Lin and Stefanie Jegelka.
\newblock Resnet with one-neuron hidden layers is a universal approximator.
\newblock \emph{Advances in Neural Information Processing Systems (NIPS 2018)},
  31:\penalty0 6169--6178, 2018.

\bibitem[Lu et~al.(2021{\natexlab{a}})Lu, Grover, Abbeel, and
  Mordatch]{Lu.Grover.ea:2021Pretrained}
Kevin Lu, Aditya Grover, Pieter Abbeel, and Igor Mordatch.
\newblock Pretrained transformers as universal computation engines.
\newblock \emph{arXiv preprint arXiv:2103.05247}, 2021{\natexlab{a}}.

\bibitem[Lu et~al.(2019)Lu, Jin, and Karniadakis]{LuJinKarniadakis2019Deeponet}
Lu~Lu, Pengzhan Jin, and George~Em Karniadakis.
\newblock Deeponet: Learning nonlinear operators for identifying differential
  equations based on the universal approximation theorem of operators.
\newblock \emph{arXiv preprint arXiv:1910.03193}, 2019.

\bibitem[Lu et~al.(2021{\natexlab{b}})Lu, Jin, Pang, Zhang, and
  Karniadakis]{Lu.Jin.ea:2021Deeponet}
Lu~Lu, Pengzhan Jin, Guofei Pang, Zhongqiang Zhang, and George~Em Karniadakis.
\newblock Learning nonlinear operators via deeponet based on the universal
  approximation theorem of operators.
\newblock \emph{Nature Machine Intelligence}, 3\penalty0 (3):\penalty0
  218--229, 2021{\natexlab{b}}.

\bibitem[Lu et~al.(2021{\natexlab{c}})Lu, Meng, Mao, and
  Karniadakis]{Lu.Meng.ea:2021DeepXDE}
Lu~Lu, Xuhui Meng, Zhiping Mao, and George~Em Karniadakis.
\newblock Deepxde: A deep learning library for solving differential equations.
\newblock \emph{SIAM Review}, 63\penalty0 (1):\penalty0 208--228,
  2021{\natexlab{c}}.

\bibitem[Monk et~al.(2003)]{Monk.others:2003Finite}
Peter Monk et~al.
\newblock \emph{Finite element methods for Maxwell's equations}.
\newblock Oxford University Press, 2003.

\bibitem[Nelsen and Stuart(2021)]{Nelsen.Stuart:2020Random}
Nicholas~H Nelsen and Andrew~M Stuart.
\newblock The random feature model for input-output maps between banach spaces.
\newblock \emph{SIAM Journal on Scientific Computing}, 43\penalty0
  (5):\penalty0 A3212--A3243, 2021.

\bibitem[Nguyen et~al.(2021)Nguyen, Suliafu, Osher, Chen, and
  Wang]{NguyenSuliafuOsherChenEtAl2021FMMformer}
Tan~M. Nguyen, Vai Suliafu, Stanley~J. Osher, Long Chen, and Bao Wang.
\newblock {FMMformer: Efficient and Flexible Transformer via Decomposed
  Near-field and Far-field Attention}.
\newblock In \emph{Advances in Neural Information Processing Systems
  (NeurIPS)}, 2021.

\bibitem[Nguyen and Salazar(2019)]{Nguyen.Salazar:2019Transformers}
Toan~Q Nguyen and Julian Salazar.
\newblock Transformers without tears: Improving the normalization of
  self-attention.
\newblock \emph{arXiv preprint arXiv:1910.05895}, 2019.

\bibitem[Okuno et~al.(2018)Okuno, Hada, and
  Shimodaira]{OkunoHadaShimodaira2018probabilistic}
Akifumi Okuno, Tetsuya Hada, and Hidetoshi Shimodaira.
\newblock A probabilistic framework for multi-view feature learning with
  many-to-many associations via neural networks.
\newblock In \emph{International Conference on Machine Learning}, pages
  3888--3897. PMLR, 2018.

\bibitem[Oliphant(2007)]{Oliphant:2007Python}
Travis~E. Oliphant.
\newblock Python for scientific computing.
\newblock \emph{Computing in Science Engineering}, 9\penalty0 (3):\penalty0
  10--20, 2007.

\bibitem[Paszke et~al.(2019)Paszke, Gross, Massa, Lerer, Bradbury, Chanan,
  Killeen, Lin, Gimelshein, Antiga, Desmaison, Kopf, Yang, DeVito, Raison,
  Tejani, Chilamkurthy, Steiner, Fang, Bai, and
  Chintala]{Paszke.Gross.ea2019PyTorch}
Adam Paszke, Sam Gross, Francisco Massa, Adam Lerer, James Bradbury, Gregory
  Chanan, Trevor Killeen, Zeming Lin, Natalia Gimelshein, Luca Antiga, Alban
  Desmaison, Andreas Kopf, Edward Yang, Zachary DeVito, Martin Raison, Alykhan
  Tejani, Sasank Chilamkurthy, Benoit Steiner, Lu~Fang, Junjie Bai, and Soumith
  Chintala.
\newblock Pytorch: An imperative style, high-performance deep learning library.
\newblock In \emph{Advances in Neural Information Processing Systems 32
  (NeurIPS 2019)}, pages 8024--8035, 2019.

\bibitem[Peng et~al.(2021)Peng, Pappas, Yogatama, Schwartz, Smith, and
  Kong]{PengPappasYogatamaSchwartzEtAl2021Random}
Hao Peng, Nikolaos Pappas, Dani Yogatama, Roy Schwartz, Noah Smith, and
  Lingpeng Kong.
\newblock Random feature attention.
\newblock In \emph{International Conference on Learning Representations}, 2021.

\bibitem[Raissi et~al.(2019)Raissi, Perdikaris, and
  Karniadakis]{Raissi.Perdikaris.ea:2019Physics}
M.~Raissi, P.~Perdikaris, and G.E. Karniadakis.
\newblock Physics-informed neural networks: A deep learning framework for
  solving forward and inverse problems involving nonlinear partial differential
  equations.
\newblock \emph{Journal of Computational Physics}, 378:\penalty0 686--707,
  2019.

\bibitem[Ramachandran et~al.(2017)Ramachandran, Zoph, and
  Le]{Ramachandran.Zoph.ea:2017Searching}
Prajit Ramachandran, Barret Zoph, and Quoc~V Le.
\newblock Searching for activation functions.
\newblock \emph{arXiv preprint arXiv:1710.05941}, 2017.

\bibitem[Rawat et~al.(2019)Rawat, Chen, Yu, Suresh, and
  Kumar]{RawatChenEtAl2019rfa}
Ankit~Singh Rawat, Jiecao Chen, Felix Xinnan~X Yu, Ananda~Theertha Suresh, and
  Sanjiv Kumar.
\newblock Sampled softmax with random fourier features.
\newblock In \emph{Advances in Neural Information Processing Systems},
  volume~32, 2019.

\bibitem[Roberts et~al.(2021)Roberts, Khodak, Dao, Li, R{\'e}, and
  Talwalkar]{RobertsKhodakDaoLiEtAl2021Rethinking}
Nicholas Roberts, Mikhail Khodak, Tri Dao, Liam Li, Christopher R{\'e}, and
  Ameet Talwalkar.
\newblock Rethinking neural operations for diverse tasks.
\newblock \emph{arXiv preprint arXiv:2103.15798}, 2021.

\bibitem[Schlag et~al.(2021)Schlag, Irie, and
  Schmidhuber]{Schlag.Irie.ea:2021Linear}
Imanol Schlag, Kazuki Irie, and J{\"u}rgen Schmidhuber.
\newblock Linear transformers are secretly fast weight programmers, 2021.

\bibitem[Shen et~al.(2021)Shen, Zhang, Zhao, Yi, and
  Li]{Shen.Zhang.ea:2021Efficient}
Zhuoran Shen, Mingyuan Zhang, Haiyu Zhao, Shuai Yi, and Hongsheng Li.
\newblock Efficient attention: Attention with linear complexities.
\newblock In \emph{Proceedings of the IEEE/CVF Winter Conference on
  Applications of Computer Vision}, pages 3531--3539, 2021.

\bibitem[Siegel and Xu(2020)]{SiegelXu2020Approximation}
Jonathan~W Siegel and Jinchao Xu.
\newblock Approximation rates for neural networks with general activation
  functions.
\newblock \emph{Neural Networks}, 128:\penalty0 313--321, 2020.

\bibitem[Smith and Topin(2019)]{Smith.Topin:2019Super}
Leslie~N Smith and Nicholay Topin.
\newblock Super-convergence: Very fast training of neural networks using large
  learning rates.
\newblock In \emph{Artificial Intelligence and Machine Learning for
  Multi-Domain Operations Applications}, volume 11006, page 1100612.
  International Society for Optics and Photonics, 2019.

\bibitem[Song et~al.(2021)Song, Jung, Kim, and
  Moon]{SongJungKimMoon2021implicit}
Kyungwoo Song, Yohan Jung, Dongjun Kim, and Il-Chul Moon.
\newblock Implicit kernel attention.
\newblock \emph{Proceedings of the AAAI Conference on Artificial Intelligence},
  35\penalty0 (11):\penalty0 9713--9721, May 2021.

\bibitem[Stern(2015)]{Stern:2015Banach}
Ari Stern.
\newblock {Banach space projections and Petrov--Galerkin estimates}.
\newblock \emph{Numerische Mathematik}, 130\penalty0 (1):\penalty0 125--133,
  2015.

\bibitem[Sutskever et~al.(2014)Sutskever, Vinyals, and
  Le]{Sutskever.Vinyals.ea:2014Sequence}
Ilya Sutskever, Oriol Vinyals, and Quoc~V. Le.
\newblock Sequence to sequence learning with neural networks.
\newblock In \emph{Proceedings of the 27th International Conference on Neural
  Information Processing Systems - Volume 2}, NIPS'14, page 3104–3112,
  Cambridge, MA, USA, 2014. MIT Press.

\bibitem[Tai et~al.(2019)Tai, Bailis, and
  Valiant]{Tai.Bailis.ea:2019Equivariant}
Kai~Sheng Tai, Peter Bailis, and Gregory Valiant.
\newblock Equivariant transformer networks.
\newblock In \emph{International Conference on Machine Learning}, pages
  6086--6095. PMLR, 2019.

\bibitem[Tolstikhin et~al.(2021)Tolstikhin, Houlsby, Kolesnikov, Beyer, Zhai,
  Unterthiner, Yung, Steiner, Keysers, Uszkoreit,
  et~al.]{Tolstikhin.Houlsby.ea:2021MLP}
Ilya Tolstikhin, Neil Houlsby, Alexander Kolesnikov, Lucas Beyer, Xiaohua Zhai,
  Thomas Unterthiner, Jessica Yung, Andreas Steiner, Daniel Keysers, Jakob
  Uszkoreit, et~al.
\newblock {MLP}-mixer: An all-{MLP} architecture for vision.
\newblock \emph{arXiv preprint arXiv:2105.01601}, 2021.

\bibitem[Tsai et~al.(2019)Tsai, Bai, Yamada, Morency, and
  Salakhutdinov]{Tsai.Bai.ea:2019Transformer}
Yao-Hung~Hubert Tsai, Shaojie Bai, Makoto Yamada, Louis-Philippe Morency, and
  Ruslan Salakhutdinov.
\newblock Transformer dissection: An unified understanding for transformer{'}s
  attention via the lens of kernel.
\newblock In \emph{Proceedings of the 2019 Conference on Empirical Methods in
  Natural Language Processing and the 9th International Joint Conference on
  Natural Language Processing (EMNLP-IJCNLP)}, pages 4344--4353, Hong Kong,
  China, November 2019.

\bibitem[Ulyanov et~al.(2016)Ulyanov, Vedaldi, and
  Lempitsky]{Ulyanov.Vedaldi.ea:2017Instance}
Dmitry Ulyanov, Andrea Vedaldi, and Victor Lempitsky.
\newblock Instance normalization: The missing ingredient for fast stylization.
\newblock \emph{arXiv preprint arXiv:1607.08022}, 2016.

\bibitem[Ummenhofer et~al.(2020)Ummenhofer, Prantl, Thuerey, and
  Koltun]{Ummenhofer.Prantl.ea:2020Lagrangian}
Benjamin Ummenhofer, Lukas Prantl, Nils Thuerey, and Vladlen Koltun.
\newblock Lagrangian fluid simulation with continuous convolutions.
\newblock In \emph{International Conference on Learning Representations}, 2020.

\bibitem[Van~Rossum and Drake(2009)]{VanRossum.Drake:2009Python}
Guido Van~Rossum and Fred~L. Drake.
\newblock \emph{Python 3 Reference Manual}.
\newblock CreateSpace, Scotts Valley, CA, 2009.
\newblock ISBN 1441412697.

\bibitem[Vaswani et~al.(2017)Vaswani, Shazeer, Parmar, Uszkoreit, Jones, Gomez,
  Kaiser, and Polosukhin]{Vaswani;Shazeer;Parmar:2017Attention}
Ashish Vaswani, Noam Shazeer, Niki Parmar, Jakob Uszkoreit, Llion Jones,
  Aidan~N Gomez, Lukasz Kaiser, and Illia Polosukhin.
\newblock Attention is all you need.
\newblock In \emph{Advances in Neural Information Processing Systems (NIPS
  2017)}, volume~30, 2017.

\bibitem[Virtanen et~al.(2020)Virtanen, Gommers, Oliphant, Haberland, Reddy,
  Cournapeau, Burovski, Peterson, Weckesser, Bright, {van der Walt}, Brett,
  Wilson, Millman, Mayorov, Nelson, Jones, Kern, Larson, Carey, Polat, Feng,
  Moore, {VanderPlas}, Laxalde, Perktold, Cimrman, Henriksen, Quintero, Harris,
  Archibald, Ribeiro, Pedregosa, {van Mulbregt}, and {SciPy 1.0
  Contributors}]{Virtanen.Gommers.ea:2020SciPy}
Pauli Virtanen, Ralf Gommers, Travis~E. Oliphant, Matt Haberland, Tyler Reddy,
  David Cournapeau, Evgeni Burovski, Pearu Peterson, Warren Weckesser, Jonathan
  Bright, St{\'e}fan~J. {van der Walt}, Matthew Brett, Joshua Wilson, K.~Jarrod
  Millman, Nikolay Mayorov, Andrew R.~J. Nelson, Eric Jones, Robert Kern, Eric
  Larson, C~J Carey, {\.I}lhan Polat, Yu~Feng, Eric~W. Moore, Jake
  {VanderPlas}, Denis Laxalde, Josef Perktold, Robert Cimrman, Ian Henriksen,
  E.~A. Quintero, Charles~R. Harris, Anne~M. Archibald, Ant{\^o}nio~H. Ribeiro,
  Fabian Pedregosa, Paul {van Mulbregt}, and {SciPy 1.0 Contributors}.
\newblock {{SciPy} 1.0: Fundamental Algorithms for Scientific Computing in
  Python}.
\newblock \emph{Nature Methods}, 17:\penalty0 261--272, 2020.

\bibitem[Wang et~al.(2021{\natexlab{a}})Wang, Teng, and
  Perdikaris]{Wang.Teng.ea:2021Understanding}
Sifan Wang, Yujun Teng, and Paris Perdikaris.
\newblock Understanding and mitigating gradient flow pathologies in
  physics-informed neural networks.
\newblock \emph{SIAM Journal on Scientific Computing}, 43\penalty0
  (5):\penalty0 A3055--A3081, 2021{\natexlab{a}}.

\bibitem[Wang et~al.(2021{\natexlab{b}})Wang, Wang, and
  Perdikaris]{Wang.Wang.ea:2021Learning}
Sifan Wang, Hanwen Wang, and Paris Perdikaris.
\newblock Learning the solution operator of parametric partial differential
  equations with physics-informed {DeepOnets}.
\newblock \emph{arXiv preprint arXiv:2103.10974}, 2021{\natexlab{b}}.

\bibitem[Wang et~al.(2020)Wang, Li, Khabsa, Fang, and
  Ma]{Wang.Li.ea:2020Linformer}
Sinong Wang, Belinda~Z Li, Madian Khabsa, Han Fang, and Hao Ma.
\newblock Linformer: Self-attention with linear complexity.
\newblock \emph{arXiv preprint arXiv:2006.04768}, 2020.

\bibitem[Waskom(2021)]{Waskom:2021seaborn}
Michael~L. Waskom.
\newblock seaborn: statistical data visualization.
\newblock \emph{Journal of Open Source Software}, 6\penalty0 (60):\penalty0
  3021, 2021.

\bibitem[Whitney(2015)]{Whitney:2015Geometric}
H.~Whitney.
\newblock \emph{Geometric Integration Theory}.
\newblock Princeton Legacy Library. Princeton University Press, 2015.
\newblock ISBN 9781400877577.

\bibitem[Wright and Gonzalez(2021)]{Wright.Gonzalez:2021Transformers}
Matthew~A Wright and Joseph~E Gonzalez.
\newblock Transformers are deep infinite-dimensional non-{M}ercer binary kernel
  machines.
\newblock \emph{arXiv preprint arXiv:2106.01506}, 2021.

\bibitem[Xiong et~al.(2020)Xiong, Yang, He, Zheng, Zheng, Xing, Zhang, Lan,
  Wang, and Liu]{Xiong.Yang.ea:2020Layer}
Ruibin Xiong, Yunchang Yang, Di~He, Kai Zheng, Shuxin Zheng, Chen Xing,
  Huishuai Zhang, Yanyan Lan, Liwei Wang, and Tieyan Liu.
\newblock On layer normalization in the transformer architecture.
\newblock In \emph{International Conference on Machine Learning}, pages
  10524--10533. PMLR, 2020.

\bibitem[Xiong et~al.(2021)Xiong, Zeng, Chakraborty, Tan, Fung, Li, and
  Singh]{Xiong.Zeng.ea:2021Nystromformer}
Yunyang Xiong, Zhanpeng Zeng, Rudrasis Chakraborty, Mingxing Tan, Glenn Fung,
  Yin Li, and Vikas Singh.
\newblock Nyströmformer: A {N}yström-based algorithm for approximating
  self-attention.
\newblock \emph{Proceedings of the AAAI Conference on Artificial Intelligence},
  35\penalty0 (16):\penalty0 14138--14148, May 2021.

\bibitem[Xu and Zikatanov(2017)]{Xu.Zikatanov:2017Algebraic}
Jinchao Xu and Ludmil Zikatanov.
\newblock Algebraic multigrid methods.
\newblock \emph{Acta Numerica}, 26:\penalty0 591--721, 2017.

\bibitem[Yun et~al.(2020)Yun, Bhojanapalli, Rawat, Reddi, and
  Kumar]{YunBhojanapalliRawatEtAl2020Are}
Chulhee Yun, Srinadh Bhojanapalli, Ankit~Singh Rawat, Sashank Reddi, and Sanjiv
  Kumar.
\newblock Are transformers universal approximators of sequence-to-sequence
  functions?
\newblock In \emph{International Conference on Learning Representations}, 2020.

\bibitem[Zhu and Zabaras(2018)]{Zhu.Zabaras:2018Bayesian}
Yinhao Zhu and Nicholas Zabaras.
\newblock Bayesian deep convolutional encoder--decoder networks for surrogate
  modeling and uncertainty quantification.
\newblock \emph{Journal of Computational Physics}, 366:\penalty0 415--447,
  2018.

\end{thebibliography}

}
\end{document}